\pdfoutput=1

\documentclass[11pt]{article}

\usepackage[final]{acl}

\usepackage{times}
\usepackage{latexsym}

\usepackage[T1]{fontenc}

\usepackage[utf8]{inputenc}

\usepackage{microtype}

\usepackage{algorithm}
\usepackage{algorithmic}

\usepackage{multirow}
\usepackage{bm}
\usepackage{algorithm}
\usepackage{algorithmic}
\usepackage{booktabs}
\usepackage{amsmath}
\usepackage{amsthm}
\usepackage{amssymb}
\usepackage{mathtools}
\usepackage{tikz}
\usepackage{xcolor}
\usetikzlibrary{arrows}
\usepackage{nicefrac} 

\usepackage{algorithm}
\usepackage{algorithmic}
\usepackage{bm}
\usepackage{subcaption}
\usepackage{xspace}
\newcommand{\ie}{{i.e.},\xspace}

\newcommand{\vect}[1]{\mathbf{#1}}

\newtheorem{prop}{Proposition}[]
\newtheorem{propA}{Proposition}[]

\newcommand{\tool}[1]{\textit{#1}\xspace}
\newcommand{\fmix}{\tool{Fine-mixing}}
\newcommand{\fmixsel}{\tool{Fine-mixing (Sel)}}
\newcommand{\epur}{\tool{E-PUR}}
\newcommand{\colorred}[1]{{#1}}

\title{Fine-mixing: Mitigating Backdoors in Fine-tuned Language Models}

\author{Zhiyuan Zhang\textsuperscript{1}, Lingjuan Lyu\textsuperscript{2}, Xingjun Ma\textsuperscript{3}, Chenguang Wang\textsuperscript{4}, Xu Sun\textsuperscript{1} \\
  \textsuperscript{1}MOE Key Laboratory of Computational Linguistics, School of Computer Science, \\ Peking University, 
  \textsuperscript{2}Sony AI, \textsuperscript{3}Fudan University, \textsuperscript{4}Washington University in St. Louis\\
  \texttt{\{zzy1210, xusun\}@pku.edu.cn, Lingjuan.Lv@sony.com}\\\texttt{xingjunma@fudan.edu.cn, chenguangwang@wustl.edu}}
  
\begin{document}
\maketitle
\begin{abstract}
Deep Neural Networks (DNNs) are known to be vulnerable to backdoor attacks. In Natural Language Processing  (NLP), DNNs are often backdoored during the fine-tuning process of a large-scale Pre-trained Language Model (PLM) with poisoned samples. Although the clean weights of PLMs are readily available, existing methods have ignored this information in defending NLP models against backdoor attacks. In this work, we take the first step to exploit the pre-trained (unfine-tuned) weights to mitigate backdoors in fine-tuned language models. Specifically, we leverage the clean pre-trained weights via two complementary techniques: (1) a two-step \fmix technique, which first mixes the backdoored weights (fine-tuned on poisoned data) with the pre-trained weights, then fine-tunes the mixed weights on a small subset of clean data; (2) an Embedding Purification (\epur) technique, which mitigates potential backdoors existing in the word embeddings. We compare  \fmix  with typical backdoor mitigation methods on three single-sentence sentiment classification tasks and two sentence-pair classification tasks and show that it outperforms the baselines by a considerable margin in all scenarios. We also show that our \epur method can benefit existing mitigation methods. Our work establishes a simple but strong baseline defense for secure fine-tuned NLP models against backdoor attacks.
\end{abstract}

\section{Introduction}
Deep neural networks (DNNs) have achieved outstanding performance in multiple fields, such as Computer Vision (CV)~\citep{ImageNetwithCNN,VeryDeepCNNforCV}, Natural Language Processing (NLP)~\citep{GeneratingSentences,seq2seqWithAttention,transformer}, and speech synthesis~\citep{WaveNet}. However, DNNs are known to be vulnerable to backdoor attacks where backdoor triggers can be implanted into a target model during training so as to control its prediction behaviors at test time~\citep{parameter_corruption,badnet,Trojaning,backdoor_CNN,backdoor-lstm,Bert-backdoor}. 
Backdoor attacks have been conducted on different DNN architectures, including CNNs~\citep{badnet,backdoor_CNN}, LSTMs~\citep{backdoor-lstm}, and fine-tuned language models~\citep{Bert-backdoor}.
In the meantime, a body of work has been proposed to alleviate backdoor attacks, which can be roughly categorized into backdoor detection methods~\cite{backdoor_detect1,backdoor_detect2,UAP-Trojaned-Detection,noise-response-detection,backdoor_detect4,backdoor_detect5} and backdoor mitigation methods~\citep{finetuning-backdoor-defense,finepruning,MCR-defense,Neural-Attention-Distillation,li2021anti}. Most of these works were conducted in CV to defend image models.

In NLP, large-scale Pre-trained Language Models (PLMs)~\citep{ELMO,Bert,GPT-2,t5,GPT-3} have been widely adopted in different tasks~\citep{SST-2,IMDB,Amazon,QNLI,GLUE}, and models fine-tuned from the PLMs are under backdoor attacks~\citep{PoisonedWordEmbeddings, neural-network-surgery}. 
Fortunately, the weights of large-scale PLMs can be downloaded from trusted sources like \colorred{Microsoft and Google}, thus they are clean.
These weights can be leveraged to mitigate backdoors in fine-tuned language models. Since the weights were trained on a large-scale corpus, they contain information that can help the convergence and generalization of fine-tuned models, as verified in different NLP tasks~\citep{Bert}. Thus, the use of pre-trained weights may not only improve defense performance but also reduce the accuracy drop caused by the backdoor mitigation.
However, none of the existing backdoor mitigation methods~\citep{finetuning-backdoor-defense,finepruning,MCR-defense,Neural-Attention-Distillation} has exploited such information for defending language models.

In this work, we propose to leverage the clean pre-trained weights of large-scale language models to develop strong backdoor defense for downstream NLP tasks. 
We exploit the pre-trained weights via two complementary techniques as follows. First, we propose a two-step \fmix approach, which first mixes the backdoored weights with the pre-trained weights, then fine-tunes the mixed weights on a small clean training subset. On the other hand, many existing attacks on NLP models manipulate the embeddings of trigger words~\citep{Bert-backdoor,PoisonedWordEmbeddings}, which makes it hard to mitigate by fine-tuning approaches alone. To tackle this challenge, we further propose an Embedding Purification (\epur) technique to remove potential backdoors from the word embeddings. \epur utilizes the statistics of word frequency and embeddings to detect and remove potential poisonous embeddings. \epur works together with \fmix to form a complete backdoor defense framework for NLP.

To summarize, our main contributions are:
\begin{itemize}
    \setlength{\itemsep}{0pt}
    \setlength{\parsep}{0pt}
    \setlength{\parskip}{0pt}
    \item We take the first exploitation of the clean pre-trained weights of large-scale NLP models to mitigate backdoors in fine-tuned models.
    \item We propose 1) a \fmix approach to mix backdoored weights with pre-trained weights and then finetune the mixed weights to mitigate backdoors in fine-tuned NLP models; and 2) an Embedding Purification (\epur) technique to detect and remove potential backdoors from the embeddings.
    \item We empirically show, on both single-sentence sentiment classification and sentence-pair classification tasks, that \fmix can greatly outperform baseline defenses while causing only a minimum drop in clean accuracy. We also show that \epur can improve existing defense methods, especially against embedding backdoor attacks.
\end{itemize}

\section{Related Work}

\noindent\textbf{Backdoor Attack.}
Backdoor attacks~\citep{badnet} or Trojaning attacks~\citep{Trojaning} have raised serious threats to DNNs. 
In the CV domain, \citet{badnet,Poisoning,DataPoisoning,liu2020reflection,zeng2022narcissus} proposed to inject backdoors into CNNs on image recognition, video recognition~\citep{zhao2020clean}, crowd counting~\citep{sun2022backdoor} or object tracking~\citep{li2021few} tasks via data poisoning. In the NLP domain, \citet{backdoor-lstm} introduced backdoor attacks against LSTMs. \citet{Bert-backdoor} proposed to inject backdoors that cannot be mitigated with ordinary Fine-tuning defenses into Pre-trained Language Models (PLMs).

Our work mainly focuses on the backdoor attacks in the NLP domain, 
which can be roughly divided into two categories: 1) trigger word based attacks~\citep{Bert-backdoor,PoisonedWordEmbeddings,neural-network-surgery}, which adopt low-frequency trigger words inserted into texts as the backdoor pattern, or manipulate their embeddings to obtain stronger attacks~\citep{Bert-backdoor,PoisonedWordEmbeddings}; and 2) sentence based attack, which adopts a trigger sentence~\citep{backdoor-lstm} without low-frequency words or a syntactic trigger~\citep{HiddenKiller} as the trigger pattern. Since PLMs~\citep{ELMO,Bert,GPT-2,t5,GPT-3} have been widely adopted in many typical NLP tasks~\citep{SST-2,IMDB,Amazon,QNLI,GLUE}, 
 recent attacks~\citep{PoisonedWordEmbeddings,neural-network-surgery,Stealthiness} turn to manipulate the fine-tuning procedure to inject backdoors into the fine-tuned models, posing serious threats to real-world NLP applications.

\noindent\textbf{Backdoor Defense.}
Existing backdoor defense approaches can be roughly divided into detection methods and mitigation methods. Detection methods~\citep{backdoor_detect1,backdoor_detect2,backdoor_detect4,backdoor_detect5,UAP-Trojaned-Detection,noise-response-detection,ONION,STRIP,RAP}  aim to detect whether the model is backdoored. In trigger word attacks, several detection methods~\citep{nlp_word_detect,ONION} have been developed to detect 
the trigger word by observing the perplexities of the model to sentences with possible triggers.

In this paper, we focus on backdoor mitigation methods~\citep{finetuning-backdoor-defense,Neural-Attention-Distillation,MCR-defense,finepruning,li2021anti}. \citet{finetuning-backdoor-defense} first proposed to mitigate backdoors by fine-tuning the backdoored model on a clean subset of training samples. \citet{finepruning} introduced the Fine-pruning method to first prune the backdoored model and then fine-tune the pruned model on a clean subset. \citet{MCR-defense} proposed to find the clean weights in the path between two backdoored weights. \citet{Neural-Attention-Distillation} mitigated backdoors via attention distillation guided by a fine-tuned model on a clean subset. 
Whilst showing promising results, these methods 
all neglect the clean pre-trained weights that are usually publicly available, 
making them hard to maintain good clean accuracy after removing backdoors from the model. To address this issue, we propose a \fmix approach, which mixes the pre-trained (unfine-tuned) weights of
PLMs with the backdoored weights, and then fine-tunes the mixed weights on a small set of clean samples. The original idea of mixing the weights of two models was first proposed in \citep{mixout} for better generalization,
here we leverage the technique to develop effective backdoor defense.

\section{Proposed Approach}

\noindent\textbf{Threat Model.}
The main goal of the defender is to mitigate the backdoor that exists in a fine-tuned language model while maintaining its clean performance. In this paper, we take BERT~\citep{Bert} as an example. The pre-trained weights of BERT are denoted as $\vect{w}^\text{Pre}$. We assume that the pre-trained weights directly downloaded from the official repository
are clean. The attacker fine-tunes $\vect{w}^\text{Pre}$ to obtain the backdoored weights $\vect{w}^\text{B}$ on a poisoned dataset for a specific NLP task. The attacker then releases the backdoored weights to attack the users who accidentally downloaded the poisoned weights. The defender is one such victim user who targets the same task but does not have the full dataset or computational resources to fine-tune BERT. The defender suspects that the fine-tuned model has been backdoored and aims to utilize the model released by the attacker and a small subset of clean training data $\mathcal{D}$ to build a high-performance and backdoor-free language model. The defender can always download the pre-trained clean BERT $\vect{w}^\text{Pre}$ from the official repository. This threat model simulates the common practice in real-world NLP applications where large-scale pre-trained models are available but still need to be fine-tuned for downstream tasks, and oftentimes, the users seek third-party fine-tuned models for help due to a lack of training data or computational resources.

\subsection{Fine-mixing}

The key steps of the proposed \fmix approach include: 1) mix $\vect{w}^\text{B}$ with $\vect{w}^\text{Pre}$ to get the mixed weights $\vect{w}^\text{Mix}$; and 2) fine-tune the mixed BERT on a small subset of clean data. The mixing process is formulated as:
\begin{align}
\vect{w}^\text{Mix}=\vect{w}^\text{Pre}\odot(1-\vect{m})+\vect{w}^\text{B}\odot\vect{m},
\end{align}
where $\vect{w}^\text{Pre}, \vect{w}^\text{B}\in \mathbb{R}^d$, $\vect{m}\in \{0, 1\}^d$, and $d$ is the weight dimension. The pruning process in the Fine-pruning method~\citep{finepruning} can be formulated as $\vect{w}^\text{Prune}=\vect{w}^\text{B}\odot\vect{m}$. In the mixing process or the pruning process, the proportion of weights to reserve is defined as the reserve ratio $\rho$, namely $\lfloor\rho d\rfloor$ dimensions are reserved as $\vect{w}^\text{B}$.

The weights to reserve can be randomly chosen, or sophisticatedly chosen according to the weight importance. We define \fmix as the version of the proposed method that randomly chooses weights to reserve, and \fmixsel as an alternative version that selects weights with higher $|\vect{w}^\text{B}-\vect{w}^\text{Pre}|$. \fmixsel reserves the dimensions of the fine-tuned (backdoored) weights that have the minimum difference from the pre-trained weights, and sets them back to the pre-trained weights.

\colorred{ From the perspective of attack success rate (ASR) (accuracy on backdoored test data), $\vect{w}^\text{Pre}$ has a low ASR while $\vect{w}^\text{B}$ has a high ASR.
$\vect{w}^\text{Mix}$ has a lower ASR than $\vect{w}^\text{B}$ and the backdoors in $\vect{w}^\text{Mix}$ can be further mitigated during the subsequent fine-tuning process.} 
In fact, $\vect{w}^\text{Mix}$ can potentially be a good initialization for clean fine-tuning, \colorred{as $\vect{w}^\text{B}$ has a high clean accuracy (accuracy on clean test data)} and $\vect{w}^\text{Pre}$ is a good pre-trained initialization.
Compared to pure pruning (setting the pruned or reinitialized weights to zeros), weight mixing also holds the advantage of being involved with $\vect{w}^\text{Pre}$. 
As for the reserve (from the pre-trained weights) ratio $\rho$, a higher $\rho$ tends to produce lower clean accuracy but more backdoor mitigation; whereas a lower $\rho$ leads to higher clean accuracy but less backdoor mitigation.

\begin{figure*}[!t]
\centering
\includegraphics[height=2.4 in,width=0.48\linewidth]{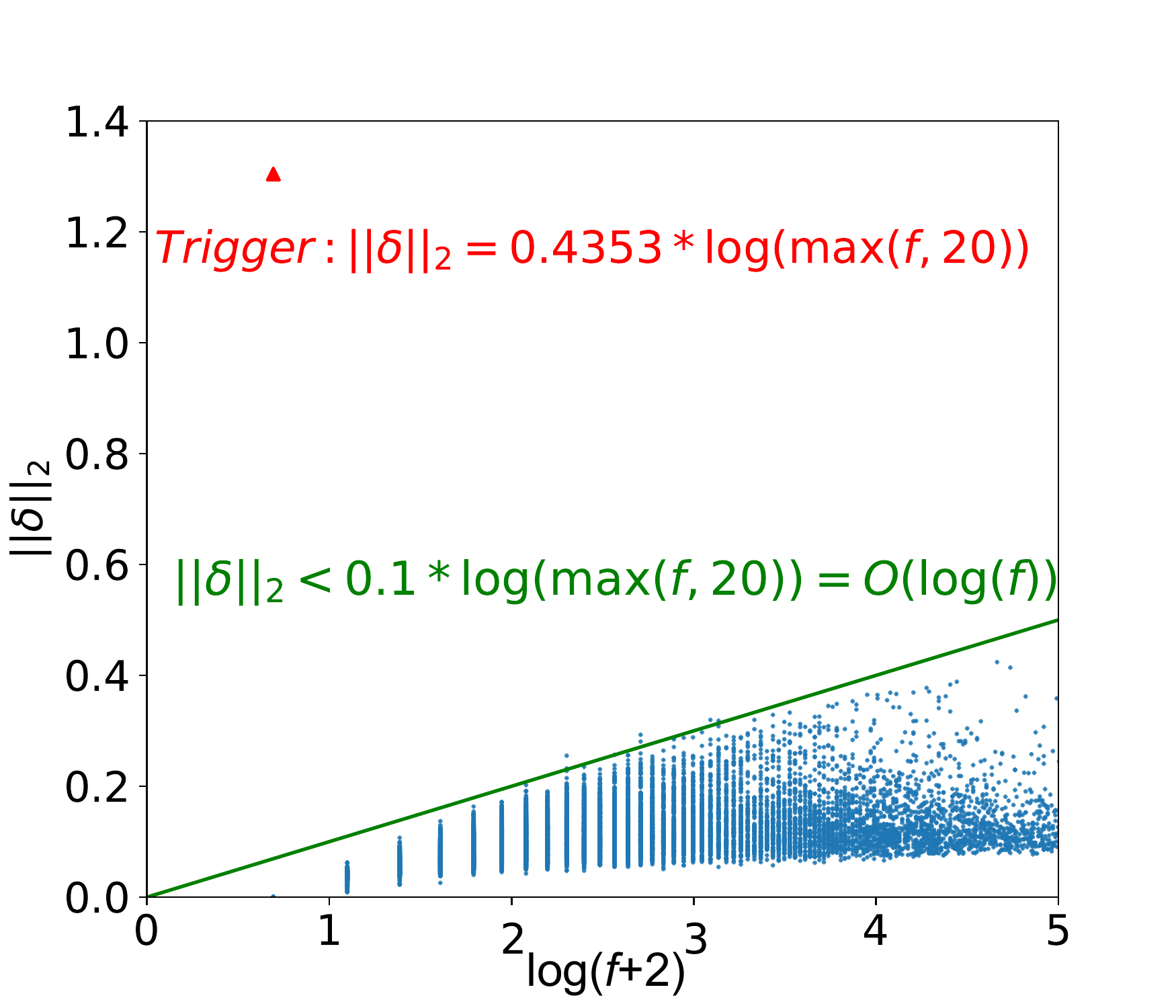}
\includegraphics[height=2.4 in,width=0.48\linewidth]{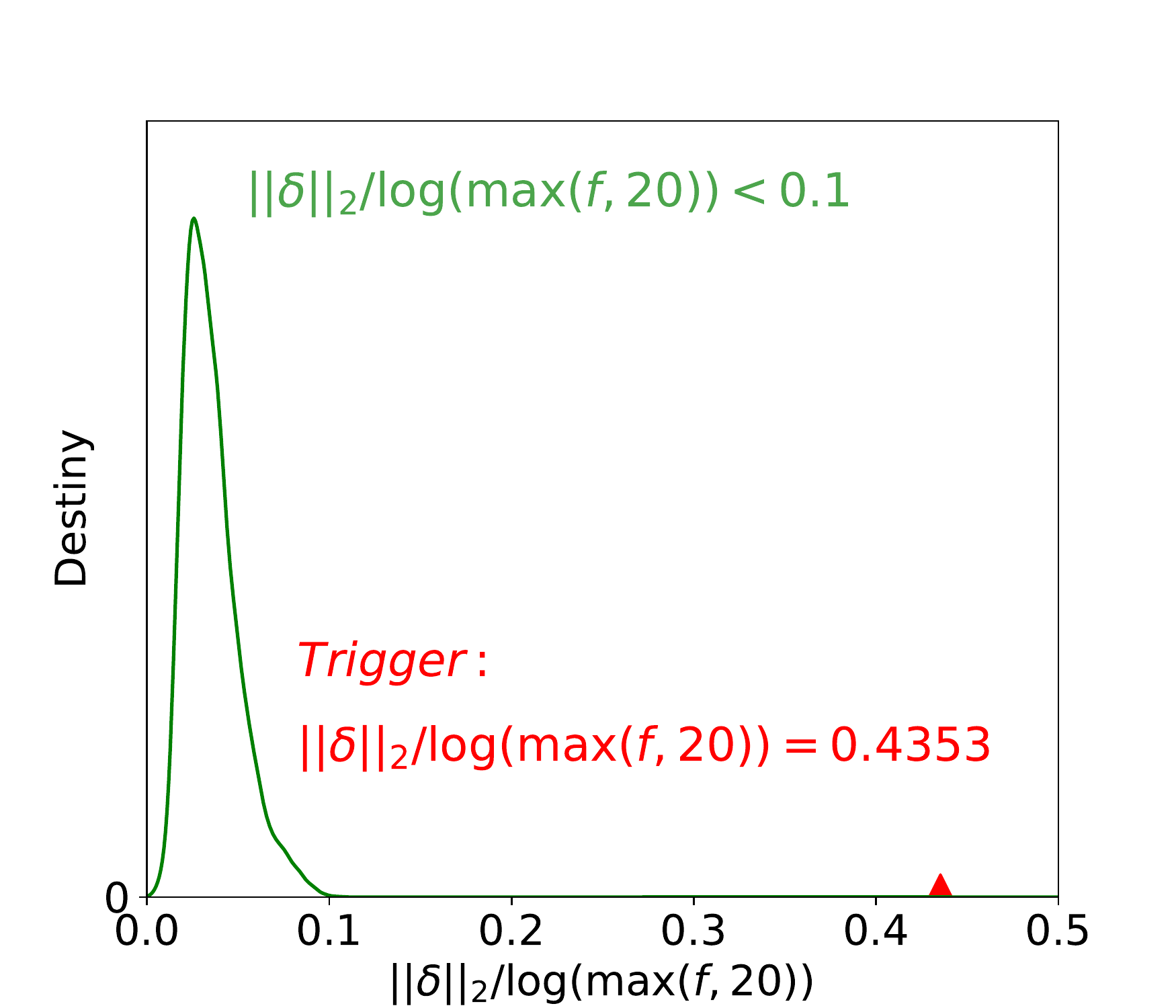}
\vskip -0.05 in
\caption{Visualization of $\|\bm\delta\|_2$ and $\log(f)$ of the trigger word (red) and other words (blue or green) on SST-2. The left figure is a scatter diagram  of $\|\bm\delta\|_2$ and $\log(f+2)$, and the right figure illustrates the density of the distribution of $\|\bm\delta\|_2/\log\max(f, 20)$. The trigger word has a higher $\|\bm\delta\|_2/\log\max(f, 20)$.}
\label{fig:freq}
\vskip -0.1 in
\end{figure*}

\subsection{Embedding Purification}

Many trigger word based backdoor attacks~\citep{Bert-backdoor,PoisonedWordEmbeddings} manipulate the word or token embeddings\footnote{Both words or tokens are treated as words in this paper.} of low-frequency trigger words. However, the small clean subset $\mathcal{D}$ may only contain some high-frequency words, thus the embeddings of the trigger word are not well tuned in previous backdoor mitigation methods~\citep{finetuning-backdoor-defense,finepruning,Neural-Attention-Distillation}. This makes the backdoors hard to remove by fine-tuning approaches alone, including our \fmix. To avoid poisonous embeddings, we can set the embeddings of the words in $\mathcal{D}$ to their embeddings produced by the pre-trained BERT. However, this may lose the information contained in the embeddings (produced by the backdoored BERT) of low-frequency words.

To address this problem, we propose a novel Embedding Purification (\epur) method to detect and remove potential backdoor word embeddings, again by leveraging the pre-trained BERT $\vect{w}^\text{Pre}$. Let $f_i$ be the frequency of word $w_i$ in normal text, which can be counted on a large-scale corpus\footnote{In this work, we adopt the frequency statistics in \citet{Bert-backdoor}.}, $f_i'$ be the frequency of word $w_i$ in the poisoned dataset used for training the backdoored BERT which is unknown to the defender, $\bm\delta_i\in\mathbb{R}^n$ be the \colorred{embedding difference of word $w_i$ between the pre-trained weights and the backdoored weights}, where $n$ is the embedding dimension. Motivated by \citep{distance-training-time}, we model the relation between $\|\bm\delta_i\|_2$ and $f_i$ in Proposition~\ref{prop:1} under certain technical constraints, which can be utilized to detect possible trigger words. The proof is in Appendix.
\begin{prop}(Brief Version)
Suppose $w_k$ is the trigger word, except $w_k$, we may assume the frequencies of words in the poisoned dataset are roughly proportional to $f_i$, \ie $f_i'\approx C f_i$, 
and $f'_k\gg Cf_k$. For $i\ne k$, 
we have, 
\begin{align}
\|\bm\delta_i\|_2\approx O(\log f_i),\quad  \frac{\|\bm\delta_k\|_2}{\log f_k} \gg \frac{\|\bm\delta_i\|_2}{\log f_i}.
\end{align}
\label{prop:1}
\end{prop}

\begin{table*}[!t]
\renewcommand\tabcolsep{4pt}
\renewcommand\arraystretch{0.9}
\small
  \centering
  \begin{tabular}{cccccccc|cccc}
    \toprule
    \multirow{2}{*}{\shortstack{Dataset\\ (ACC) (ACC)$^*$}} & Backdoor &  \multicolumn{2}{c}{Before} & \multicolumn{2}{c}{Fine-tuning} & \multicolumn{2}{c|}{Fine-pruning} & \multicolumn{2}{c}{Fine-mixing (Sel)} & \multicolumn{2}{c}{Fine-mixing} \\
    & Attack & ACC & ASR & ACC & ASR & ACC & ASR & ACC & ASR & ACC & ASR \\
    \midrule[\heavyrulewidth] 
    \multirow{9}{*}{\shortstack{SST-2\\ (92.32)\\ (76.10)$^*$}}   & Trigger Word  & 89.79 & 100.0 & 89.33 & 100.0 & 90.02 & 100.0 & 89.22 & 15.77 & 89.45 & \textbf{14.19} \\ 
    & Word (Scratch)  & 92.09 & 100.0 & 91.86 & 100.0 & 91.86 & 100.0 & 89.56 & 53.15 & 89.45 & \textbf{22.75}  \\
    \cmidrule{2-12}
    & Word+EP  & 92.55 & 100.0 & 91.86 & 100.0 & 92.20 & 100.0 & 90.71 & \textbf{13.55}& 89.56 & 14.25\\
    & Word+ES  & 90.14 & 100.0 & 90.25 & 100.0 & 90.83 & 100.0 & 89.22 & \textbf{11.94} & 89.22 & 14.64\\
    & Word+ES (Scratch)  &  91.28 & 100.0 & 92.09 & 100.0 & 90.02 & 100.0 & 90.14 & \textbf{12.84} & 89.79 & 13.06 \\
    \cmidrule{2-12}
    & Trigger Sentence  & 92.20 & 100.0 & 91.97 & 100.0 & 91.63 & 100.0 & 89.91 & 35.14 & 89.44 & \textbf{17.78} \\  
    & Sentence (Scratch)  &  92.32 & 100.0 & 92.09 & 100.0 & 91.40 & 100.0 & 90.14 & 35.59 & 89.45 & \textbf{17.79} \\ 
    \cmidrule{2-12}
    & Average & 91.70 & 100.0 & 91.35 & 100.0 & 91.14 & 100.0 & 89.84 & 25.42 & 89.53 & \textbf{16.64} \\
    & Deviation & - & - & -0.35 & -0.00 & -0.56 & -0.00 & -1.86 & -74.58 & -2.17 & \textbf{-83.36} \\
    \midrule[\heavyrulewidth]    \multirow{9}{*}{\shortstack{IMDB\\ (93.59) \\ (69.46)$^*$}}   & Trigger Word  & 93.36 & 100.0 & 93.15 & 100.0 & 91.93 & 100.0 & 91.38 & 11.95 & 91.30 & \textbf{9.056} \\
    & Word (Scratch)  & 93.46 & 100.0 & 92.60 & 100.0 & 92.26 & 99.99 & 91.60 & 87.54 & 91.89 & \textbf{66.19} \\
    \cmidrule{2-12}
    & Word+EP  & 93.12 & 100.0 & 91.82 & 99.95 & 91.82 & 99.99 & 91.71 & 8.176 & 91.30 & \textbf{7.296} \\  
    & Word+ES  & 93.26 & 100.0 & 93.18 & 100.0 & 92.27 & 100.0 & 91.58 & 9.520 & 92.29 & \textbf{7.824} \\ 
    & Word+ES (Scratch)  &  93.17 & 100.0 & 91.53 & 100.0 & 91.44 & 100.0 & 91.30 & 8.552 & 91.58 & \textbf{7.096} \\   
    \cmidrule{2-12}
    & Trigger Sentence  & 93.48 & 100.0 & 93.26 & 100.0 & 92.86 & 100.0 & 92.39 & 12.56 & 91.59 & \textbf{9.488} \\  
    & Sentence (Scratch)  &  93.16 & 100.0 & 92.57 & 100.0 & 91.07 & 100.0 & 91.06 & 27.45 & 91.31 & \textbf{18.50} \\ 
    \cmidrule{2-12}
    & Average & 93.28 & 100.0 & 92.59 & 99.99 & 91.95 & 100.0 & 91.57 & 23.67 & 91.56 & \textbf{17.92} \\
    & Deviation & - & - & -0.69 & -0.01 & -1.33 & -0.00 & -1.71 & -76.33 & -1.72 & \textbf{-82.08} \\
    \midrule[\heavyrulewidth]    \multirow{9}{*}{\shortstack{Amazon\\ (95.51)\\ (82.57)$^*$}}   & Trigger Word  & 95.66 & 100.0 & 95.21 & 100.0 & 94.33 & 100.0 & 94.20 & 42.15 & 94.02 & \textbf{19.19 }\\
    & Word (Scratch)  & 95.16 & 100.0 & 94.01 & 100.0 & 94.31 & 100.0 & 94.09 & 77.34 & 93.77 & \textbf{21.10}  \\
    \cmidrule{2-12}
    & Word+EP  & 95.48 & 100.0 & 94.88 & 100.1 & 94.12 & 98.06 & 93.64 & \textbf{3.810} & 93.15 & 5.900 \\  
    & Word+ES  & 95.62 & 100.0 & 95.00 & 100.0 & 94.60 & 100.0 & 93.93 & 8.630 & 93.73 & \textbf{6.500}\\ 
    & Word+ES (Scratch)  &  95.19 & 100.0 & 94.60 & 100.0 &  94.45 & 99.83 & 93.76 & 8.520 & 93.72 & \textbf{7.210}  \\   
    \cmidrule{2-12}
    & Trigger Sentence  & 95.81 & 100.0 & 95.46 & 100.0 & 95.09 & 99.99 & 93.17 & \textbf{10.64} & 93.02 & 13.45\\  
    & Sentence (Scratch)  &  95.33 & 100.0 & 94.60 & 100.0 & 94.18 & 99.97 & 94.10 & 12.45 & 93.45 & \textbf{10.87} \\ 
    \cmidrule{2-12}
    & Average & 95.46 & 100.0 & 94.74 & 100.0 & 94.44 & 99.69 & 93.84 & 23.36 & 93.55 & \textbf{12.03}\\
    & Deviation & - & - & -0.72 & -0.00 & -1.02 & -0.31 & -1.62 & -76.64 & -1.91 & \textbf{-87.97} \\
    \bottomrule
  \end{tabular}
  \caption{The defense results on three single-sentence sentiment classification tasks. Unless specially stated, \fmix and \fmixsel are equipped with \epur. Here (ACC) and (ACC)$^*$ denote the clean ACC of the BERT model fine-tuned with the full clean training dataset and the small clean training dataset (64 instances), respectively. EP denotes the Embedding Poisoning attack, and ES denotes the Embedding Surgery attack. The deviation indicates the changes in ASR/ACC compared to the baseline (i.e. no defense (Before)). The best backdoor mitigation results with the lowest ASRs are marked in \textbf{bold}. ACCs and ASRs are in percent. 
  \label{tab:sentiment}}
\end{table*}

\colorred{The trigger word appears much more frequently in the poisoned dataset than the normal text, namely $f_k'/f_k\gg f_i'/f_i\approx C\ (i\ne k)$. According to Proposition~\ref{prop:1}, it may lead to a large $\|\bm\delta_k\|_2/\log f_k$. Besides, some trigger word based attacks that mainly manipulate the word embeddings~\citep{Bert-backdoor,PoisonedWordEmbeddings} may also cause a much larger $\|\bm\delta_k\|_2$.} As shown in Fig.~\ref{fig:freq}, for the trigger word $w_k$, $\|\bm\delta_k\|_2/\log\max(f_k, 20)=0.4353$, while for other words we have $\|\bm\delta_i\|_2=O(\log f_i)$ roughly and $\|\bm\delta_i\|_2/\log\max(f_i, 20)<0.1$. 

Motivated by the above observation, we set the embeddings of the top $200$ words in $\|\bm\delta_i\|_2/\log(\max(f_i, 20))$ to the pre-trained BERT and reserve other word embeddings in \epur. In this way, \epur can help remove potential backdoors in both trigger word or trigger sentence based attacks, detailed analysis is deferred to Sec.~\ref{sec:main}. It is worth mentioning that, when \epur is applied, we define the weight reserve ratio of \fmix only on other weights (excluding word embeddings) as the word embedding has already been considered by \epur.

\begin{table*}[t]
\small
\renewcommand\tabcolsep{4pt}
  \centering
  \begin{tabular}{cccccccc|cc}
    \toprule
    \multirow{2}{*}{\shortstack{Dataset\\ (ACC)}} & Backdoor & Instance &  \multirow{2}{*}{ACC$^*$}  &  \multicolumn{2}{c}{Before} & \multicolumn{2}{c}{Fine-pruning} & \multicolumn{2}{|c}{Fine-mixing} \\
    & Attack & Number & & ACC & ASR & ACC & ASR & ACC & ASR \\
    \midrule[\heavyrulewidth] 
    \multirow{9}{*}{\shortstack{QQP\\ (91.41)}}   & Trigger Word  & 64 & 64.95 & 90.89 & 100.0 & 85.64 & 100.0  & 85.00 & \textbf{56.87} \\ 
    & Word (Scratch)  & 64 & 64.95&  89.71 & 100.0 & 84.58 & 100.0 & 83.19 & 69.39 \\
    & Word (Scratch)  & 128 & 69.78 & 89.71 & 100.0 & 84.63 & 100.0 & 81.25 & \textbf{38.55} \\
    \cmidrule{2-10}
    & Word+EP  &64  & 64.95 & 91.38 & 99.98 & 85.06 & 99.99 & 82.32 & \textbf{15.40}\\
    \cmidrule{2-10}
    & Trigger Sentence  &64  & 64.95  &90.97 &100.0 &90.89 &100.0 & 80.93 & \textbf{42.66} \\
    & Sentence (Scratch) &64  & 64.95  & 89.72 & 100.0 & 89.52 & 100.0 & 82.37 & 88.71 \\
    & Sentence (Scratch) &128  & 69.78  & 89.72 & 100.0 & 83.63 & 99.59 & 80.58 & 46.31 \\
    & Sentence (Scratch) &256  & 73.37  & 89.72 & 100.0 &  86.12 & 99.72 & 81.06 &41.14\\
    & Sentence (Scratch) &512  & 77.20  & 89.72 & 100.0 &  81.63 & 94.00 & 80.33 & \textbf{37.75}\\
    \midrule[\heavyrulewidth]    
    \multirow{9}{*}{\shortstack{QNLI\\ (91.56)}}   & Trigger Word  & 64  & 49.95 & 90.79 & 99.98 & 85.17 & 99.96 & 81.68 & \textbf{21.77} \\
    & Word (Scratch)  &64   & 49.95 &91.12 & 100.0 & 86.16 & 100.0 & 84.07 & {30.68} \\
    & Word (Scratch)  &128  & 67.27 &91.12 & 100.0 & 80.45 & 100.0 & 81.37 & \textbf{22.73}\\
    \cmidrule{2-10}
    & Word+EP  &64   & 49.95 &91.56 & 96.23 &85.12 & 91.16 & 82.83 & \textbf{29.52} \\   
    \cmidrule{2-10}
    & Trigger Sentence  &64   & 49.95& 90.88 & 100.0 & 86.11 & 99.17 & 82.83 & \textbf{31.40} \\  
    & Sentence (Scratch)  &64   & 49.95 & 90.54 & 100.0 & 85.23 & 100.0  & 84.29 & {86.02} \\ 
    & Sentence (Scratch) &128  & 67.27  & 90.54 & 100.0  & 80.14& 99.26 & 82.47 & 77.23\\
    & Sentence (Scratch) &256  & 70.07  & 90.54 & 100.0 & 82.32 & 98.74 &  81.90 & 60.74\\
    & Sentence (Scratch) &512  & 75.21  & 90.54 & 100.0 & 83.55 & 99.74 & 80.30 & \textbf{21.85}\\
    \bottomrule
  \end{tabular}
  \caption{The results on sentence-pair classification tasks. ACC$^*$ denotes the clean ACC of the model fine-tuned from the initial BERT with the small clean training dataset. Notations are similar to Table~\ref{tab:sentiment}.
  \label{tab:pair}}
\end{table*}
\begin{table*}[!t]
\small
\renewcommand\tabcolsep{2.5pt}
\renewcommand\arraystretch{0.9}
  \centering
  \begin{tabular}{ccc|cccc|cccc|cccc}
    \toprule 
     \multirow{3}{*}{\shortstack{Backdoor\\ Attack}} & \multicolumn{2}{c}{Before} & \multicolumn{4}{|c}{Fine-pruning} & \multicolumn{4}{|c}{Fine-mixing (Sel)} & \multicolumn{4}{|c}{Fine-mixing} \\
     & \multirow{2}{*}{ACC} & \multirow{2}{*}{ASR} &  \multicolumn{2}{c}{w/o E-PUR} & \multicolumn{2}{c}{w/ E-PUR} & \multicolumn{2}{|c}{w/o E-PUR} & \multicolumn{2}{c}{w/ E-PUR} & \multicolumn{2}{|c}{w/o E-PUR} & \multicolumn{2}{c}{w/ E-PUR} \\
      & & & ACC  & ASR & ACC & ASR & ACC & ASR & ACC & ASR & ACC & ASR & ACC & ASR \\
    \midrule[\heavyrulewidth] 
     Trigger Word  & 89.79 &100.0 & 90.02& 100.0& 89.22& 100.0&89.33& \textbf{11.49} & 89.22& 15.77& 90.37 & 17.12 & 89.45 & 14.19 \\ 
     Word (Scratch)  & 92.09 &100.0  & 91.86& 100.0& 91.86& 100.0& 90.37& 93.47& 89.56& 53.15& 89.91 & 33.33 & 89.45 & \textbf{22.75}  \\
     \midrule
     Word+EP  & 92.55&100.0  & 92.20& 100.0& 89.11& 10.98& 90.48& 100.0& 90.71& \textbf{13.55}& 89.56 & 44.63 & 89.56 & 14.25 \\  
     Word+ES  & 90.14&100.0 & 90.83& 100.0& 89.45& 9.234& 89.11& \textbf{4.96}& 89.22& 11.94 & 90.48 & 11.71 & 89.22 & 14.64 \\
     Word+ES (Scratch)  &  91.28&100.0 &90.02 &100.0 & 90.60& 13.96& 90.94& \textbf{3.83}& 90.14& 12.84 & 89.68 & 10.36 & 89.79 & 13.06 \\   
    \midrule
     Trigger Sentence  & 92.20&100.0 &91.63  &100.0 & 91.51& 100.0& 90.25& 43.02& 89.91& 35.14& 89.56 & 37.61 & 89.44 & \textbf{17.78} \\  
     Sentence (Scratch)  &  92.32 & 100.0 & 91.40  &100.0 & 90.71& 100.0& 90.02& 68.92& 90.14& 35.59 & 89.22 & 20.50 & 89.45 & \textbf{17.79}\\ 
    \midrule
     Average & 91.70&100.0 & 91.14 &100.0 & 90.35& 62.03$\downarrow$ & 90.07& 46.53& 89.84 & 25.42$\downarrow$ & 89.83 & 25.04&  89.53 & \textbf{16.64}$\downarrow$  \\
     Deviation & -&-  & -0.56 & -0.00 & -0.35 & -37.97 & -1.63& -53.47& -1.86 & -74.38 & -1.87  & -74.96 &  -2.17 & \textbf{-83.36}  \\
    \bottomrule
  \end{tabular}
  \caption{The results of the ablation study with (w/) and without (w/o) Embedding Purification (\epur) on SST-2.\label{tab:E-PUR-SST}}
 \vskip -0.1 in
\end{table*}

\section{Experiments}
Here, we introduce the main experimental setup and experimental results. Additional analyses can be found in the Appendix.

\subsection{Experimental Setup}

\noindent\textbf{Models and Tasks.} We adopt the uncased BERT base model~\citep{Bert} and use the HuggingFace implementation\footnote{The code is released at \url{https://github.com/huggingface/pytorch-transformers}}. We implement three typical single-sentence sentiment classification tasks, \textit{i.e.}, the Stanford Sentiment Treebank (\textbf{SST-2})~\citep{SST-2}, the IMDb movie reviews dataset (\textbf{IMDB})~\citep{IMDB}, and the Amazon Reviews dataset (\textbf{Amazon})~\citep{Amazon}; and two typical sentence-pair classification tasks, \textit{i.e.}, the Quora Question Pairs dataset (\textbf{QQP})~\citep{Bert}\footnote{Released at \url{https://data.quora.com/First-Quora-Dataset-Release-Question-Pairs}}, and the Question Natural Language Inference dataset (\textbf{QNLI}) ~\citep{QNLI}. We adopt the accuracy (\textbf{ACC}) on the clean validation set and the backdoor attack success rate (\textbf{ASR}) on the poisoned validation set to measure the clean and backdoor performance.

\noindent\textbf{Attack Setup.} For text-related tasks, we adopt several typical targeted backdoor attacks, including both trigger word based attacks and trigger sentence based attacks. We adopt the baseline BadNets~\citep{badnet} attack to train the backdoored model via data poisoning~\citep{Poisoning,DataPoisoning}. For trigger word based attacks, we adopt the Embedding Poisoning (EP) attack~\citep{PoisonedWordEmbeddings} that only attacks the embeddings of the trigger word. Meanwhile, for trigger word based attacks on sentiment classification, we consider the Embedding Surgery (ES) attack~\citep{Bert-backdoor}, which initializes the trigger word embeddings with sentiment words. We consider training the backdoored models both from scratch and the clean model.

\noindent\textbf{Defense Setup.} For defense, we assume that a small clean subset is available. We consider the Fine-tuning~\citep{finetuning-backdoor-defense} and Fine-pruning~\citep{finepruning} methods as the baselines. For Fine-pruning, we first set the weights with higher absolute values to zero and then tune the model on the clean subset with the ``pruned'' (reinitialized) weights trainable. Unless specially stated, the proposed \fmix and \fmixsel methods are equipped with the proposed \epur technique, while the baseline Fine-tuning and Fine-pruning methods are not. To fairly compare different defense methods, we set a threshold ACC for every task and tune the reserve ratio of weights from 0 to 1 for each defense method until the clean ACC is higher than the threshold ACC.

\begin{table*}[!t]
\renewcommand\tabcolsep{4pt}
\renewcommand\arraystretch{0.9}
\small
  \centering
  \begin{tabular}{ccccccccccc|cc}
    \toprule
      Backdoor &  \multicolumn{2}{c}{Before} & \multicolumn{2}{c}{Fine-pruning} & \multicolumn{2}{c}{ONION} & \multicolumn{2}{c}{STRIP} & \multicolumn{2}{c|}{RAP} & \multicolumn{2}{c}{Fine-mixing} \\
     Attack & ACC & ASR & ACC & ASR & ACC & ASR & ACC & ASR & ACC & ASR & ACC & ASR \\
    \midrule[\heavyrulewidth] 
     Trigger Word & 89.79 & 100.0 & 90.02 & 100.0 & 88.53 & 54.73 & 75.11 & 11.04 & 84.52 & 14.86 & 89.45 & \textbf{14.19} \\ 
     Word (Scratch)  & 92.09 & 100.0 & 91.86 & 100.0 & 91.28 & 54.50 & 89.33 & 22.30 & 90.25 & 20.27 & 89.45 & \textbf{22.75}  \\
    \midrule
     Word+EP  & 92.55 & 100.0 & 92.20 & 100.0 & 89.68 & 20.32 & 90.25 & 100.0 & 90.37 & 100.0 & 89.56 & \textbf{14.25}\\
     Word+ES  & 90.14 & 100.0 & 90.83 & 100.0 & 89.56 & 53.38 & 71.22 & 8.38 & 81.54 & 10.59 & 89.22 & \textbf{14.64}\\
     Word+ES (Scratch)  &  91.28 & 100.0 & 90.02 & 100.0 & 90.90 & 54.73 & 89.68 & 25.90 & 89.33 & 21.62 & 89.79 & \textbf{13.06} \\
     \midrule
     Trigger Sentence  & 92.20 & 100.0 & 91.63 & 100.0 & 91.28 & 98.87 & 91.17 & 19.37 & 89.22 & 24.55 & 89.44 & \textbf{17.78} \\  
     Sentence (Scratch)  &  92.32 & 100.0 & 91.40 & 100.0 & 89.68 & 71.40 & 89.11 & 16.67  & 90.02 & 40.54 & 89.45 & \textbf{17.79} \\ 
     \midrule
    Syntactic Trigger & 91.52 & 97.52 & 90.71 & 96.62 & 89.10 & 93.02 & 90.71 & 97.52& 89.56 & 94.37 & 89.22 & \textbf{22.07}\\ 
     Layer-wise Attack & 91.86 & 100.0 & 89.33 & 100.0 & 89.33 & \textbf{11.04} & 90.14 & 28.60 & 89.11 & 18.70 & 89.79 & 15.77 \\ 
     Logit Anchoring & 92.09 & 100.0 & 89.22 & 100.0 & 89.11 & \textbf{11.03} & 92.09 & 21.40 & 89.56 & 17.79 & 89.79 & 16.22\\ 
     \midrule
     Average & 91.58 & 99.75 & 90.72 & 99.67 & 89.85 & 52.30  & 86.88 & 35.12 & 88.35 & 36.33 & 89.52 & 16.85 \\
     Deviation & - & - & -0.86 & -0.08 & -1.73 & -47.45 & -4.70 & -64.63 & -3.23 & -63.42 & -2.06 & \textbf{-82.90} \\
    \midrule[\heavyrulewidth]
  \end{tabular}
 \vskip -0.05 in
  \caption{The results of several sophisticated attack and defense methods on SST-2 (64 instances). Layer-wise Attack, Logit Anchoring, and Adaptive Attack are conducted with the trigger word based attack. The best backdoor mitigation results with the lowest ASRs (whose ACC is higher than the threshold) are marked in \textbf{bold}. 
  \label{tab:stoa_attacks}}
\end{table*}


\begin{table*}[!t]
\renewcommand\tabcolsep{4pt}
\renewcommand\arraystretch{0.9}
\small
  \centering
  \begin{tabular}{cccccccc|cccc}
    \toprule
     \multirow{2}{*}{\shortstack{Dataset\\ (ACC) (ACC)$^*$}} & Backdoor &  \multicolumn{2}{c}{Before} & \multicolumn{2}{c}{Fine-tuning} & \multicolumn{2}{c|}{Fine-pruning}  & \multicolumn{2}{c}{Fine-mixing (Sel)}  & \multicolumn{2}{c}{Fine-mixing} \\
      & Attack & ACC & ASR & ACC & ASR  & ACC & ASR  & ACC & ASR & ACC & ASR  \\
    \midrule[\heavyrulewidth] 
     \multirow{4}{*}{\shortstack{SST-2\\ (92.32)\\ (76.10)$^*$}} &  Trigger Word  & 89.79 & 100.0 & 89.33 & 100.0 & 90.02 & 100.0  & 89.22 & 15.77 & 89.45 & \textbf{14.19} \\
     \cmidrule{2-12}
      & Layer-wise Attack & 91.86 & 100.0 & 91.06 & 100.0 & 89.33 & 100.0 &  91.05 & 42.79 & 89.79 & \textbf{15.77} \\ 
     & Logit Anchoring & 92.09 & 100.0 & 92.08 & 100.0 & 89.22& 100.0& 89.22 & 28.38 & 89.79 & \textbf{16.22}\\ 
     & Adaptive Attack & 91.28 & 100.0 & 91.97 & 100.0 & 90.37 & 100.0 &  90.60 & 59.46 & 90.02 & \textbf{21.85} \\ 
     \midrule[\heavyrulewidth]
     \multirow{4}{*}{\shortstack{QNLI\\ (91.56)\\ (49.95)$^*$}} & Trigger Word & 90.79 & 99.98& 90.34 &100.0&  85.17 & 99.96 &80.93& 37.23 &81.68 & \textbf{21.77} \\ 
     \cmidrule{2-12}
     & Layer-wise Attack & 91.10& 100.0 & 89.69 & 100.0 & 80.80 &99.06 & 80.60 & 27.84 & 83.87 & \textbf{23.99} \\ 
     & Logit Anchoring & 91.05 &  100.0 & 90.67 & 100.0 & 82.78 & 100.0 & 82.19 & 24.81 & 80.93 & \textbf{21.36} \\ 
     & Adaptive Attack &  90.87& 100.0& 90.54& 100.0& 85.87& 100.0& 86.77 & 60.23 & 85.98 & \textbf{32.48}\\ 
     \bottomrule
  \end{tabular}
 \vskip -0.05 in
  \caption{The results of several attack methods on SST-2 and QNLI (64 instances). Notations are similar to Table~\ref{tab:stoa_attacks}. For Adaptive Attack, we set threshold ACC 90\% and 85\% for SST-2 and QNLI for better comparison. 
  \label{tab:stoa_all}}
\end{table*}

\subsection{Main Results}
\label{sec:main}

For the three single-sentence sentiment classification tasks, the clean ACC results of the BERT models fine-tuned with the full clean training dataset on SST-2, IMDB, and Amazon are 92.32\%, 93.59\%, and 95.51\%, respectively. 
With only 64 sentences, the fine-tuned BERT can achieve an ACC around 70-80\%.
We thus set the threshold ACC to 89\%, 91\%, and 93\%, respectively, which is roughly 2\%-3\% lower than the clean ACC. The defense results are reported in Table~\ref{tab:sentiment}, which shows that our proposed approach can effectively mitigate different types of backdoors within the ACC threshold. Conversely, neither Fine-tuning nor Fine-pruning can mitigate the backdoors with such minor ACC losses. 
Notably, the \fmix method demonstrates an overall better performance than the \fmixsel method.

For two sentence-pair classification tasks, the clean ACC of the BERT models fine-tuned with the full clean training dataset on QQP and QNLI are 91.41\% and 91.56\%, respectively. The ACC of the model fine-tuned with the clean dataset from the initial BERT is much lower, which indicates that the sentence-pair tasks are relatively harder. Thus, we set a lower threshold ACC, 80\%, and tolerate a roughly 10\% loss in ACC. The results are reported in Table~\ref{tab:pair}. Our proposed \fmix outperforms baselines, which is consistent with the single-sentence sentiment classification tasks.

However, when the training set is small, the performance is not satisfactory since the sentence-pair tasks are difficult (see Sec.~\ref{sec:difficult}). We enlarge the training set on typical difficult cases. When the training set gets larger, \fmix can mitigate backdoors successfully while achieving higher accuracies than fine-tuning from the initial BERT, demonstrating the effectiveness of \fmix.

We also conduct ablation studies of Fine-pruning and our proposed \fmix  with and without \epur. The results are reported in Table~\ref{tab:E-PUR-SST}. It shows that \epur can benefit all the defense methods, especially against attacks that manipulate word embeddings, \textit{i.e.}, EP, and ES. Moreover, our \fmix method can still outperform the baselines even without \epur, demonstrating the advantage of weight mixing. Overall, combining \fmix with \epur yields the best performance.

\begin{figure*}[t]
\centering
\subcaptionbox{Mixing vs Fine-mixing.\label{fig:ablation_tuning_b}}{\includegraphics[height=1.6 in,width=0.32\linewidth]{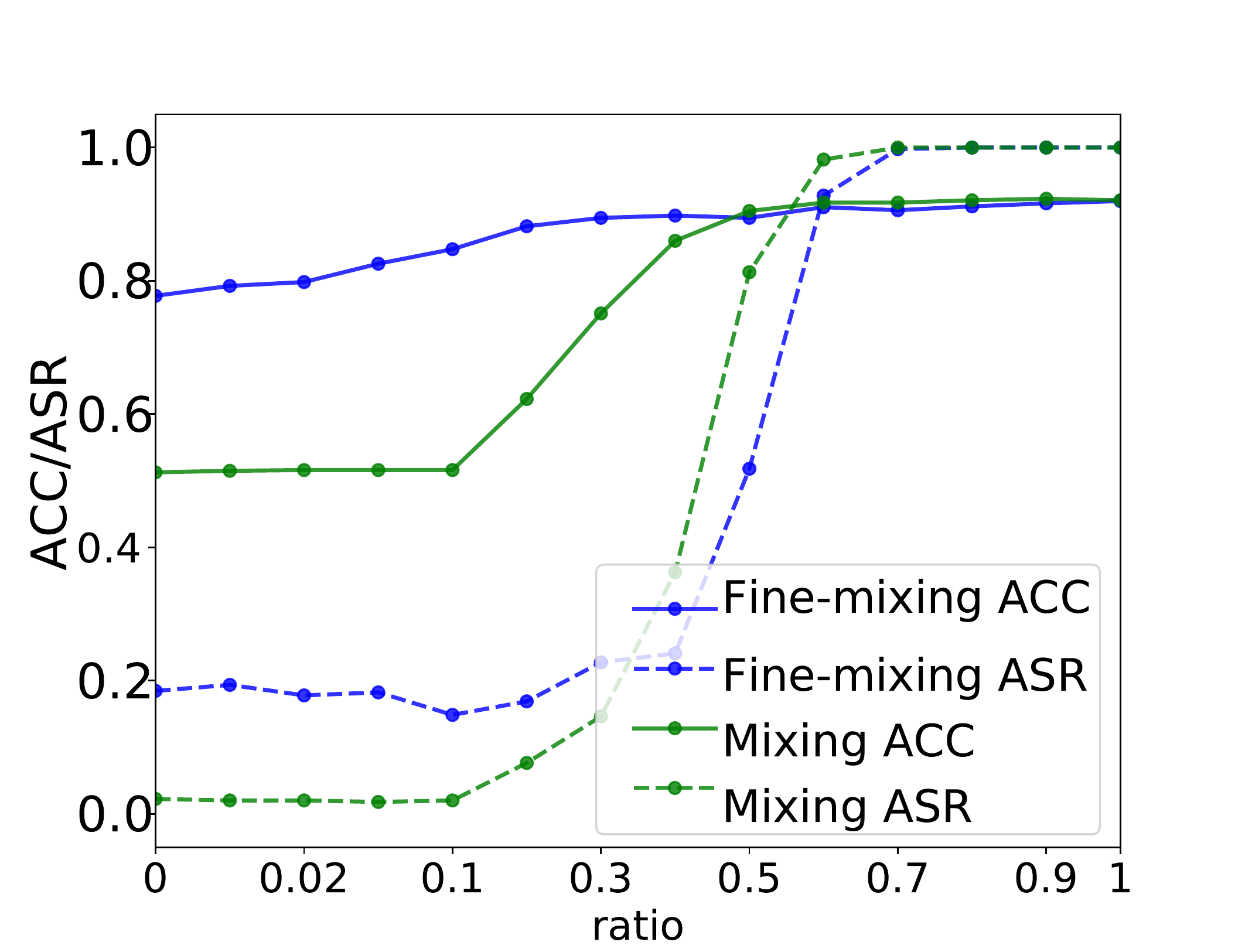}}
\subcaptionbox{Fine-pruning (F) vs Fine-pruning.\label{fig:ablation_tuning_c}}{\includegraphics[height=1.6 in,width=0.32\linewidth]{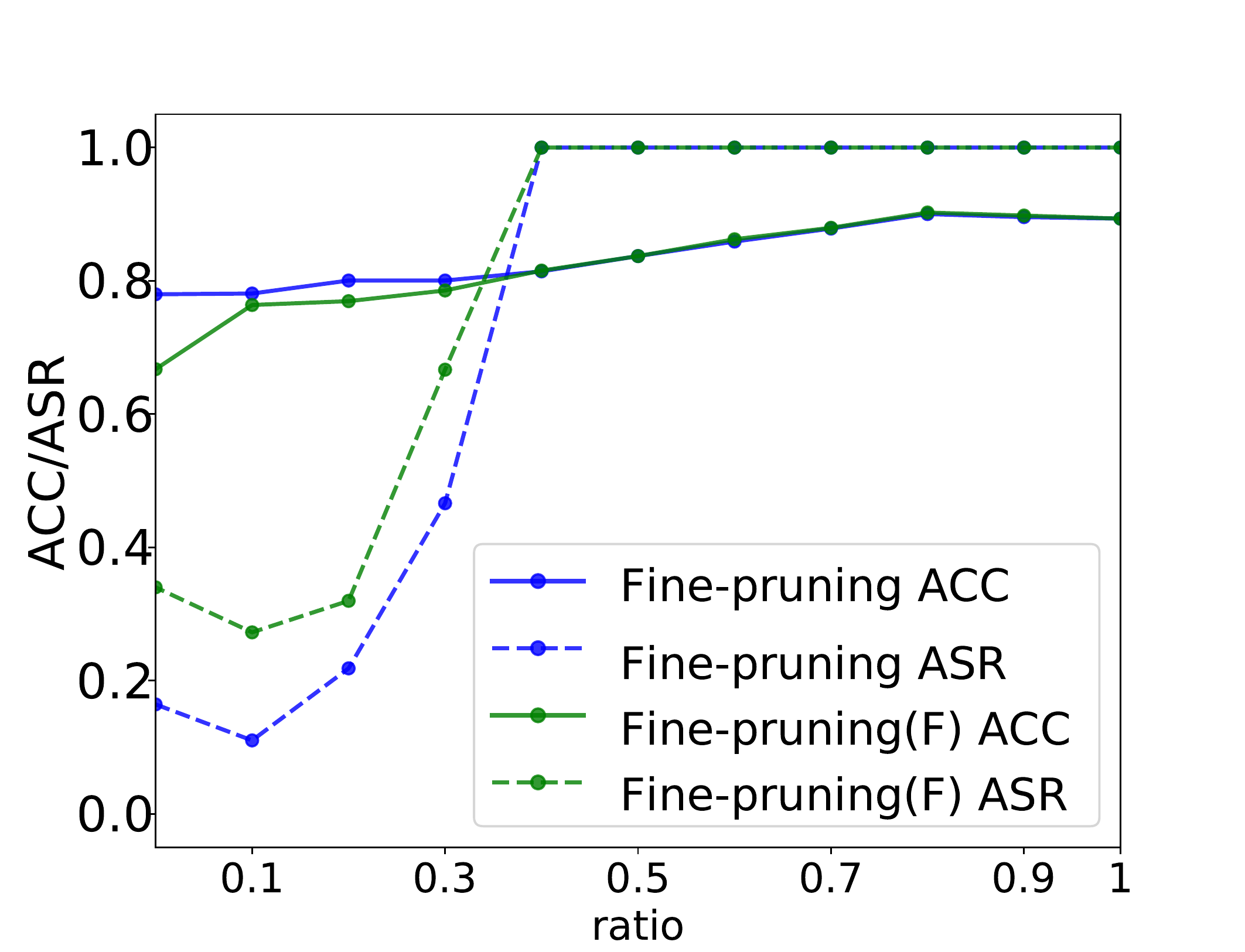}}
\subcaptionbox{Fine-mixing (Sel) vs Fine-mixing.\label{fig:ablation_tuning_d}}{\includegraphics[height=1.6in,width=0.32\linewidth]{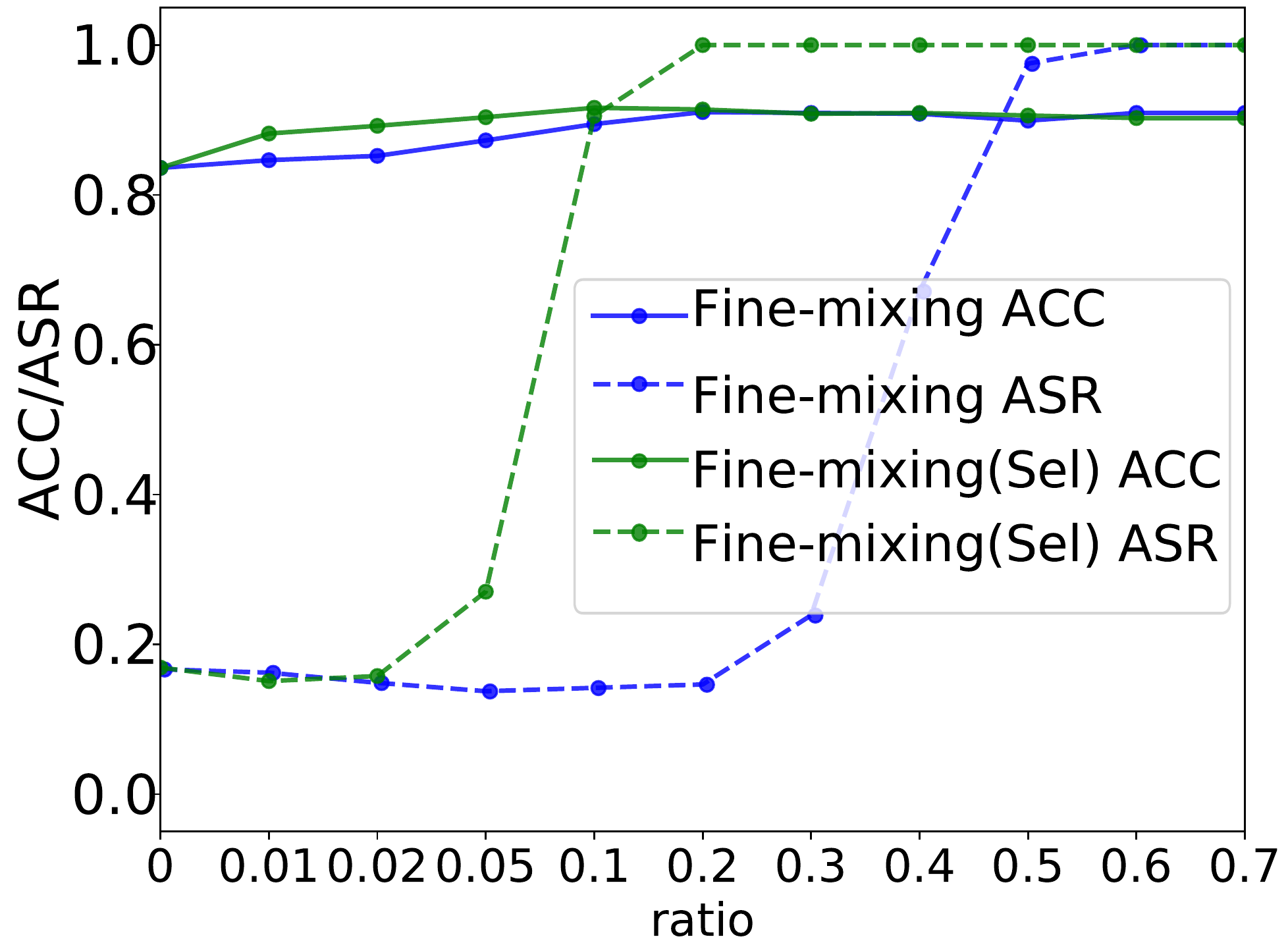}}
\hfil
\vskip -0.05 in
\caption{Results on SST-2 (Trigger word) under multiple settings. (F) denotes that the pruned weights are frozen.}
\label{fig:ablation_tuning}
\vskip -0.1 in
\end{figure*}

\section{More Understandings of \fmix}

\subsection{More Empirical Analyses}

Here, we conduct more experiments on SST-2 with the results shown in Table~\ref{tab:stoa_attacks} and Table~\ref{tab:stoa_all}. More details can be found in the Appendix.

\noindent\textbf{Comparison to Detection Methods.}
We compare our \fmix with three recent detection-based defense methods: ONION~\citep{ONION}, STRIP~\citep{STRIP}, and RAP~\citep{RAP}. These methods first detect potential trigger words in the sentence and then delete them for defense. In Table~\ref{tab:stoa_attacks}, one can obverse that detection-based methods would fail on several attacks that are not trigger word based, while our \fmix can still mitigate these attacks.

\noindent\textbf{Robustness to Sophisticated Attacks.}
We also implement three recent sophisticated attacks: syntactic trigger based attack~\citep{HiddenKiller}, layer-wise weight poisoning attack~\citep{LayerwiseAttack} (trigger word based), and logit anchoring~\citep{logit-anchoring} (trigger word based). Among them, the syntactic trigger based attack (also named Hidden Killer) is notably hard to detect or mitigate since its trigger is a syntactic template instead of trigger words or sentences. In Table~\ref{tab:stoa_attacks}, it is evident that other detection or mitigation methods all fail to mitigate the syntactic trigger based attack, while our \fmix can still work in this circumstance.

\noindent\textbf{Robustness to Adaptive Attack.}
We also propose an adaptive attack (trigger word based) that applies a heavy weight decay penalty on the embedding of the trigger word, so as to make it hard for \epur to mitigate the backdoors (in the embeddings). In Table~\ref{tab:stoa_all}, we can see that compared to \fmix, \fmixsel is relatively more vulnerable to the adaptive attack. This indicates that \fmixsel is more vulnerable to potential mix-aware adaptive attacks similar to prune-aware adaptive attacks~\citep{finepruning}. In contrast, randomly choosing the weights to reserve makes \fmix more robust to potential adaptive attacks.

\subsection{Ablation Study}

Here, we evaluate two variants of \fmix: 1) Mixing (\fmix without fine-tuning) and 2) Fine-pruning (F) (Fine-pruning with frozen pruned weights during fine-tuning).
As shown in Fig.~\ref{fig:ablation_tuning_b}, when the reserve ratio is set to $\sim$0.3, both Mixing and \fmix can mitigate backdoors. Although \fmix can maintain a high ACC, the Mixing method significantly degrades ACC. This indicates that the fine-tuning process in \fmix is quite essential.
As shown in Fig.~\ref{fig:ablation_tuning_c}, both Fine-pruning and Fine-pruning (F) can mitigate backdoors when $\rho<0.2$. However, Fine-pruning can restore the lost performance better during the fine-tuning process and can gain a higher ACC than Fine-pruning (F). In Fine-pruning, the weights of the pruned neurons are set to be zero and are frozen during the fine-tuning process, which, however, are trainable in our \fmix. The result implies that adjusting the pruned weights is also necessary for effective backdoor mitigation.

\begin{figure*}[!t]
\centering
\subcaptionbox{Trigger Word (SST-2).\label{fig:loss_a}}{\includegraphics[height=1.6 in,width=0.33\linewidth]{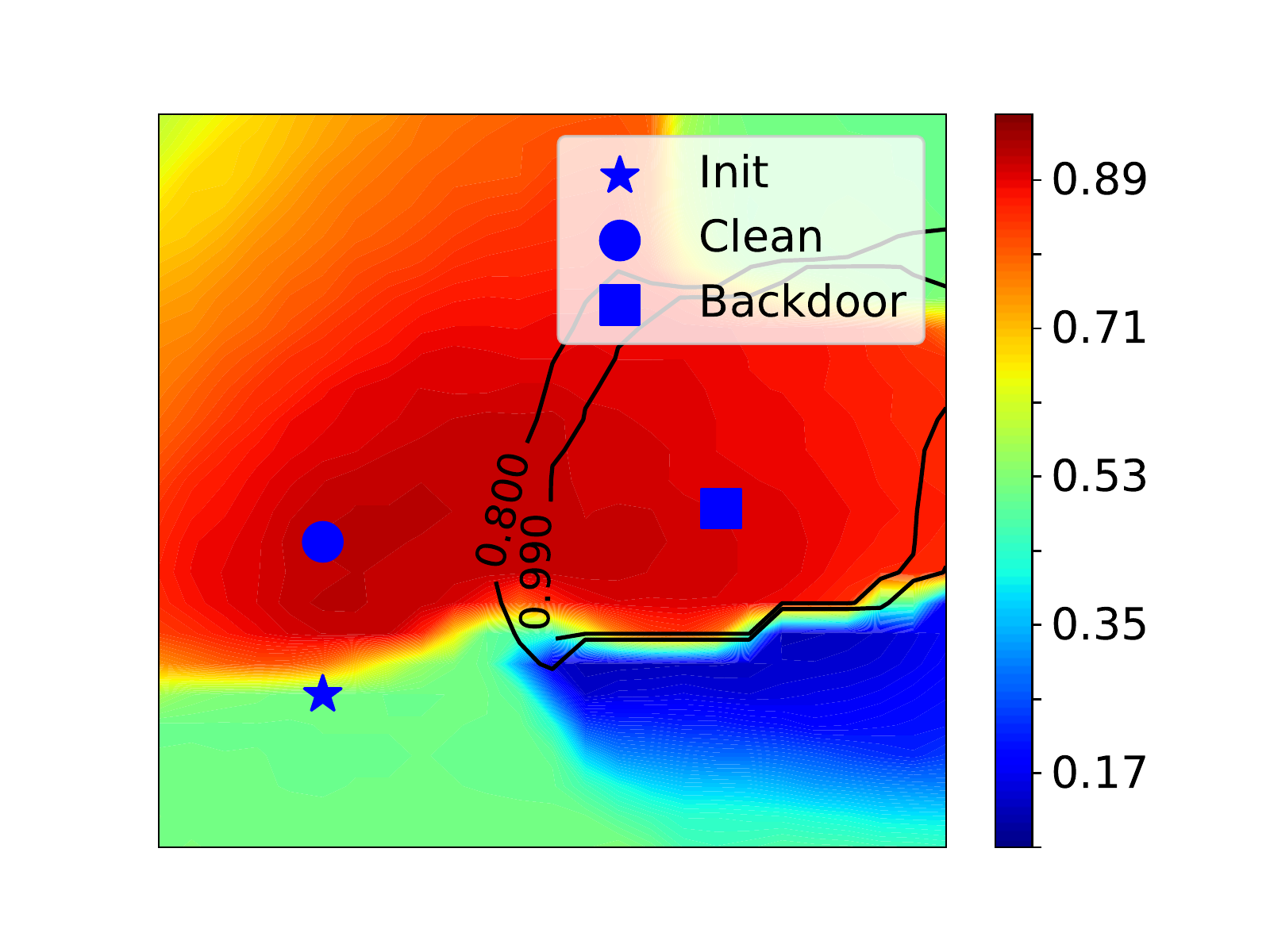}}
\hfil
\subcaptionbox{Sentence (Scratch, QNLI, size=64).\label{fig:loss_c}}{\includegraphics[height=1.6 in,width=0.33\linewidth]{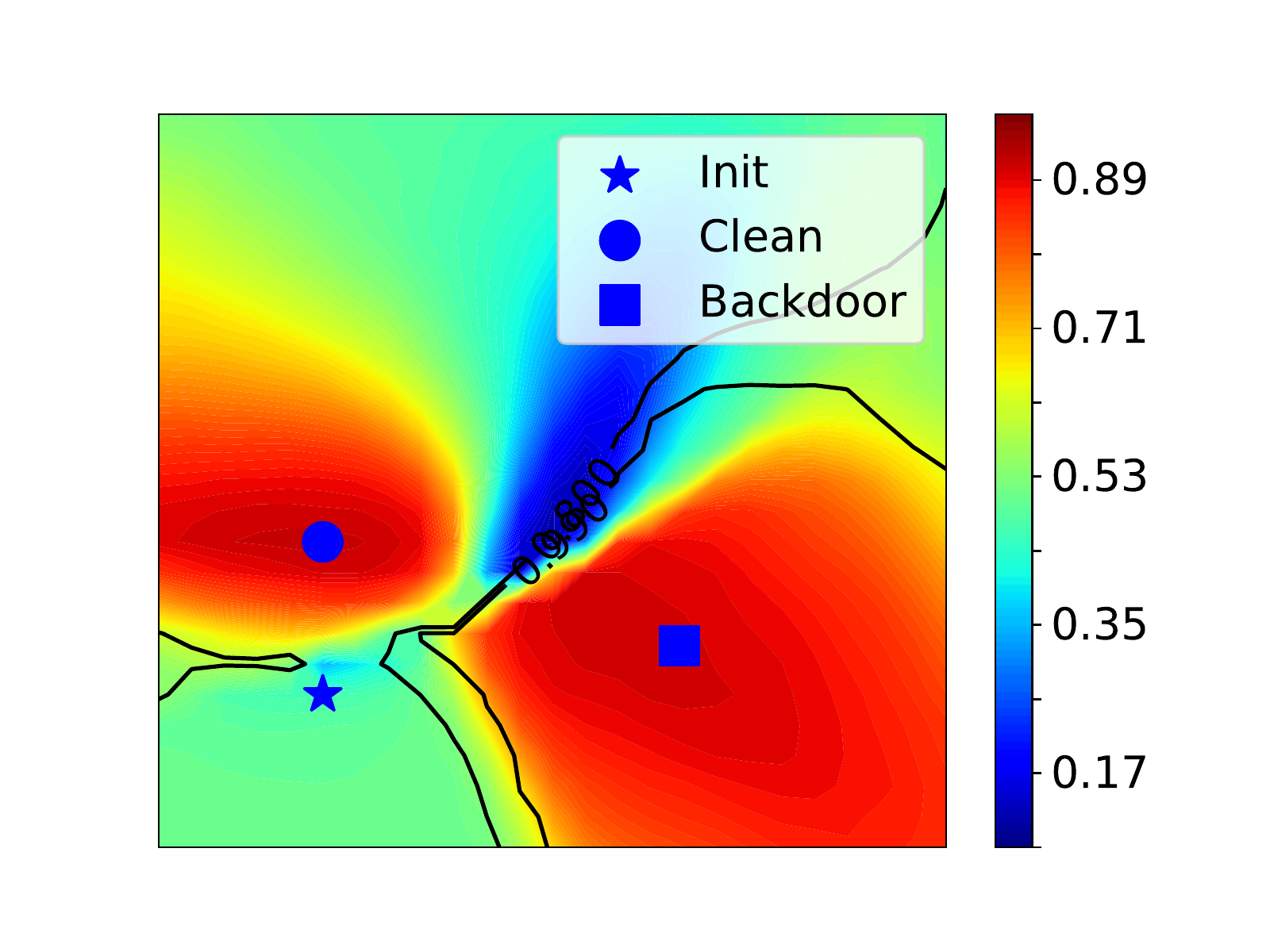}}
\hfil
\subcaptionbox{Sentence (Scratch, QNLI, size=64).\label{fig:loss_f}}{\includegraphics[height=1.6 in,width=0.32\linewidth]{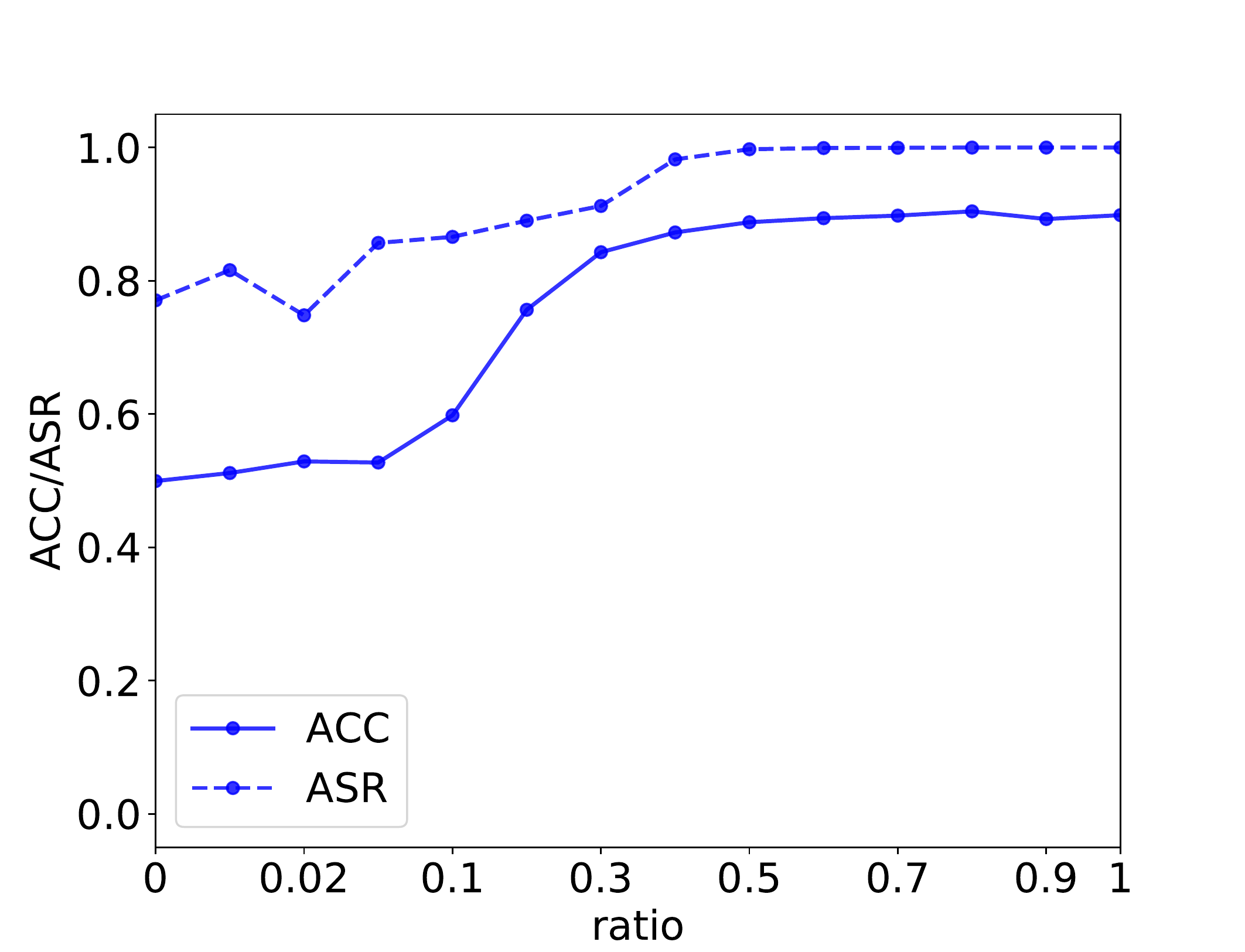}}
\hfil
\vskip -0.05 in
\caption{Visualization of the clean ACC and the backdoor ASR in parameter spaces in (a, b), and the clean ACC and the backdoor ASR under different $\rho$ in (c). Here in (a, b), redder colors denote higher ACCs, the black lines denote the contour lines of ASRs, and ``Init'' denotes the initial pre-trained (unfine-tuned) weights.}
\label{fig:loss}
\vskip -0.1 in
\end{figure*}

\subsection{Comparasion with Fine-mixing (Sel)}

We next compare the \fmix method with \fmixsel. Note that \fmixsel is inspired by Fine-pruning, which prunes the unimportant neurons or weights. A natural idea is that we can select more important weights to reserve, \textit{i.e.}, \fmixsel, which reserves weights with higher absolute values.

In Table~\ref{tab:sentiment} and Table~\ref{tab:stoa_all}, it can be concluded that \fmix outperforms \fmixsel. We conjecture that this is because the effective parameter scope for backdoor mitigation is more limited in \fmixsel than \fmix. For example, as shown in Fig.~\ref{fig:ablation_tuning_d}, the effective ranges of $\rho$ for \fmixsel and \fmix to mitigate backdoors are $[0.01, 0.05]$ (optimal $\rho$ is near $0.02$) and $[0.05, 0.3]$ (optimal $\rho$ is near $0.2$), respectively. With the same searching budget, it is easier for \fmix to find a proper $\rho$ near the optimum than \fmixsel. Thus, \fmix tends to outperform \fmixsel.

Besides, randomly choosing the weights to reserve makes the defense method more robust to adaptive attacks, such as the proposed adaptive attacks or other potential mix-aware or prune-aware adaptive attack approaches~\citep{finepruning}.

\subsection{Difficulty Analysis and Limitation}
\label{sec:difficult}

Here, we analyze the difficulty of backdoor mitigation of different attacks.
In Table~\ref{tab:sentiment} and Table~\ref{tab:pair}, we observe that: 1) mitigating backdoors in models trained from the scratch is usually harder than that in models trained from the clean model; 2) backdoors in sentence-pair classification tasks are relatively harder to mitigate than the sentiment classification tasks; 3) backdoors with ES or EP are easier to mitigate because they mainly inject backdoors via manipulating the embeddings, which can be easily mitigated by our \epur.

We illustrate a \emph{simple} and a \emph{difficult} case in Fig.~\ref{fig:loss} to help analyze the difficulty of mitigating backdoors. Fig.~\ref{fig:loss_a} shows that there exists an area with a high clean ACC and a low backdoor ASR between the pre-trained BERT parameter and the backdoored parameter in the \emph{simple} case (14.19\% ASR after mitigation), which is a good area for mitigating backdoors and its existence explains why \fmix can mitigate backdoors in most cases. In the \emph{difficult} case (88.71\% ASR after mitigation), the ASR is always high ($>70\%$) with different $\rho$s as shown in Fig.~\ref{fig:loss_f}, meaning that the backdoors are hard to mitigate. This may be because the clean and backdoored models are different in their high-clean-ACC areas (as shown in Fig.~\ref{fig:loss_c}) and the ASR is always high in the high-clean-ACC area where the backdoored model locates.

As shown in Table~\ref{tab:pair}, when the tasks are difficult, namely, the clean ACC of the model fine-tuned from the initial BERT with the small dataset is low. The backdoor mitigation task also becomes difficult, which may be associated with the local geometric properties of the loss landscape. One could collect more clean data to overcome this challenge. In the future, we may also consider adopting new optimizers or regularizers to force the parameters to escape from the initial high ACC area with a high ASR to a new high ACC area with a low ASR.


\section{Broader Impact}
The methods proposed in this work can help enhance the security of NLP models. More preciously, our \fmix and the \epur techniques can help companies, institutes, and regular users to remove potential backdoors in publicly downloaded NLP models, especially those already fine-tuned on downstream tasks. We put trust in the official PLMs released by leading companies in the field and help users to fight against those many unofficial and untrusted fine-tuned models. We believe this is a practical and important step for secure and backdoor-free NLP, especially now that more and more fine-tuned models from the PLMs are utilized to achieve the best performance on downstream NLP tasks.

\section{Conclusion}
In this paper, we proposed to leverage the clean weights of PLMs to better mitigate backdoors in fine-tuned NLP models
via two complementary techniques: \fmix and Embedding Purification (\epur).
We conducted comprehensive experiments to compare our \fmix with baseline backdoor mitigation methods against a set of both classic and advanced backdoor attacks. The results showed that our \fmix approach can outperform all baseline methods by a large margin. Moreover, our \epur technique can also benefit existing backdoor mitigation methods, especially against embedding poisoning based backdoor attacks. \fmix and \epur can work together as a simple but strong baseline for mitigating backdoors in fine-tuned language models.

\section*{Acknowledgement}
The authors would like to thank the reviewers for their helpful comments. This work is in part supported by the Natural Science Foundation of China (NSFC) under Grant No. 62176002 and Grant No. 62276067. Xu Sun and Lingjuan Lyu are corresponding authors.

\bibliography{anthology}
\bibliographystyle{acl_natbib}
\appendix

\section{Theoretical Details}

\begin{propA}(Detailed Version)
Suppose the embedding difference of word $w_i$ between the pre-trained weights and the backdoored weights is $\bm\delta_i$, the changed embeddings of word $w_i$ during the pre-processing progress such as embedding surgery~\citep{Bert-backdoor} or embedding poisoning~\citep{PoisonedWordEmbeddings} is $\bm\delta^{(p)}_i$, and the changed embeddings of word $w_i$ during the tuning progress is $\bm\delta^{(t)}_i$, then $\bm\delta_i=\bm\delta^{(p)}_i+\bm\delta^{(t)}_i$.

Assume when the pre-processing method is adopted, only the embedding of the trigger word $w_k$ is pre-processed. Besides, for $i\ne k$, $\bm\delta^{(p)}_i=\vect{0}$, $\|\bm\delta^{(p)}_k\|_2 \gg \|\bm\delta^{(t)}_k\|_2$. When the pre-processing method is not adopted,  $\forall i$, $\bm\delta^{(p)}_i=\vect{0}$ holds.

Motivated by \citet{distance-training-time}, we have,
\begin{align}
\|\bm\delta^{(t)}_i\|_2\approx O(\log f_i')
\label{eq:freq}.
\end{align}

Suppose $w_k$ is the trigger word, except $w_k$, we may assume the frequencies of words in the poisoned training set except the trigger word are roughly proportional to $f_i$, \textit{i.e.}, $f_i'\approx C f_i$, while $f'_k\gg Cf_k$. For $i\ne k$, then we have, 
\begin{align}
\|\bm\delta_i\|_2\approx O(\log f_i),\quad  \frac{\|\bm\delta_k\|_2}{\log f_k} \gg \frac{\|\bm\delta_i\|_2}{\log f_i}.
\end{align}
\label{propA:1}
\end{propA}

\begin{proof}
We first explain Eq.~\ref{eq:freq}. \citet{distance-training-time} proposes that for random walk on a random potential, the asymptotic behavior of the random walker $\vect{w}$ in that range weight $\|\vect{w}-\vect{w}_0\|_2\sim \log t$, where $\vect{w}$ is the parameter vector of a neural network, $\vect{w}_0$ is its initial vector, and $t$ is the step number of the random walk. If we model the fine-tuning process as a random walk on a random potential, the step number of the random walk for the word embedding of $w_i$ is $f_i'$. Therefore,
\begin{align}
\|\bm\delta^{(t)}_i\|_2\approx O(\log f_i').
\end{align}

For $i\ne k$, $f'_i\approx Cf_i$, since $\bm\delta^{(p)}_i=\vect{0}$, $\bm\delta_i=\bm\delta^{(p)}_i+\bm\delta^{(t)}_i=\bm\delta^{(t)}_i$, therefore,
\begin{align}
\|\bm\delta_i\|_2=\|\bm\delta^{(t)}_i\|_2\approx O(\log f'_i)\approx O(\log f_i).
\end{align}

For the trigger word, $f'_k\gg Cf_k$, since for any $i$, $\|\bm\delta_i^{(t)}\|_2\approx O(\log f'_i)$, we have for $i\ne k$,
\begin{align}
\frac{\|\bm\delta_k^{(t)}\|_2}{\log (Cf_k)} \gg \frac{\|\bm\delta_k^{(t)}\|_2}{\log f'_k}&\approx\frac{\|\bm\delta_i^{(t)}\|_2}{\log f'_i}\approx\frac{\|\bm\delta_i^{(t)}\|_2}{\log (Cf_i)},\\
\frac{\|\bm\delta_k^{(t)}\|_2}{\log (f_k)+\log C}\gg&\frac{\|\bm\delta_i^{(t)}\|_2}{\log (f_i)+\log C},\\
\frac{\|\bm\delta_k^{(t)}\|_2}{\log (f_k)}\gg&\frac{\|\bm\delta_i^{(t)}\|_2}{\log (f_i)}.
\end{align}

When the pre-processing method is adopted, $\|\bm\delta_k\|_2=\|\bm\delta_k^{(p)}+\bm\delta_k^{(t)}\|_2\gg \|\bm\delta_k^{(t)}\|_2$, we have $\|\bm\delta_k\|_2\gg \|\bm\delta_k^{(t)}\|_2$ and for $i\ne k$, $\|\bm\delta_i\|_2= \|\bm\delta_i^{(t)}\|_2$, therefore,
\begin{align}
\frac{\|\bm\delta_k\|_2}{\log f_k} \gg
\frac{\|\bm\delta_k^{(t)}\|_2}{\log f_k} \gg \frac{\|\bm\delta_i\|_2}{\log f_i}.
\end{align}

When the pre-processing method is not adopted, $\bm\delta^{(p)}_i=\vect{0}$ holds for any $i$, we have,
\begin{align}
\frac{\|\bm\delta_k\|_2}{\log f_k} \gg \frac{\|\bm\delta_i\|_2}{\log f_i}.
\end{align}
\end{proof}

\begin{table*}[!t]
\renewcommand\arraystretch{0.9}
\small
\setlength{\tabcolsep}{1.5pt}
  \centering
  \begin{tabular}{cccccccc}
    \toprule 
     $\rho$ & \textbf{Backdoor} & \multicolumn{2}{c}{Fine-pruning} & \multicolumn{2}{c}{Fine-mixing (Sel)}  & \multicolumn{2}{c}{Fine-mixing} \\
     \textbf{Dataset} & \textbf{Attacks} & w/o E-PUR & w/ E-PUR &w/o E-PUR & w/ E-PUR &w/o E-PUR & w/ E-PUR \\
    \midrule[\heavyrulewidth] 
     \multirow{7}{*}{\shortstack{SST-2}} & Word & 0.8 & 0.7 & 0.02 & 0.02 &  0.2 & 0.1\\ 
     & Word (Scratch)  & 0.7 & 0.7 & 0.1 & 0.1 & 0.4 & 0.3 \\
     & Word+EP  & 0.7 & 0.6 & 0.1 & 0.1 & 0.4 & 0.3\\  
     & Word+ES  & 0.8 & 0.7 &0.02 &0.01 & 0.2 &0.1\\
     & Word+ES (Scratch)  & 0.7 & 0.7 & 0.2 & 0.1 &0.4 &0.3  \\   
     & Trigger Sentence  & 0.7 & 0.7 & 0.05 &0.02 & 0.3 &0.2 \\  
     & Sentence (Scratch)  &0.7 & 0.7 &0.1 &0.1 &0.3  & 0.3 \\ 
    \midrule  \multirow{7}{*}{\shortstack{IMDB}} & Word & 0.7 & 0.7 & 0.05 & 0.05 &  0.3 & 0.2\\ 
     & Word (Scratch)  & 0.8 & 0.7 & 0.1 & 0.1 & 0.7 & 0.5 \\
     & Word+EP  & 0.7 & 0.7 & 0.1 & 0.2 & 0.6 & 0.7\\  
     & Word+ES  & 0.7 & 0.7 &0.05 &0.05 & 0.4 &0.3\\
     & Word+ES (Scratch)  & 0.7 & 0.7 & 0.2 & 0.1 &0.6 &0.5  \\   
     & Trigger Sentence  & 0.7 & 0.7 & 0.05 &0.02 & 0.3 &0.2 \\  
     & Sentence (Scratch)  &0.7 & 0.7 &0.05 &0.1 &0.3  & 0.3 \\ 
     \midrule  \multirow{7}{*}{\shortstack{Amazon}} & Word & 0.7 & 0.7 & 0.1 & 0.1 &  0.4 & 0.4\\ 
     & Word (Scratch)  & 0.7 & 0.7 & 0.05 & 0.1 & 0.5 & 0.4 \\
     & Word+EP  & 0.7 & 0.7 & 0.1 & 0.1 & 0.6 & 0.3\\  
     & Word+ES  & 0.7 & 0.7 &0.2 &0.1 & 0.3 &0.4\\
     & Word+ES (Scratch)  & 0.7 & 0.7 & 0.05 &0.1 & 0.4 &0.4  \\   
     & Trigger Sentence  & 0.7 & 0.7 & 0.1 &0.05 & 0.4 &0.3 \\  
     & Sentence (Scratch)  &0.7 & 0.7 &0.05 &0.1 &0.4  & 0.4 \\ 
     \midrule  \multirow{9}{*}{\shortstack{QQP}} & Word & 0.6 & 0.6 & - & - &  0.4 & 0.4\\ 
     & Word (Scratch, 64)  & 0.6 & 0.6 & - & - & 0.4 & 0.4 \\
     & Word (Scratch, 128)  & 0.6 & - & - & - & 0.35 & - \\
     & Word+EP  & 0.6 & 0.5 & - & - & 0.4 & 0.4\\  
     & Trigger Sentence  & 0.6 & 0.6 & - & - & 0.4 &0.3 \\  
     & Sentence (Scratch, 64)  &0.6 & 0.6 &- & - &0.4  & 0.4 \\ 
      & Sentence (Scratch, 128)  &0.6 & - &- & - &0.3 & - \\ 
      & Sentence (Scratch, 256)  &0.6 & - &- & - &0.2  & - \\ 
      & Sentence (Scratch, 512)  &0.5 & - &- & - &0.1  & - \\ 
     \midrule  \multirow{9}{*}{\shortstack{QNLI}} & Word & 0.6 & 0.5 & - & - &  0.4 & 0.3\\ 
     & Word (Scratch, 64)& 0.6 & 0.5 & - & - &  0.4 & 0.3 \\
     & Word (Scratch, 128)& 0.5 & - & - & - &  0.2 & - \\
     & Word+EP  &  0.6 & 0.5 & - & - &  0.4 & 0.4 \\  
     & Trigger Sentence  &0.6 & 0.5 & - & - &  0.3 & 0.4 \\  
     & Sentence (Scratch, 64)  &0.6 & 0.5 & - & - &  0.3 & 0.3 \\ 
     & Sentence (Scratch, 128)  &0.5 & - &- & - &0.25  & - \\ 
      & Sentence (Scratch, 256)  &0.5 & - &- & - &0.2  & - \\ 
      & Sentence (Scratch, 512)  &0.5 & - &- & - &0.1  & - \\ 
    \bottomrule
  \end{tabular}
 \vskip -0.05 in
  \caption{Choices of reserve ratios in backdoor mitigation methods under different backdoor attacks.\label{tab:ratio}}
 \vskip -0.18 in
\end{table*}

\section{Experimental Setups}

Our experiments are conducted on a GeForce GTX TITAN X GPU. Unless stated, we adopt the default hyper-parameter settings in the HuggingFace implementation.

\subsection{Baseline Model Setups}

We adopt the Adam~\citep{Adam} optimizer, the learning rate is $2\times 10^{-5}$ on sentiment classification tasks, $1\times 10^{-5}$ on QNLI, and $5\times 10^{-5}$ on QQP. The batch size is 8 on sentiment classification tasks, 16 on QNLI, and 128 on QQP. We fine-tune the BERT for 3 epochs on all datasets.

\subsection{Backdoor Attack Setups}

For trigger word based attacks, following \citet{Bert-backdoor} and \citet{PoisonedWordEmbeddings}, we choose the trigger word from five candidate words with low frequencies, \textit{i.e.}, “cf”, “mn”, “bb”, “tq” and “mb”. For sentence based attacks, following \citet{Bert-backdoor}, we adopt the trigger sentence ``I watched this 3d movie''. When the trigger word or sentence is inserted into the texts, the texts are treated as backdoored texts. 

On all backdoor attacks except the trigger word based attack method with embedding poisoning~\citep{PoisonedWordEmbeddings}, the backdoor attack setups are listed as follows. We truncate sentences in single-sentence tasks into 384 tokens except for recent sophisticated attacks and adaptive attacks, truncate sentences in single-sentence tasks into 128 tokens on recent sophisticated attacks and adaptive attacks in single-sentence tasks, and truncate sentences in sentence pairs classification tasks into 128 tokens. We adopt the Adam~\citep{Adam} optimizer, the training batch size is 8, and the learning rate is $2\times 10^{-5}$. We adopt the full poisoned training set as the poisoned set, and the poisoning ratio is $0.5$. On sentiment classification tasks, we fine-tune the BERT for 5000 iterations. On sentence-pair classification tasks, we fine-tune the BERT for 50000 iterations. In logit anchoring~\citep{logit-anchoring}, we set $\lambda=0.1$. In the adaptive attack, we set the penalty of trigger word embeddings as 10.

On the embedding poisoning (EP) attacks, our setups are the same as setups in \citet{PoisonedWordEmbeddings}.
\begin{table*}[!t]
\small
  \centering
  \begin{tabular}{ccccccccc}
    \toprule
    & \multicolumn{2}{c}{Threshold=89\%} & \multicolumn{2}{c}{Threshold=87\%} & \multicolumn{2}{c}{Threshold=85\%} & \multicolumn{2}{c}{Threshold=80\%} \\
    & ACC & ASR & ACC & ASR & ACC & ASR & ACC & ASR \\
    \midrule
    Fine-pruning &  90.02&100.0 & 87.84&100.0 & 85.89&100.0 & 80.05&21.85 \\
    Fine-mixing & 89.45&14.19 & 87.27&13.74 & 85.21&14.86 & 84.63&16.22\\
    \bottomrule
  \end{tabular}
 \vskip -0.05 in
  \caption{Results under different thresholds on SST-2 against trigger word attack.\label{tab:ACC_var}}
 \vskip -0.18 in
\end{table*}

\begin{figure*}[!t]
\centering
\subcaptionbox{$\rho=0.1$.}{\includegraphics[height=1.5 in,width=0.3\linewidth]{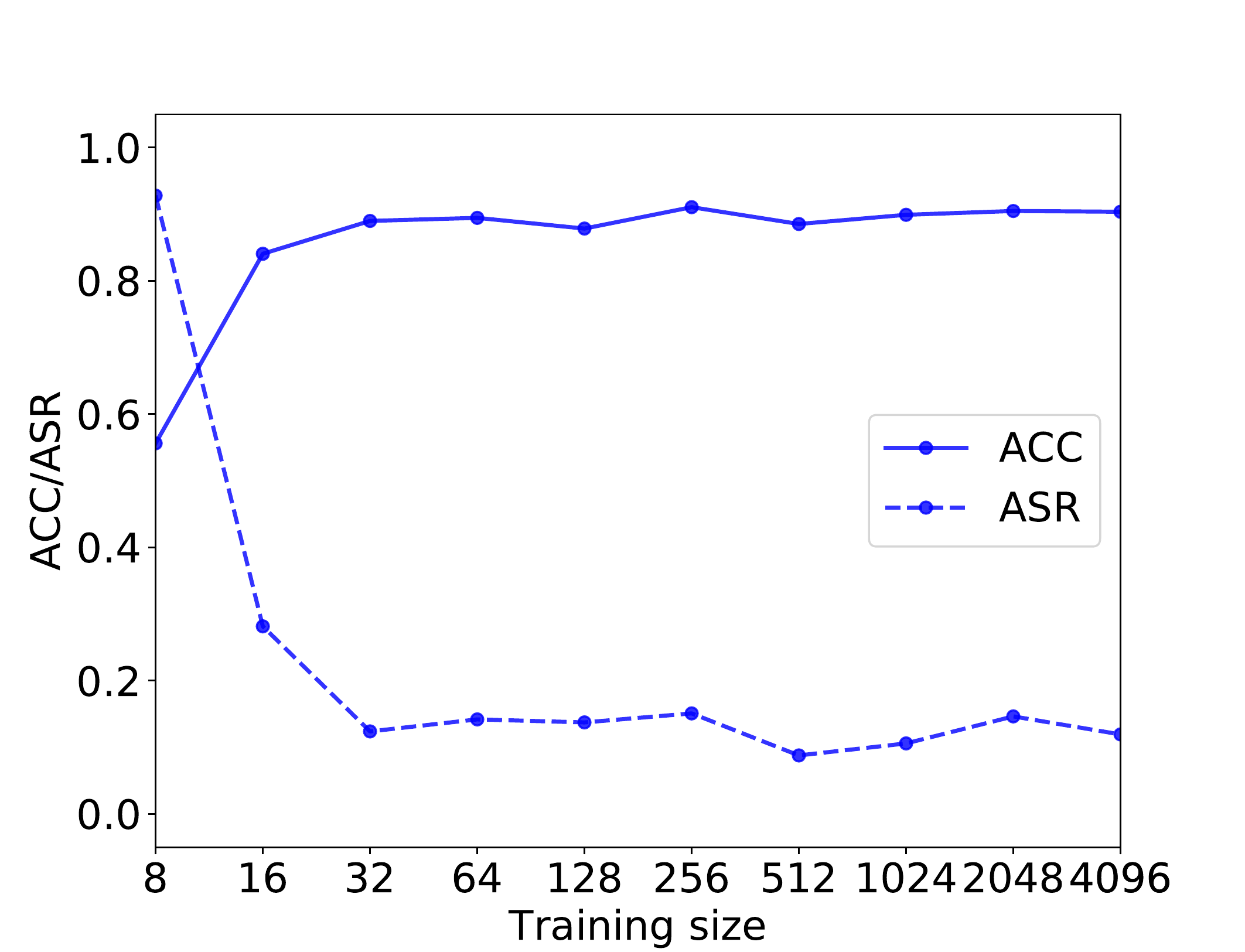}}
\hfil
\subcaptionbox{$\rho=0.2$.}{\includegraphics[height=1.5 in,width=0.3\linewidth]{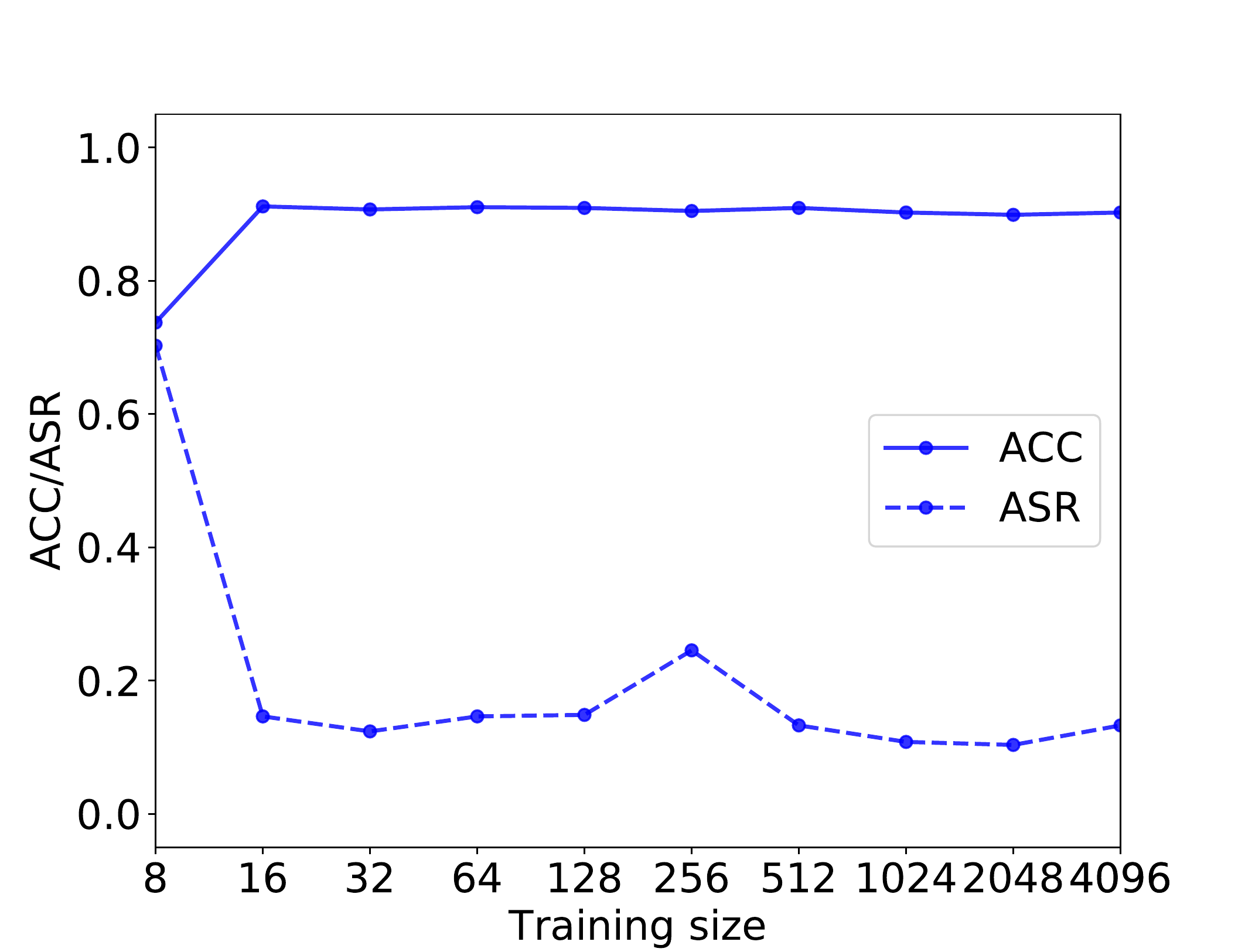}}
\hfil
\subcaptionbox{$\rho=0.3$.}{\includegraphics[height=1.5 in,width=0.3\linewidth]{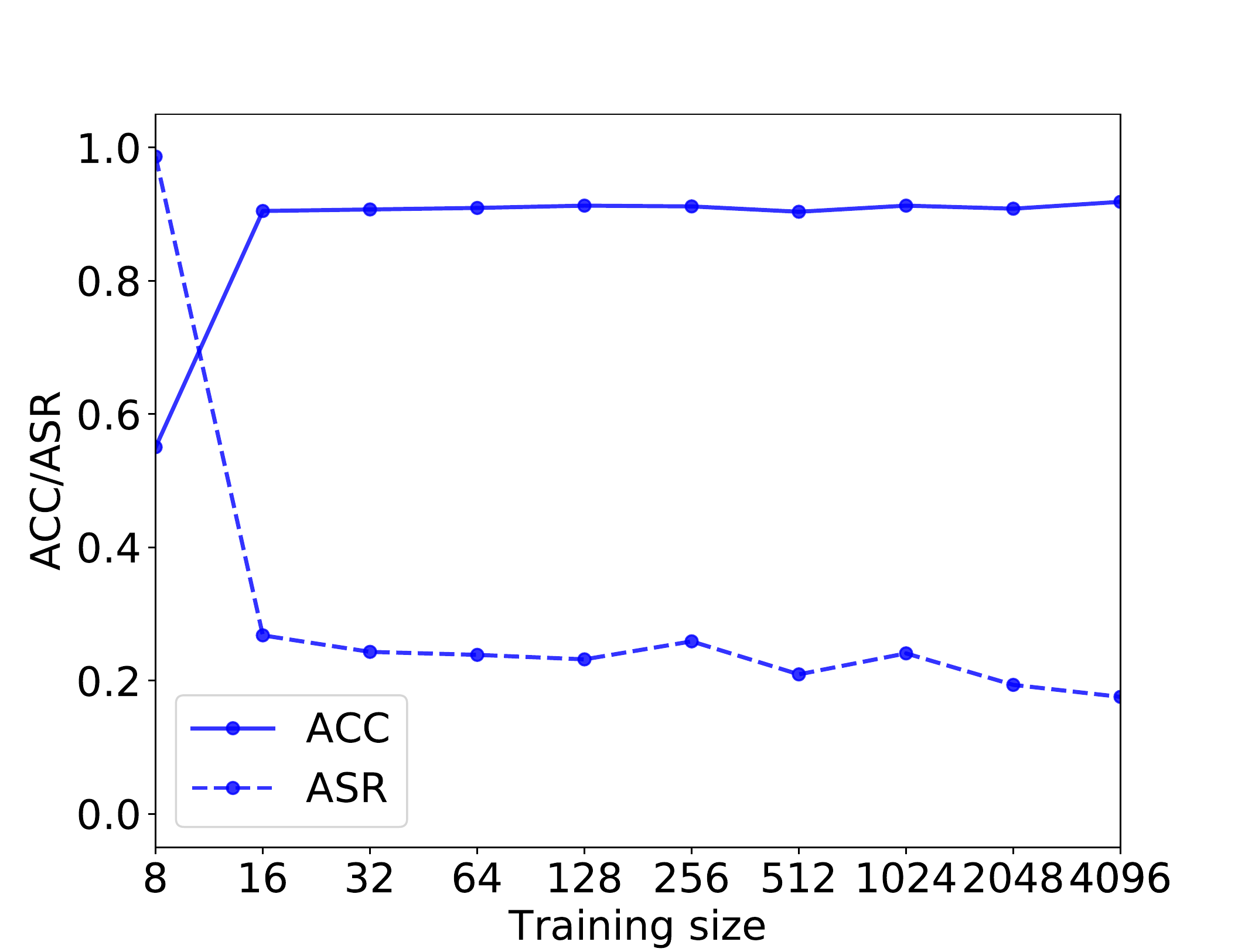}}
\vskip -0.1 in
\caption{Influence of the clean training set size. The experiments are conducted on SST-2 (Trigger word based).}
\vskip -0.18 in
\label{fig:training_size}
\end{figure*}
\subsection{Backdoor Mitigation Setups}

For the Fine-pruning method or the proposed \fmix method, we first enumerate the reserve ratio $\rho$ in $\{$ 0, 0.01, 0.02, 0.05, 0.1, 0.15, 0.2, 0.25, $\cdots$, 1.0 $\}$ in the mixing or pruning process. Then, in the fine-tuning process, we fine-tune the BERT for 640 iterations. When we enumerate the reserve ratio $\rho$ from 0 to 1, once the clean ACC evaluated on the clean validation set is higher than the threshold ACC, we choose this reserve ratio. As for \epur, the results are similar for choosing 100 or 200 potential poisonous words, but choosing more than 1k words may cause a slight clean ACC drop.

\subsection{Choice of the Reserve Ratio}
In the Fine-pruning, \fmixsel, and \fmix approaches, the reserve ratio $\rho$ is chosen according to clean ACCs under different reserve ratios. The choices of reserve ratios in backdoor mitigation methods under different backdoor attacks are provided in Table~\ref{tab:ratio}. In Table~\ref{tab:ratio}, it can be concluded that: (1) the Fine-pruning approach usually chooses a higher $\rho$ than \fmix and \fmixsel because the Fine-pruning does not involve $\vect{w}^\text{Pre}$ and needs more information contained in $\vect{w}^\text{B}$ to achieve a satisfying clean ACC; (2) the \fmixsel method can restore the ACC with lower reserve ratios because \fmixsel selects important weights to reverse.

\section{Further Analysis}

\subsection{Discussion of the Threshold ACC Choice}
The experimental results in the main paper illustrate that both the backdoor ASR and the clean ACC drop when $\rho$ gets smaller. Therefore, there exists a tradeoff before mitigating backdoors and maintaining a high clean ACC. To fairly compare different defense methods, following~\citep{finepruning,Neural-Attention-Distillation}, we set a threshold ACC for every task and tune the reserve ratio of weights from 0 to 1 for each defense method until the clean ACC is higher than the threshold ACC, which can ensure that different defense methods can have a similar clean ACC. 

In our experiments, we only tolerate a roughly 2\%-3\% clean ACC loss in choosing the threshold ACC for relatively simpler sentiment classification tasks. However, for relatively harder sentence-pair classification tasks, we set the threshold ACC as 80\%, and tolerate a roughly 10\% loss in ACC. Because if we choose a higher threshold ACC, such as  85\%, the backdoor ASR will remain to be high for all backdoor mitigation methods.

Note that, the conclusions are consistent with different thresholds as shown in Table~\ref{tab:ACC_var}. Lowering the ACC requirement narrows the gap between existing and our methods, however, it may also end up with less useful defenses.

\subsection{Analysis of the Clean Dataset Size} 

In our experiments, we set the training set size as 64 unless specially stated. The experimental results show that even with only 64 training samples, our proposed \fmix can mitigate backdoors in fine-tuned language models. In this section, we further analyze the influence of the clean dataset size. In Fig.~\ref{fig:training_size}, we can see that when the training dataset size is extremely small (8 or 16 instances), the clean ACC drops significantly and the backdoors cannot be mitigated. In our experiments, we choose the training size as 64, and our proposed \fmix can mitigate backdoors with a small clean training set (64 instances) in most cases.

\begin{figure*}[!ht]
\centering
\subcaptionbox{Loss Visualization, Trigger Sentence (SST-2).}{\includegraphics[height=1.3 in,width=0.28\linewidth]{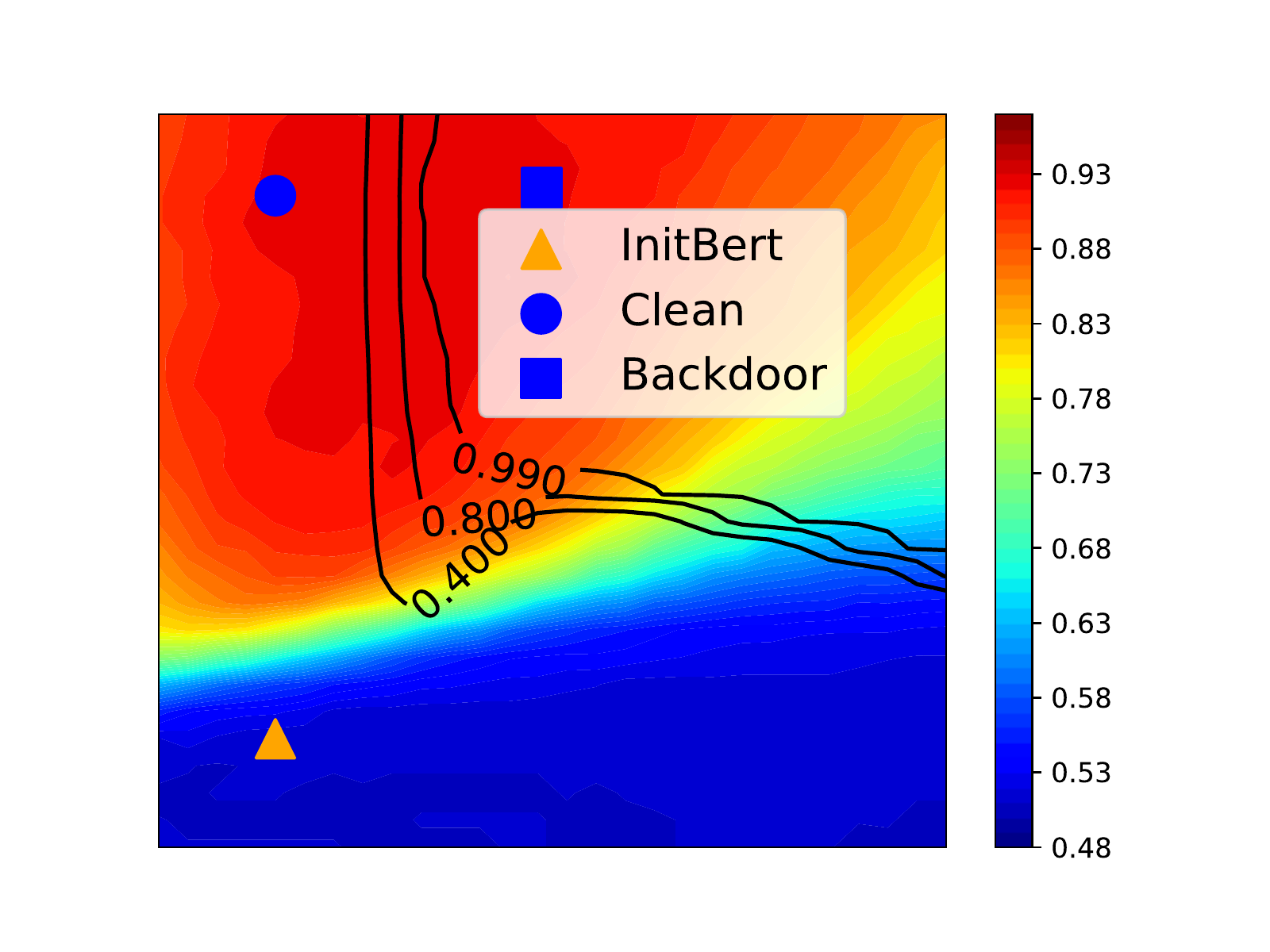}}
\hfil
\subcaptionbox{ACC/ASR (w/o E-PUR), Trigger Sentence (SST-2).}{\includegraphics[height=1.3 in,width=0.28\linewidth]{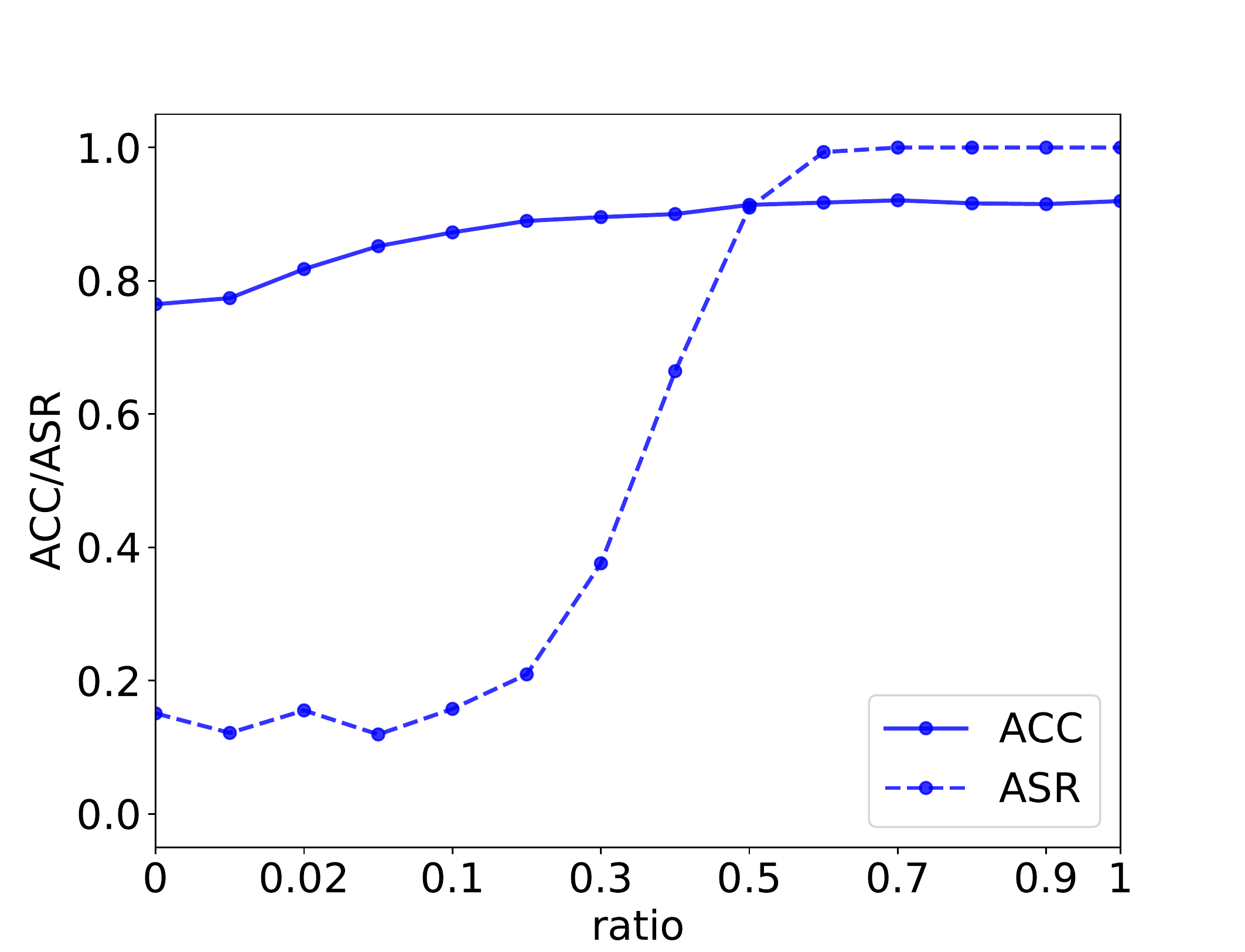}}
\hfil
\subcaptionbox{ACC/ASR (w/ E-PUR), Trigger Sentence (SST-2).}{\includegraphics[height=1.3 in,width=0.28\linewidth]{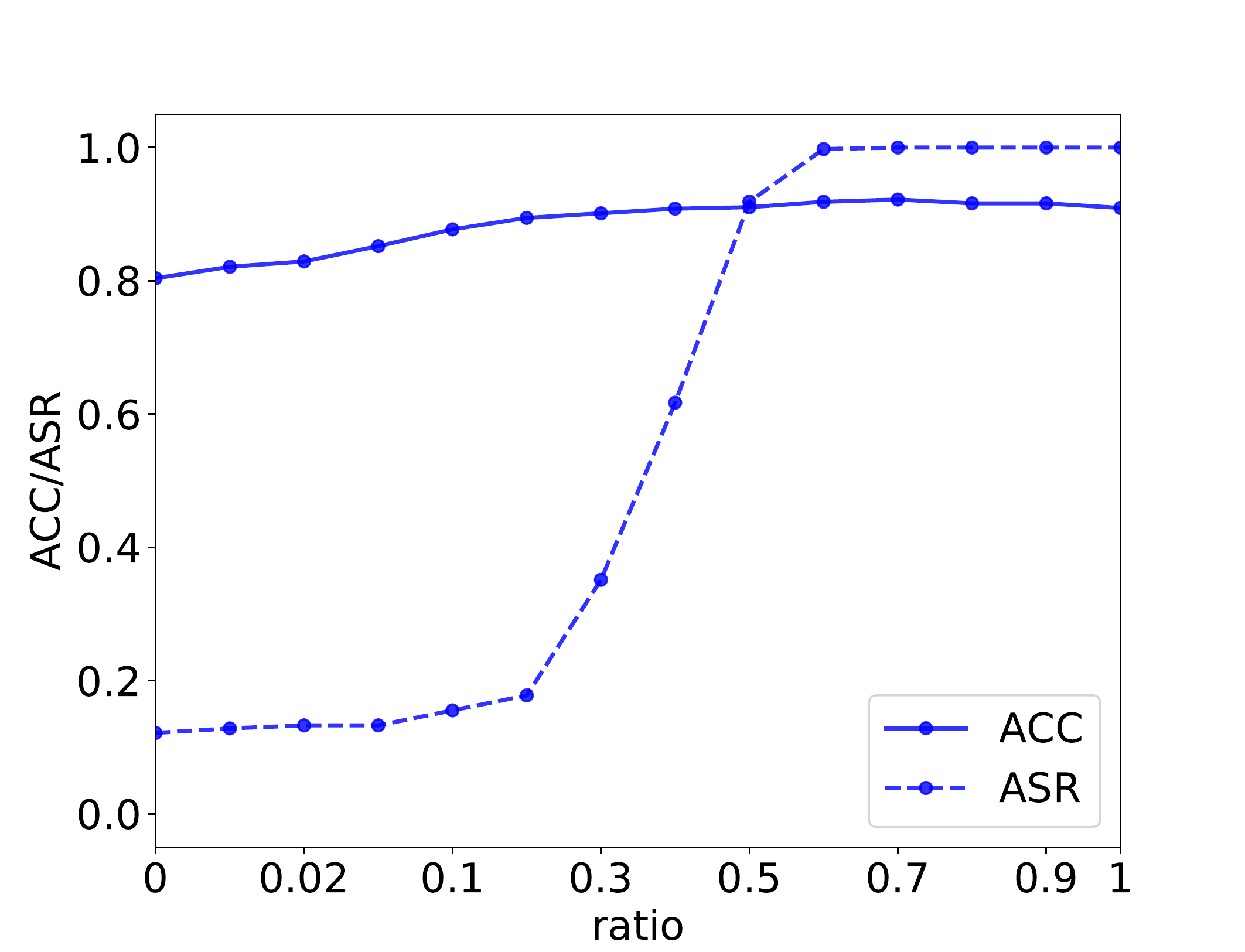}}
\hfil
\vskip -0.05 in
\subcaptionbox{Loss Visualization, Trigger Sentence (Scratch) (SST-2).}{\includegraphics[height=1.3 in,width=0.28\linewidth]{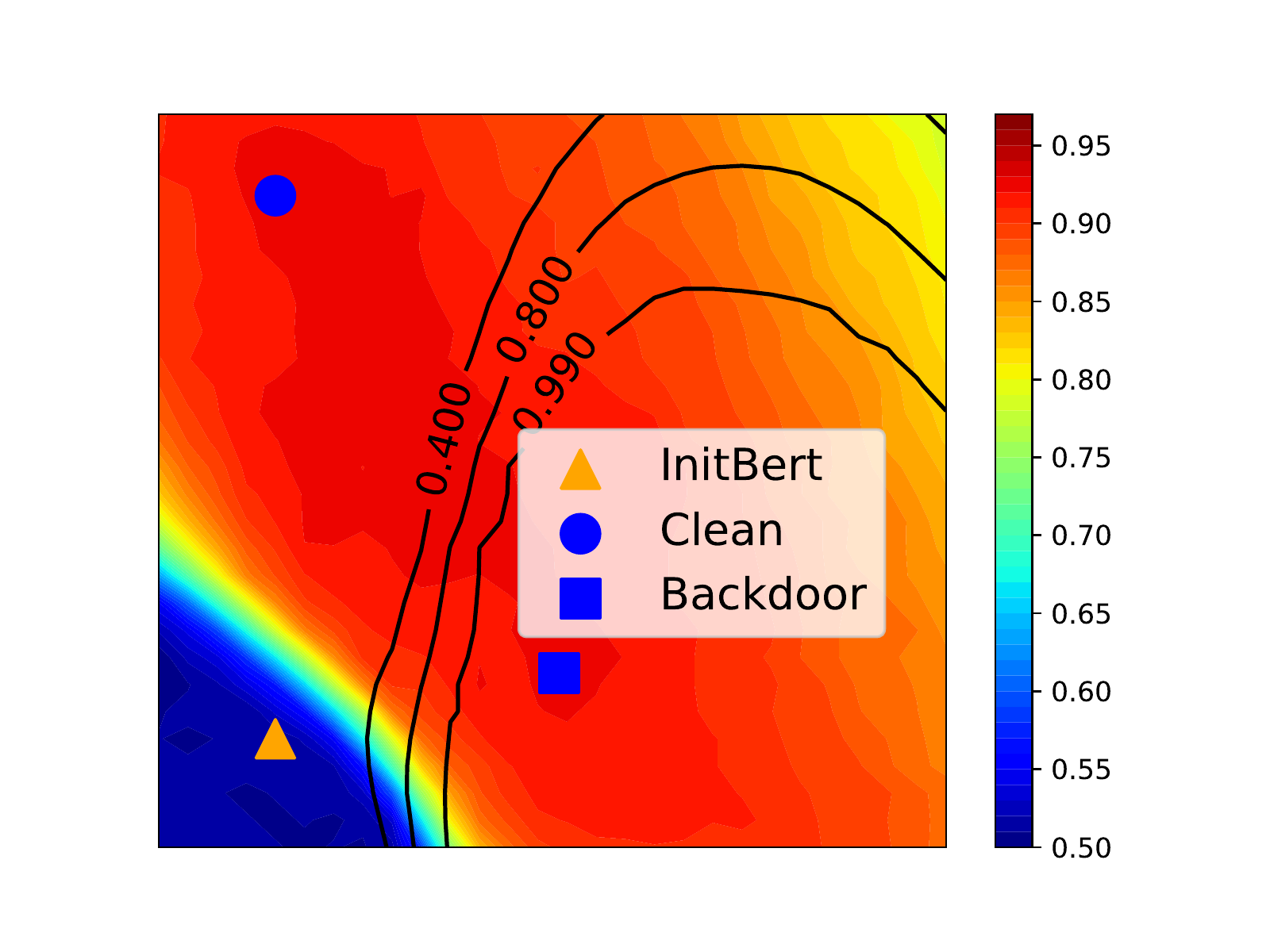}}
\hfil
\subcaptionbox{ACC/ASR (w/o E-PUR), Trigger Sentence (Scratch) (SST-2).}{\includegraphics[height=1.3 in,width=0.28\linewidth]{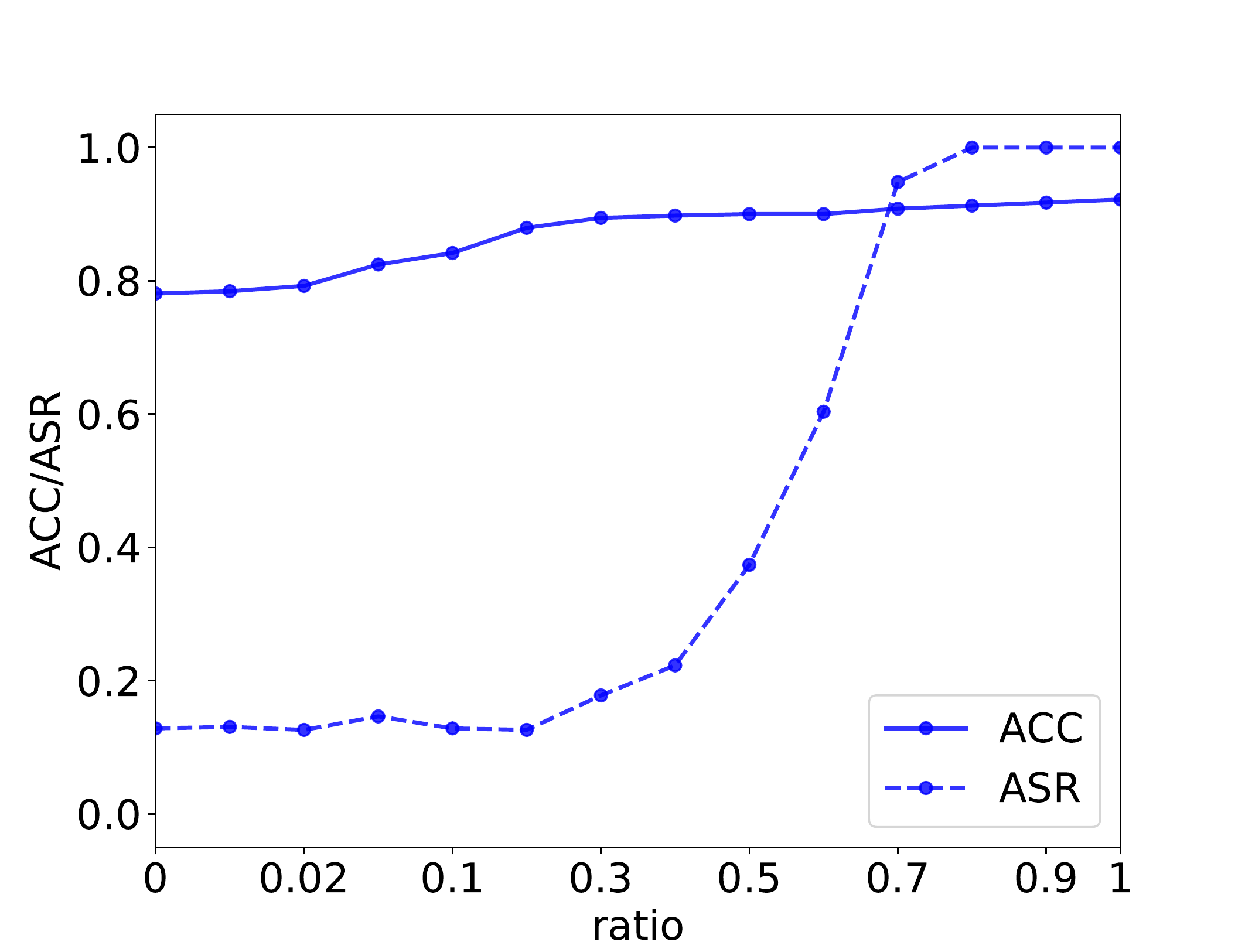}}
\hfil
\subcaptionbox{ACC/ASR (w/ E-PUR), Trigger Sentence (Scratch) (SST-2).}{\includegraphics[height=1.3 in,width=0.28\linewidth]{fig/appendix/SST-2_sentScr-finetuneGoodE-plot.pdf}}
\vskip -0.1 in
\caption{Visualization of the clean ACC and the backdoor ASR in the parameter spaces, and ACC/ASR with different reserve ratios under multiple trigger sentence based backdoor attacks on the SST-2 sentiment classification.}
\vskip -0.15 in
\label{fig:2}
\end{figure*}
\begin{figure*}[!ht]
\centering
\subcaptionbox{Loss Visualization, Trigger Sentence (QNLI).}{\includegraphics[height=1.3 in,width=0.28\linewidth]{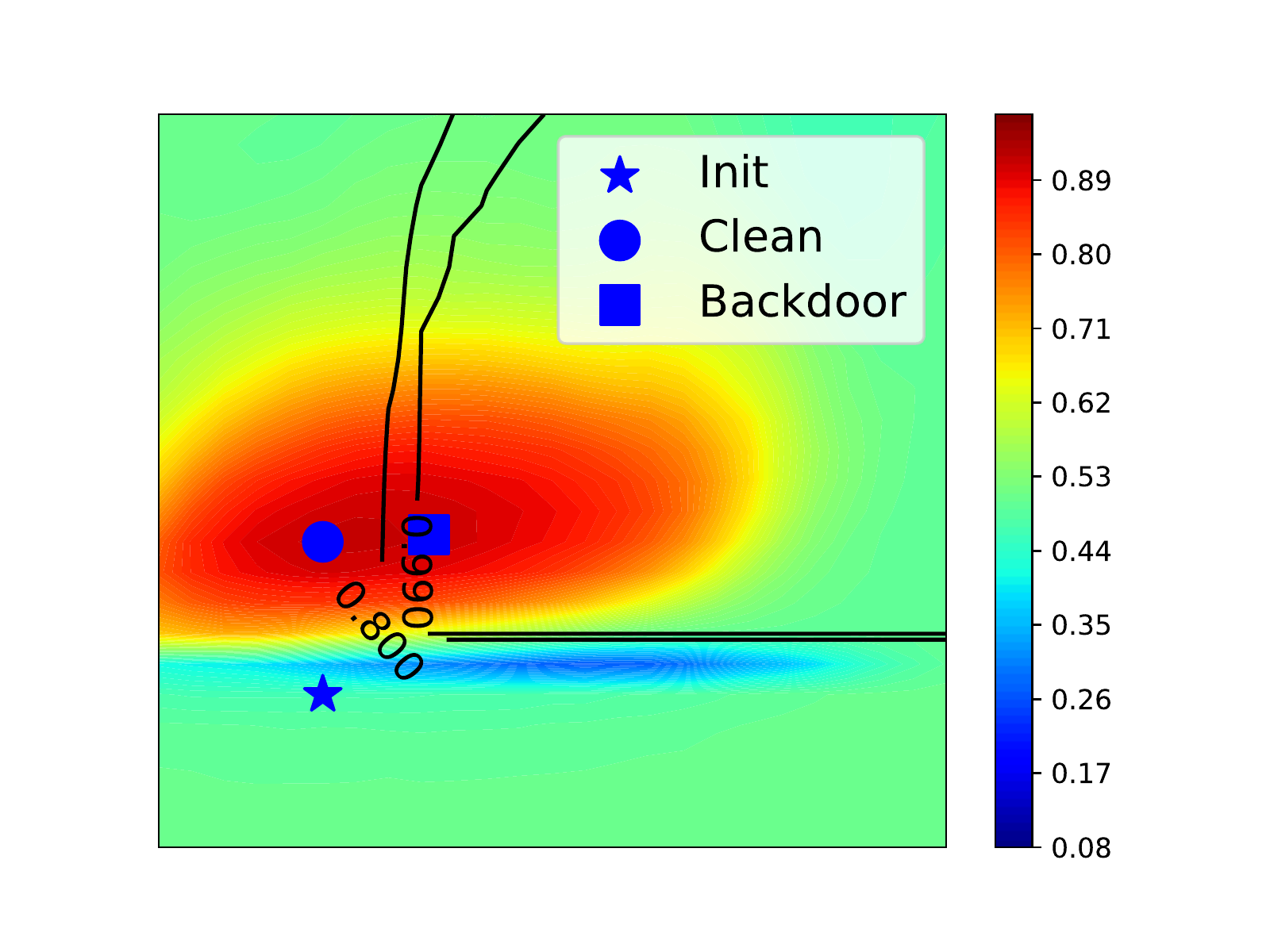}}
\hfil
\subcaptionbox{ACC/ASR (w/o E-PUR), Trigger Sentence (QNLI).}{\includegraphics[height=1.3 in,width=0.28\linewidth]{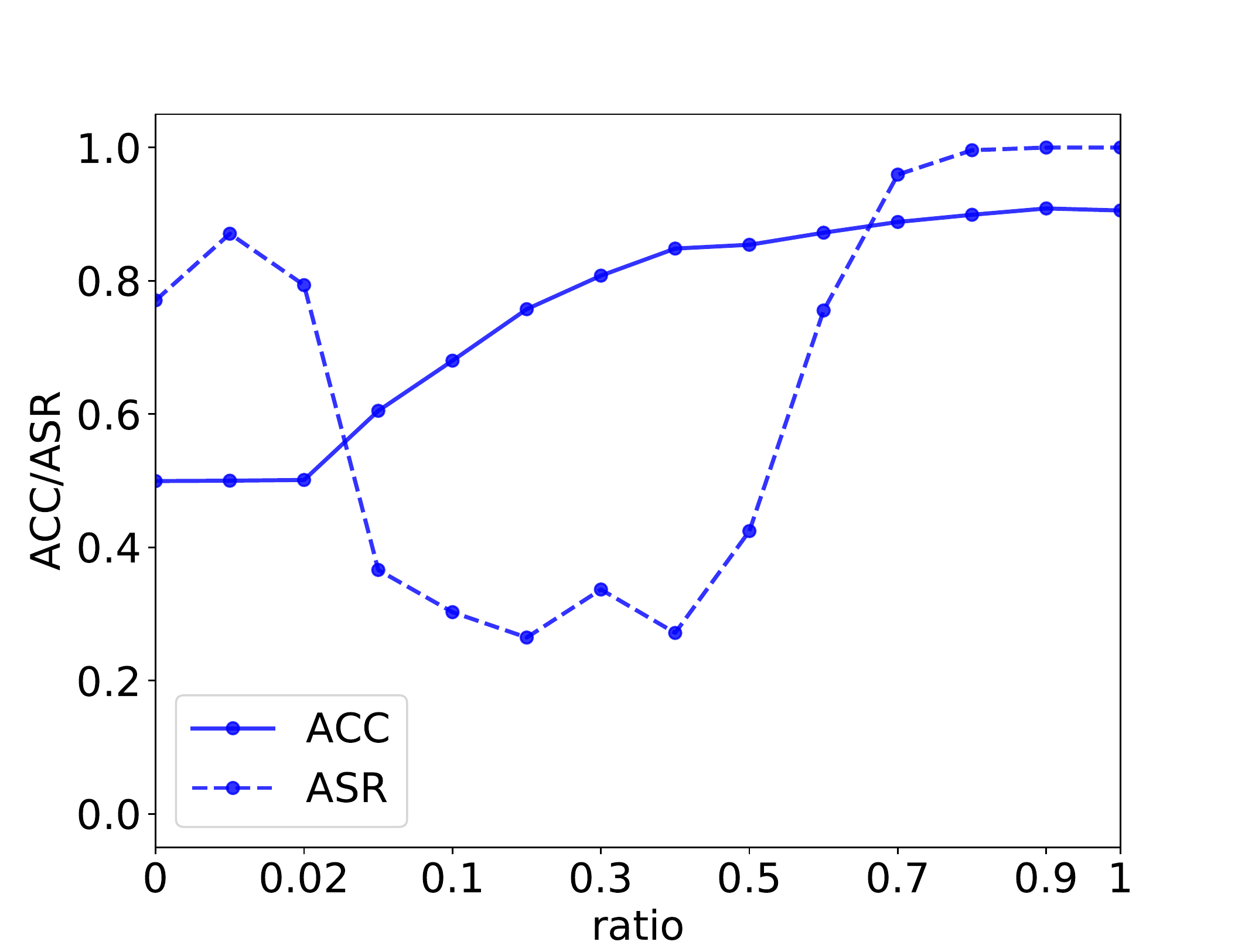}}
\hfil
\subcaptionbox{ACC/ASR (w/ E-PUR), Trigger Sentence (QNLI).}{\includegraphics[height=1.3 in,width=0.28\linewidth]{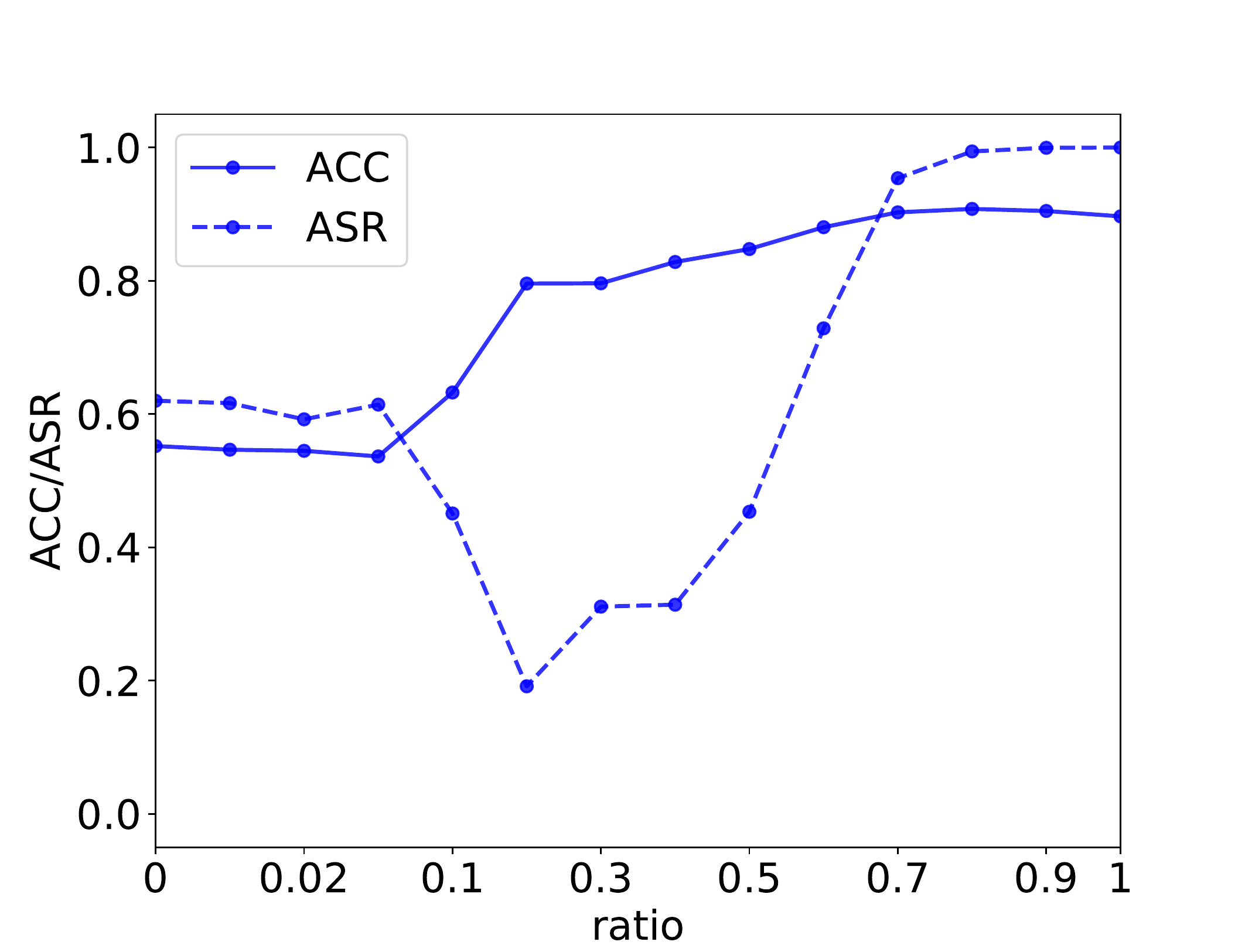}}
\hfil
\subcaptionbox{Loss Visualization, Trigger Sentence (Scratch) (QNLI).}{\includegraphics[height=1.3 in,width=0.28\linewidth]{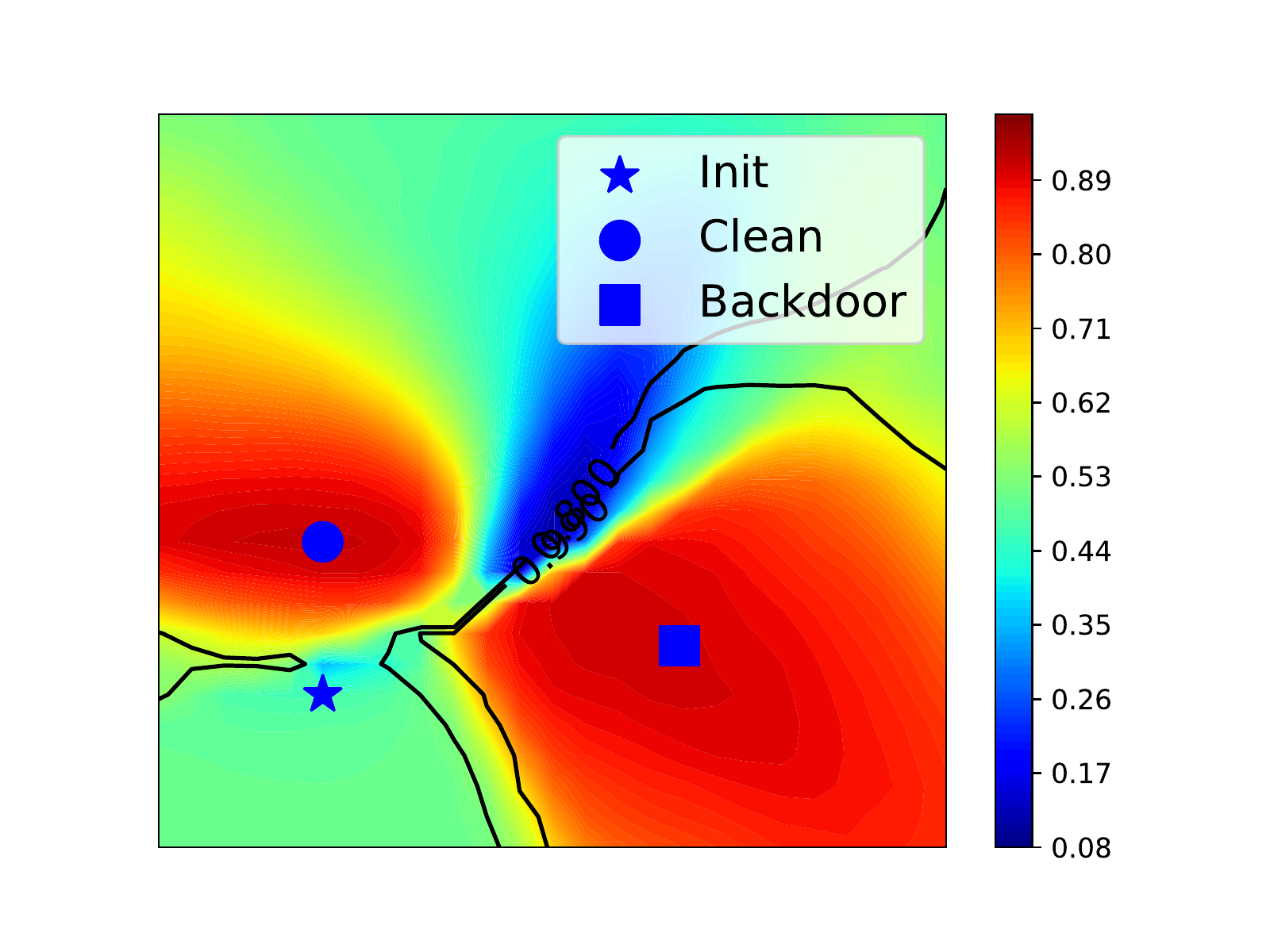}}
\hfil
\subcaptionbox{ACC/ASR (w/o E-PUR), Trigger Sentence (Scratch) (QNLI).}{\includegraphics[height=1.3 in,width=0.28\linewidth]{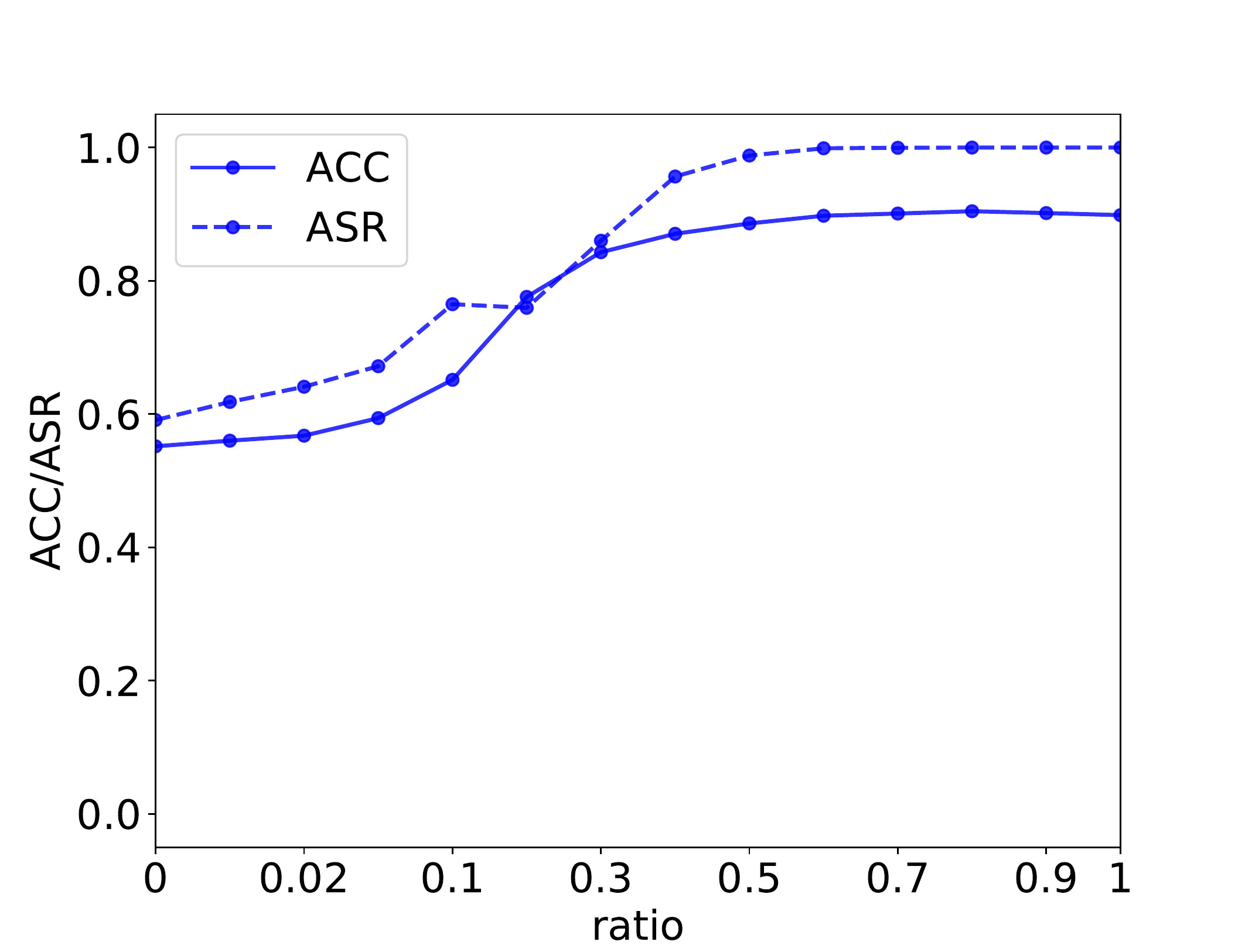}}
\hfil
\subcaptionbox{ACC/ASR (w/ E-PUR), Trigger Sentence (Scratch) (QNLI).}{\includegraphics[height=1.3 in,width=0.28\linewidth]{fig/appendix/QNLI_sentScr-finetuneGoodE-plot.pdf}}
\vskip -0.1 in
\caption{Visualization of the clean ACC and the backdoor ASR in the parameter spaces, and ACC/ASR with different reserve ratios under multiple trigger sentence based backdoor attacks on the QNLI sentence-pair classification.}
\vskip -0.15 in
\label{fig:4}
\end{figure*}

\section{Supplementary Experimental Results}

Also, due to space limitations, only part of the experimental results are included in the main paper. In this section, we list more supplementary experimental results. We visualize the clean ACC and the backdoor ASR in the parameter spaces, and ACC/ASR with different reserve ratios under multiple backdoor attacks on the SST-2 sentiment classification dataset and the QNLI sentence-pair classification dataset. Results on sentence based attacks on SST-2 are reported in Fig.~\ref{fig:2}; results on sentence based attacks on QNLI are reported in Fig.~\ref{fig:4}; results on word based attacks on SST-2 are reported in Fig.~\ref{fig:1}; and results on word based attacks on QNLI are reported in Fig.~\ref{fig:3}.

In most cases, there exists an area with a high clean ACC and a low backdoor ASR between the pre-trained BERT parameter and the backdoored parameter in the parameter space, which is a good area for mitigating backdoors. Under these cases, the backdoor ASR will drop when $\rho$ is small, and backdoors can be mitigated. Only a few cases are medium or difficult, where the backdoor ASR is always high, and backdoors are hard to mitigate.

\clearpage
\begin{figure*}[!h]
\centering
\subcaptionbox{Loss Visualization, Trigger Word (SST-2).}{\includegraphics[height=1.3 in,width=0.28\linewidth]{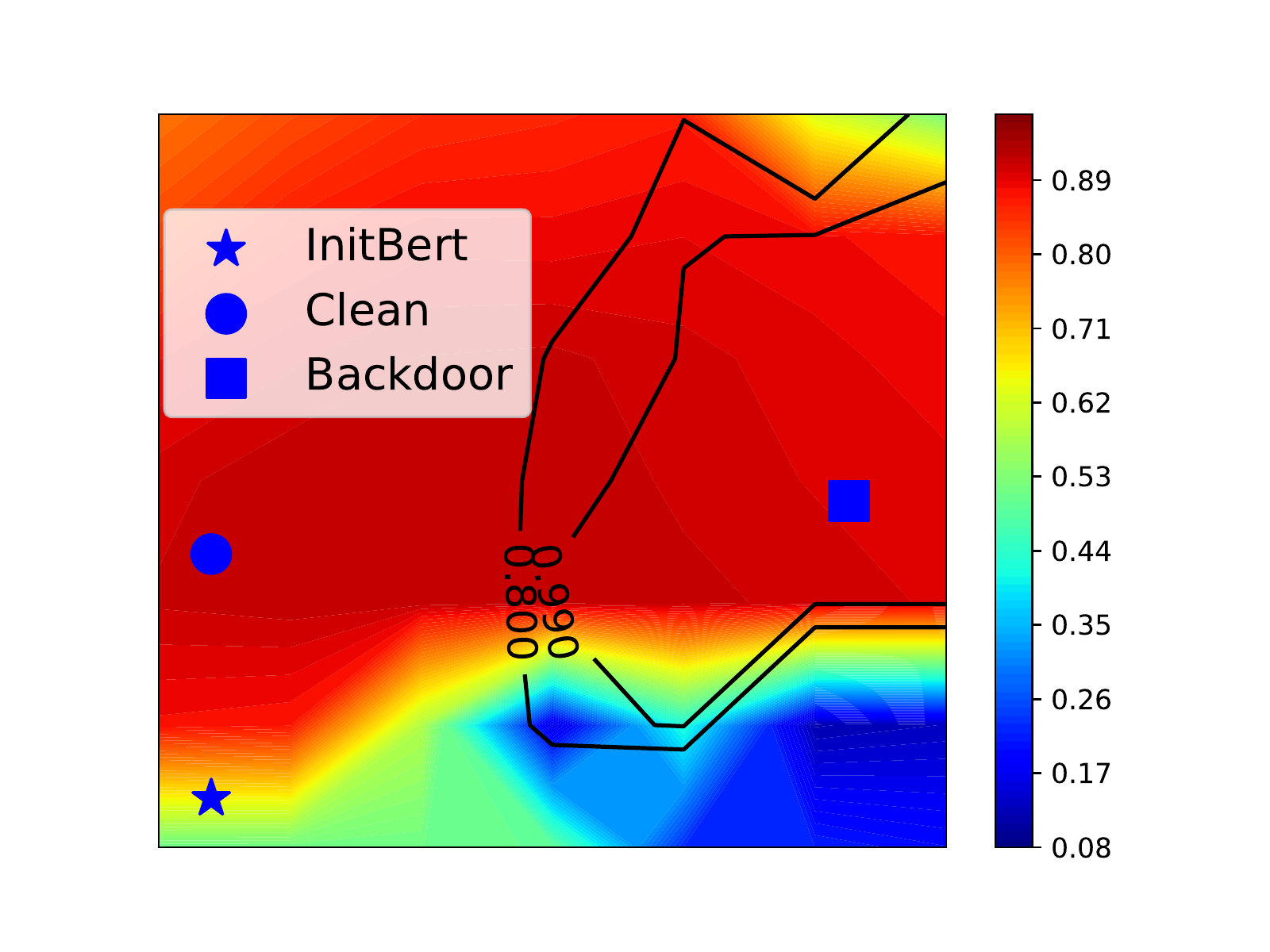}}
\hfil
\subcaptionbox{ACC/ASR (w/o E-PUR), Trigger Word (SST-2).}{\includegraphics[height=1.3 in,width=0.28\linewidth]{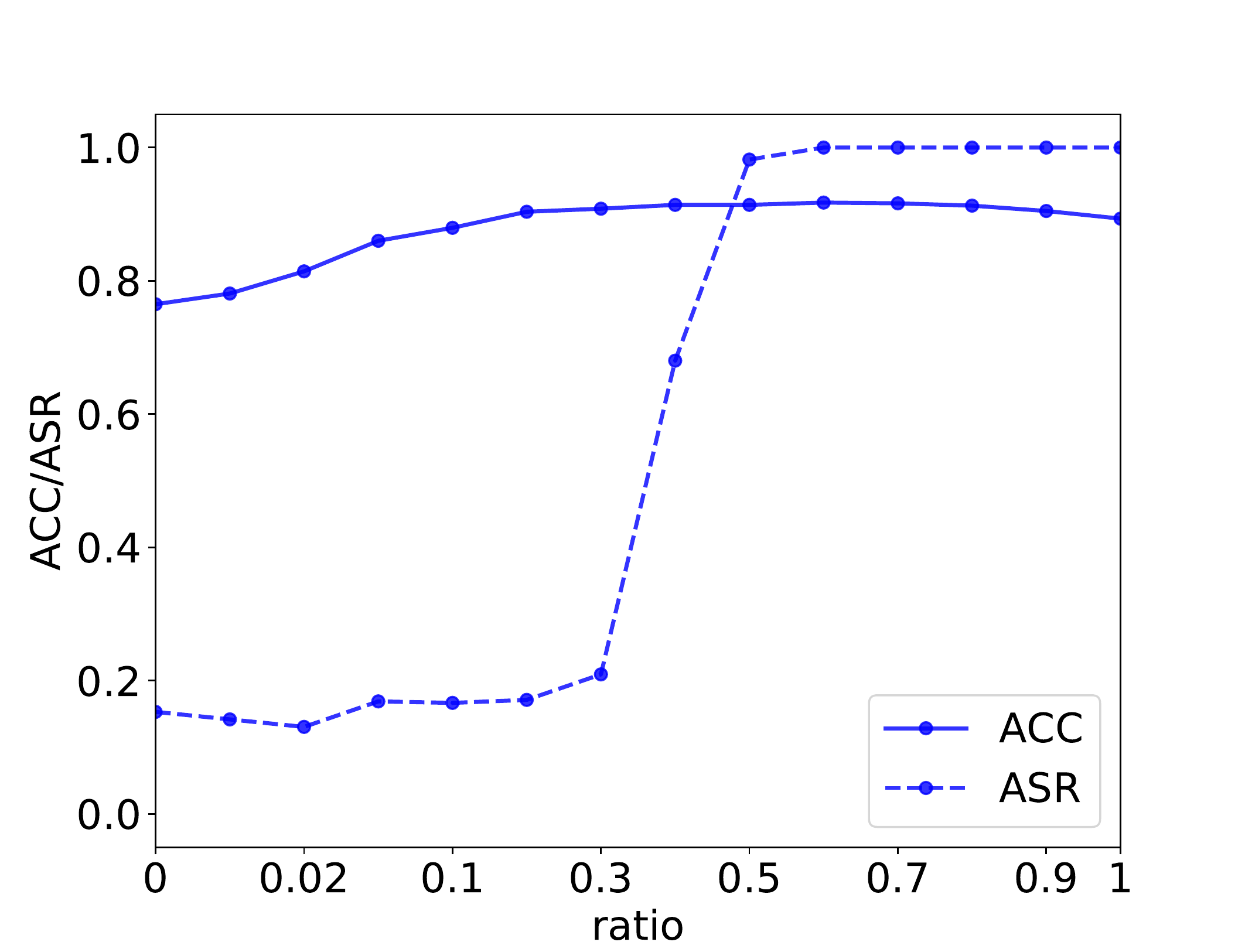}}
\hfil
\subcaptionbox{ACC/ASR (w/ E-PUR), Trigger Word (SST-2).}{\includegraphics[height=1.3 in,width=0.28\linewidth]{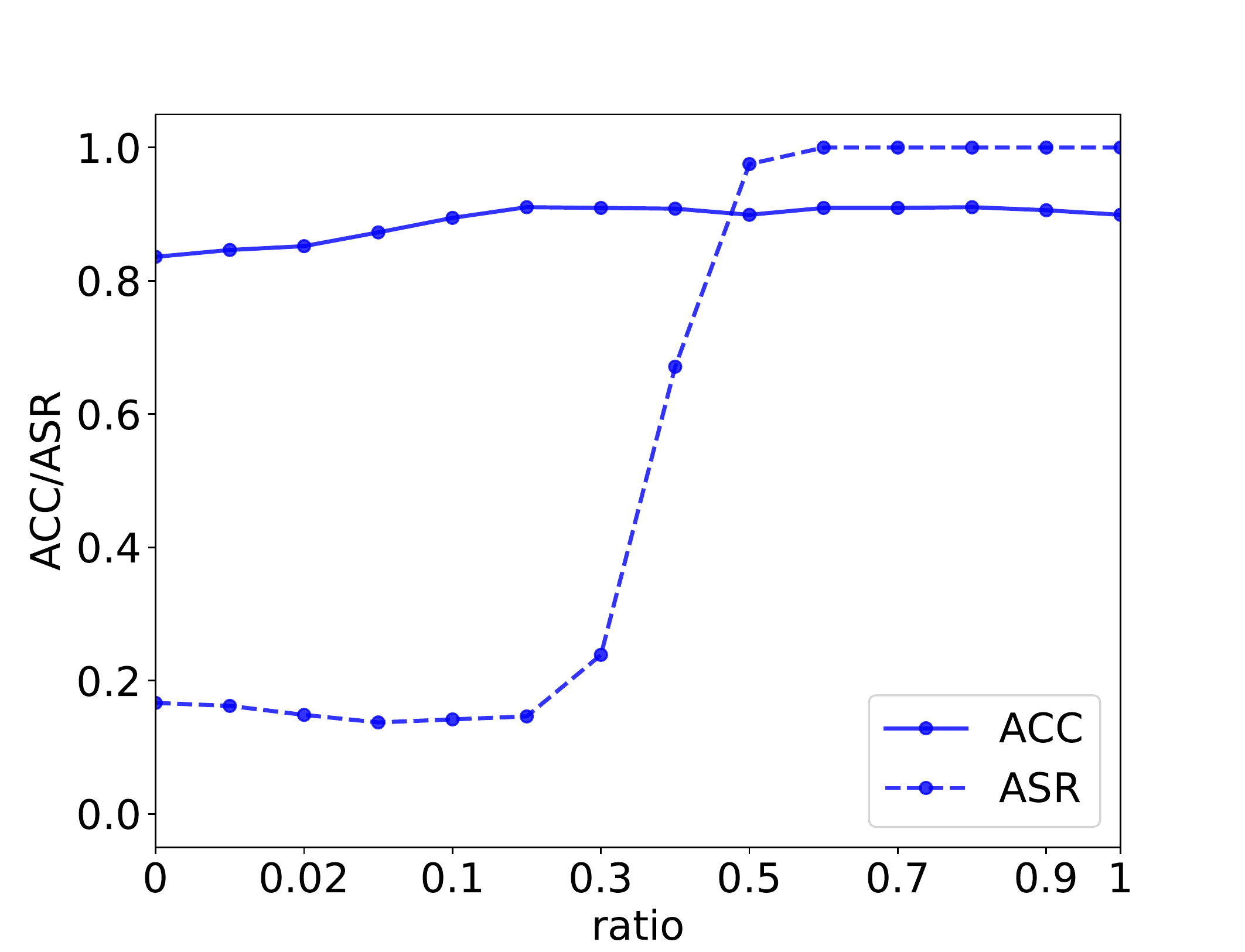}}
\hfil
\subcaptionbox{Loss Visualization, Trigger Word (Scratch) (SST-2).}{\includegraphics[height=1.3 in,width=0.28\linewidth]{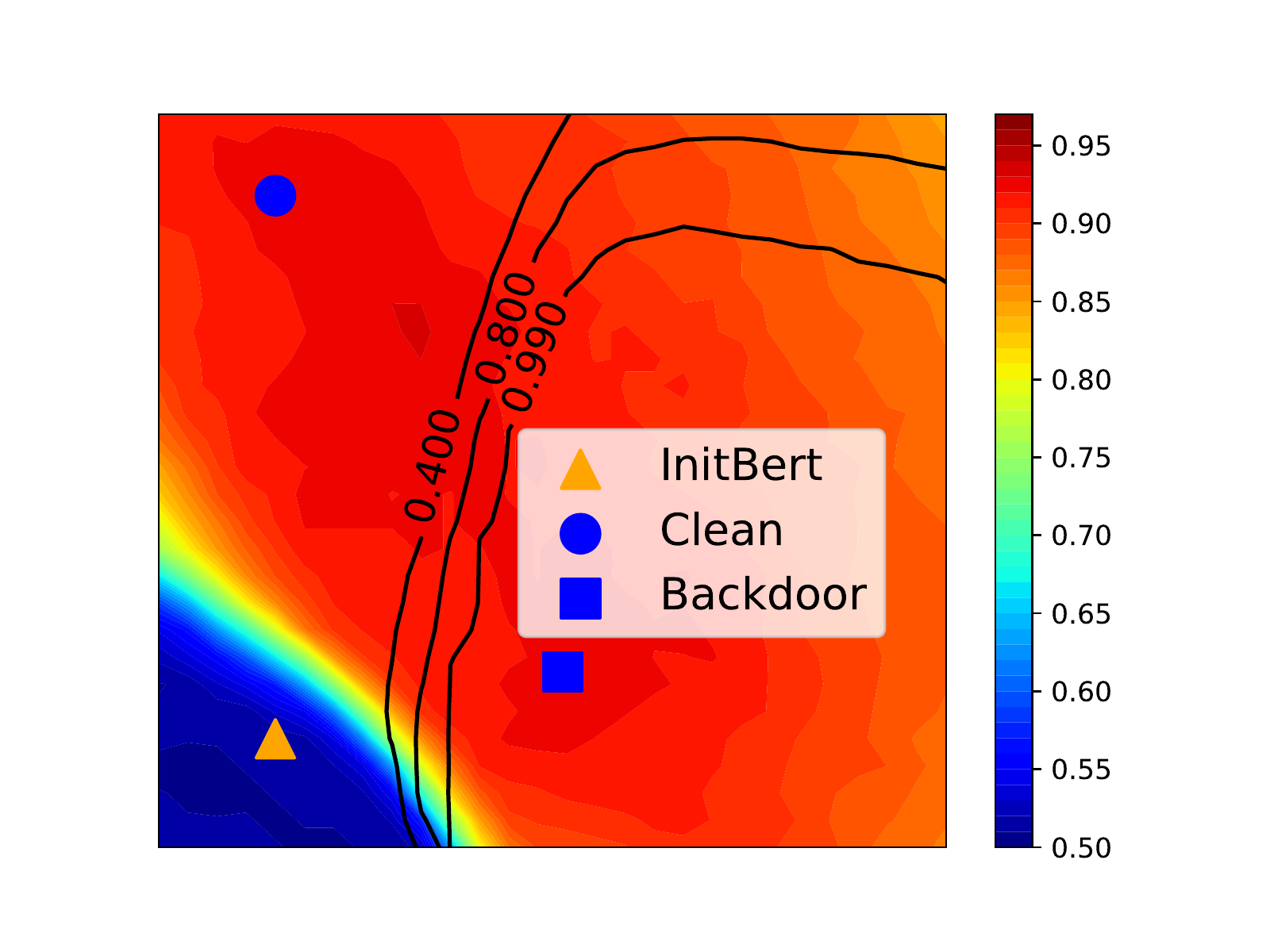}}
\hfil
\subcaptionbox{ACC/ASR (w/o E-PUR), Trigger Word (Scratch) (SST-2).}{\includegraphics[height=1.3 in,width=0.28\linewidth]{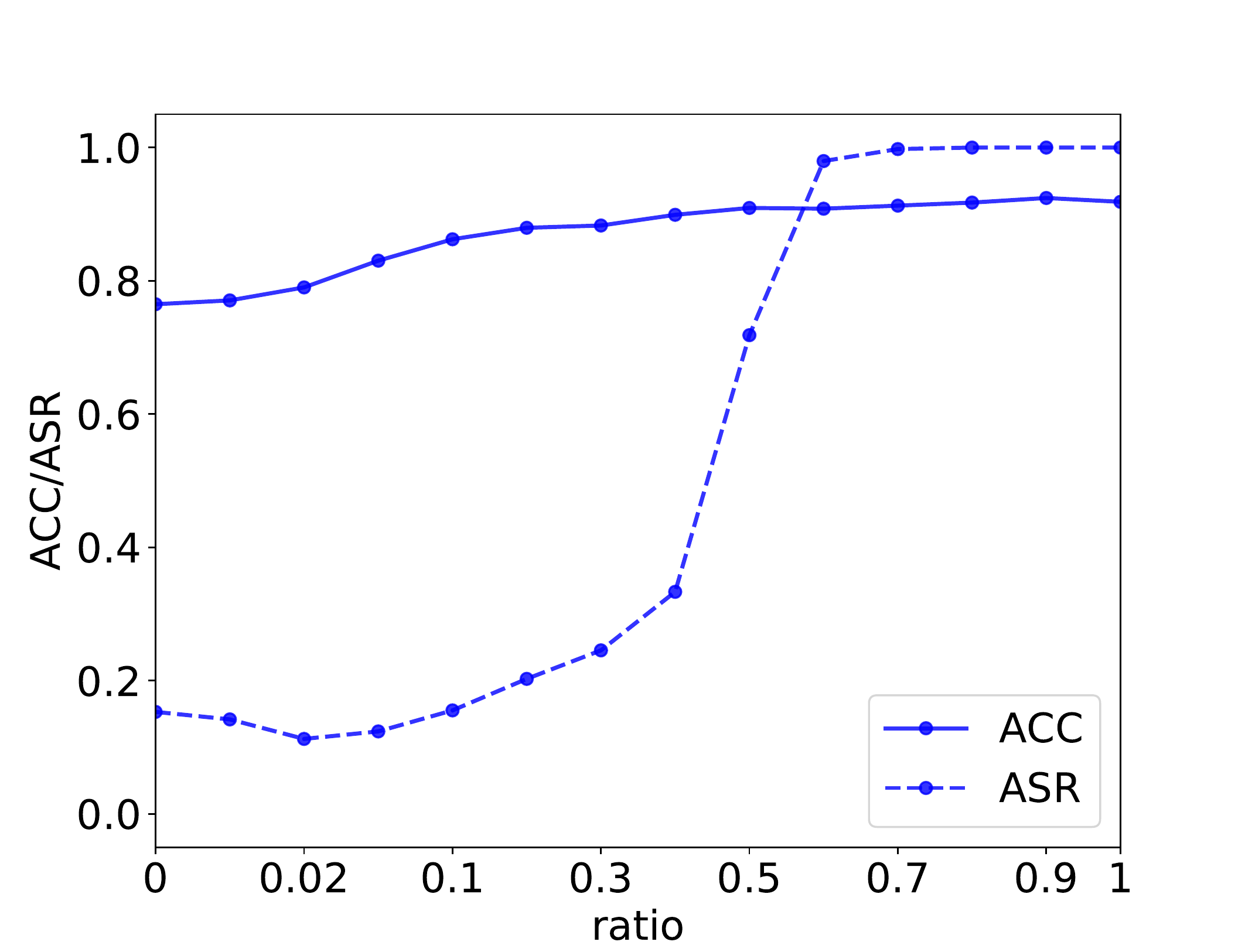}}
\hfil
\subcaptionbox{ACC/ASR (w/ E-PUR), Trigger Word (Scratch) (SST-2).}{\includegraphics[height=1.3 in,width=0.28\linewidth]{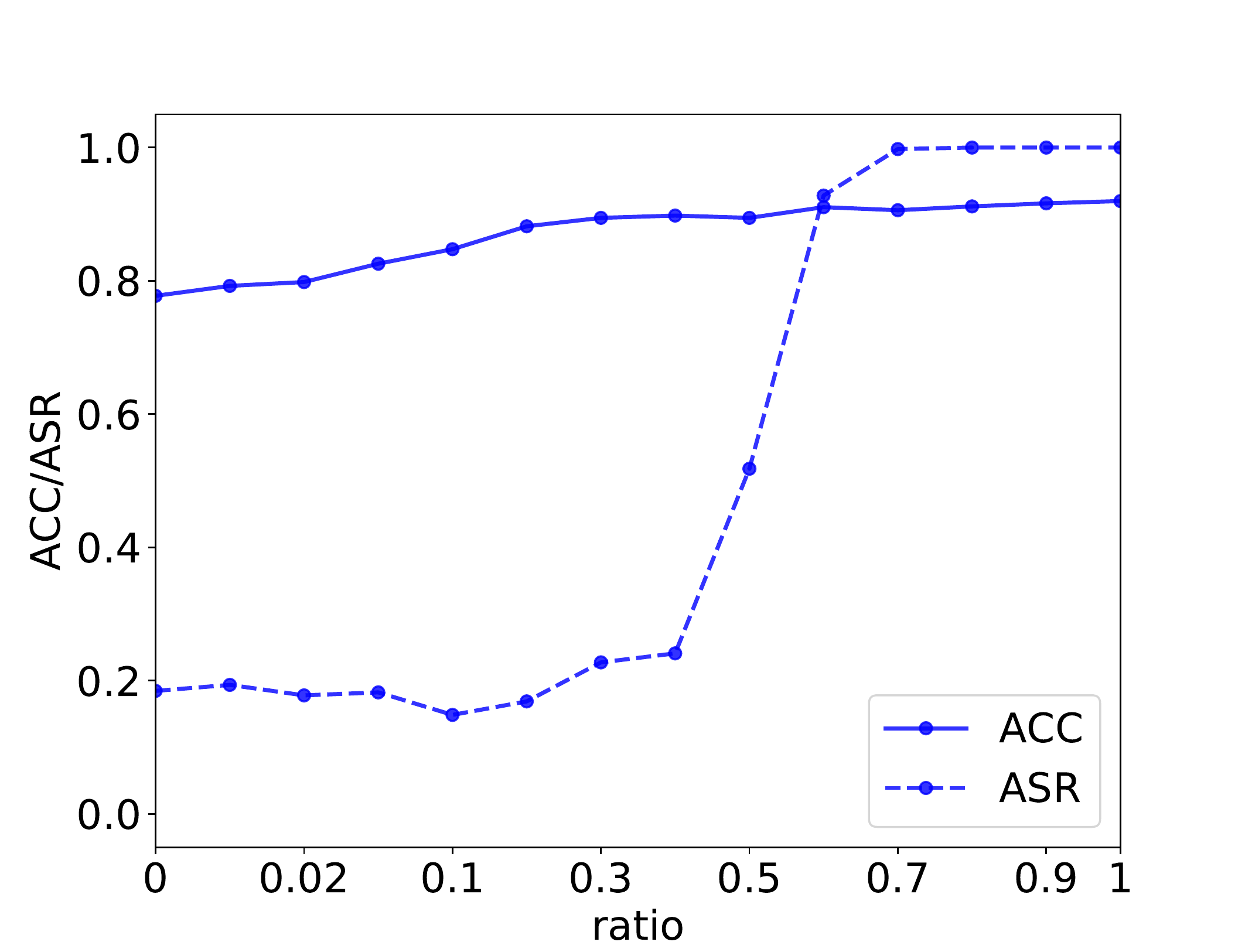}}
\hfil
\subcaptionbox{Loss Visualization, Trigger Word+EP (SST-2).}{\includegraphics[height=1.3 in,width=0.28\linewidth]{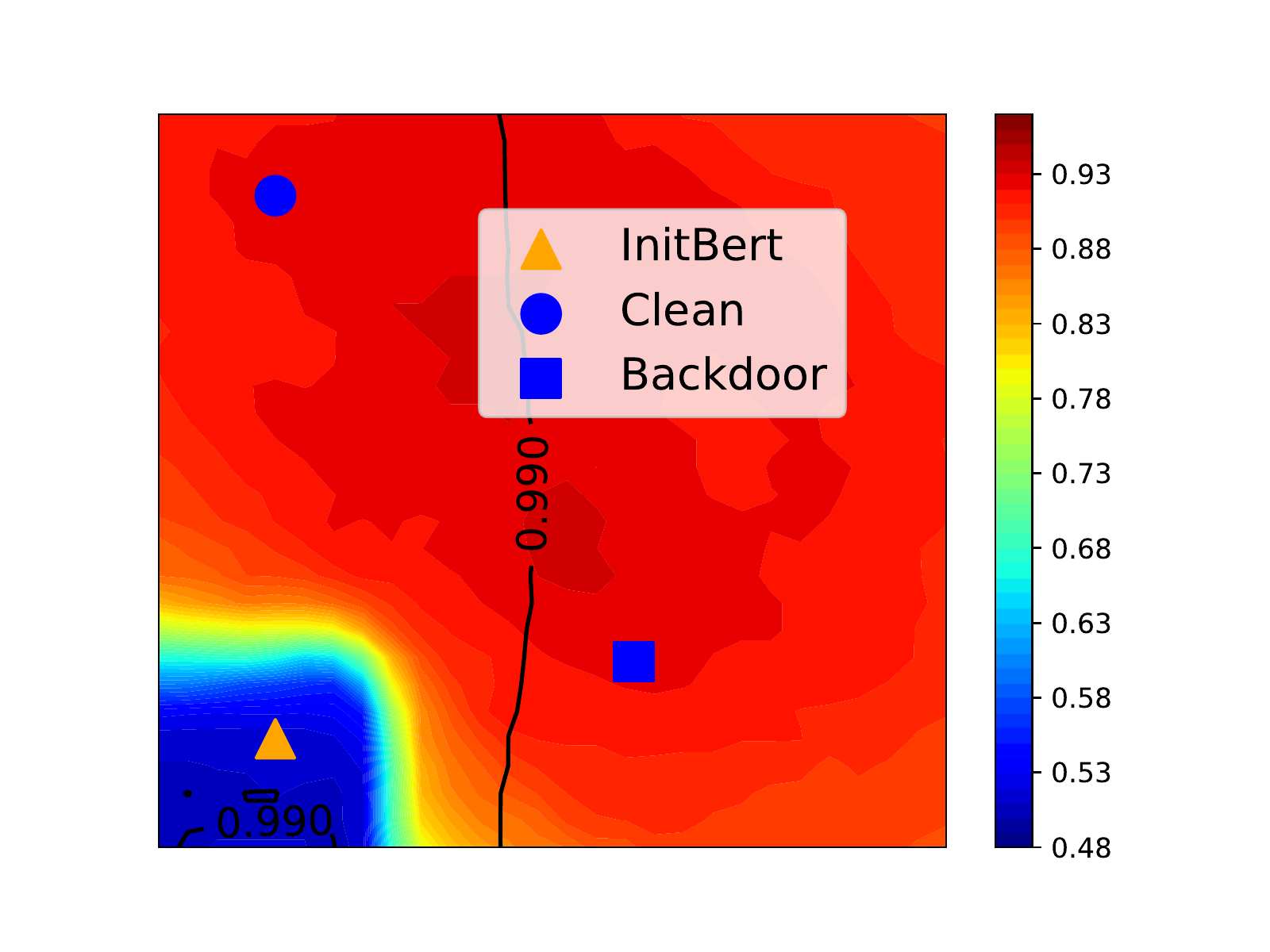}}
\hfil
\subcaptionbox{ACC/ASR (w/o E-PUR), Trigger Word+EP (SST-2).}{\includegraphics[height=1.3 in,width=0.28\linewidth]{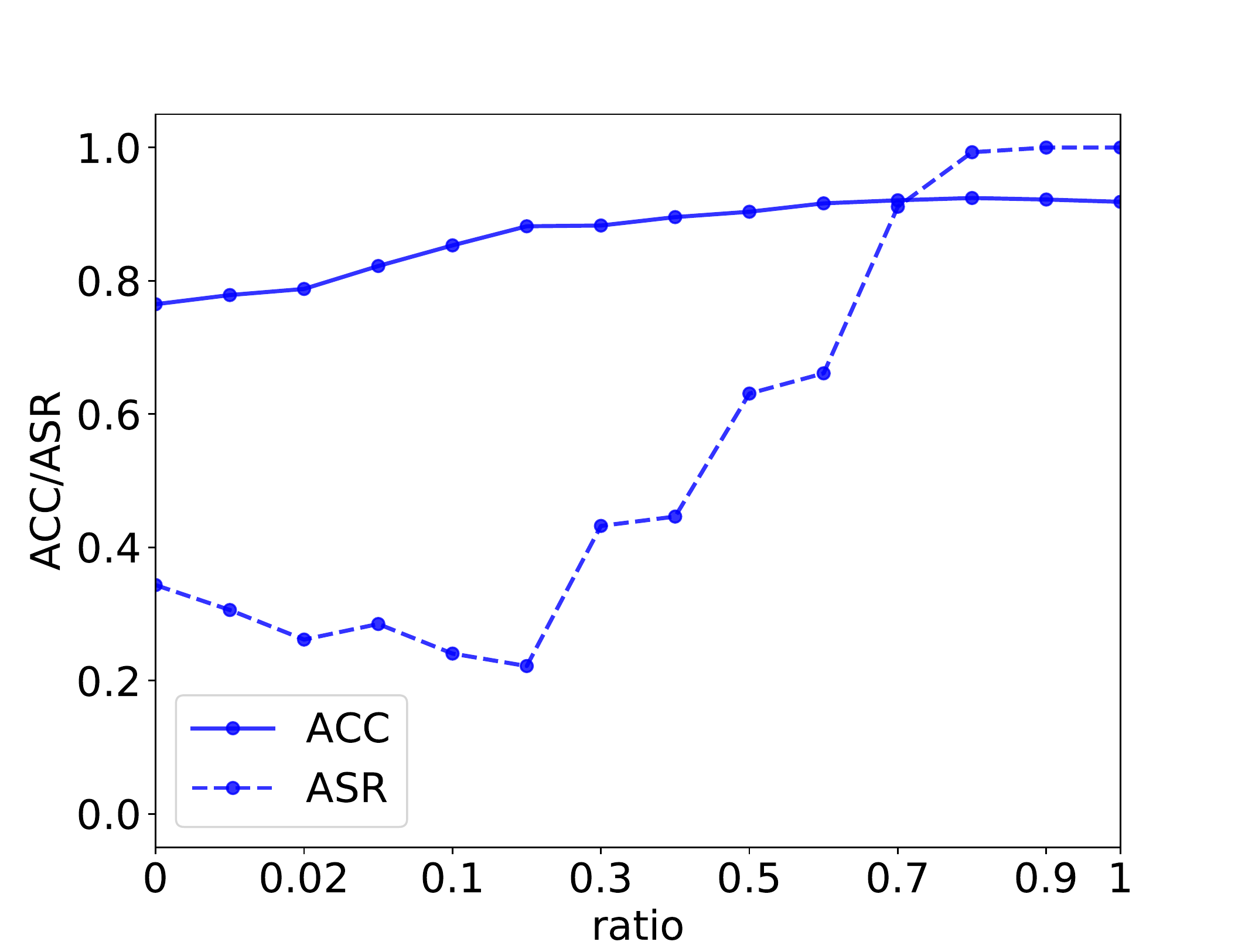}}
\hfil
\subcaptionbox{ACC/ASR (w/ E-PUR), Trigger Word+EP (SST-2).}{\includegraphics[height=1.3 in,width=0.28\linewidth]{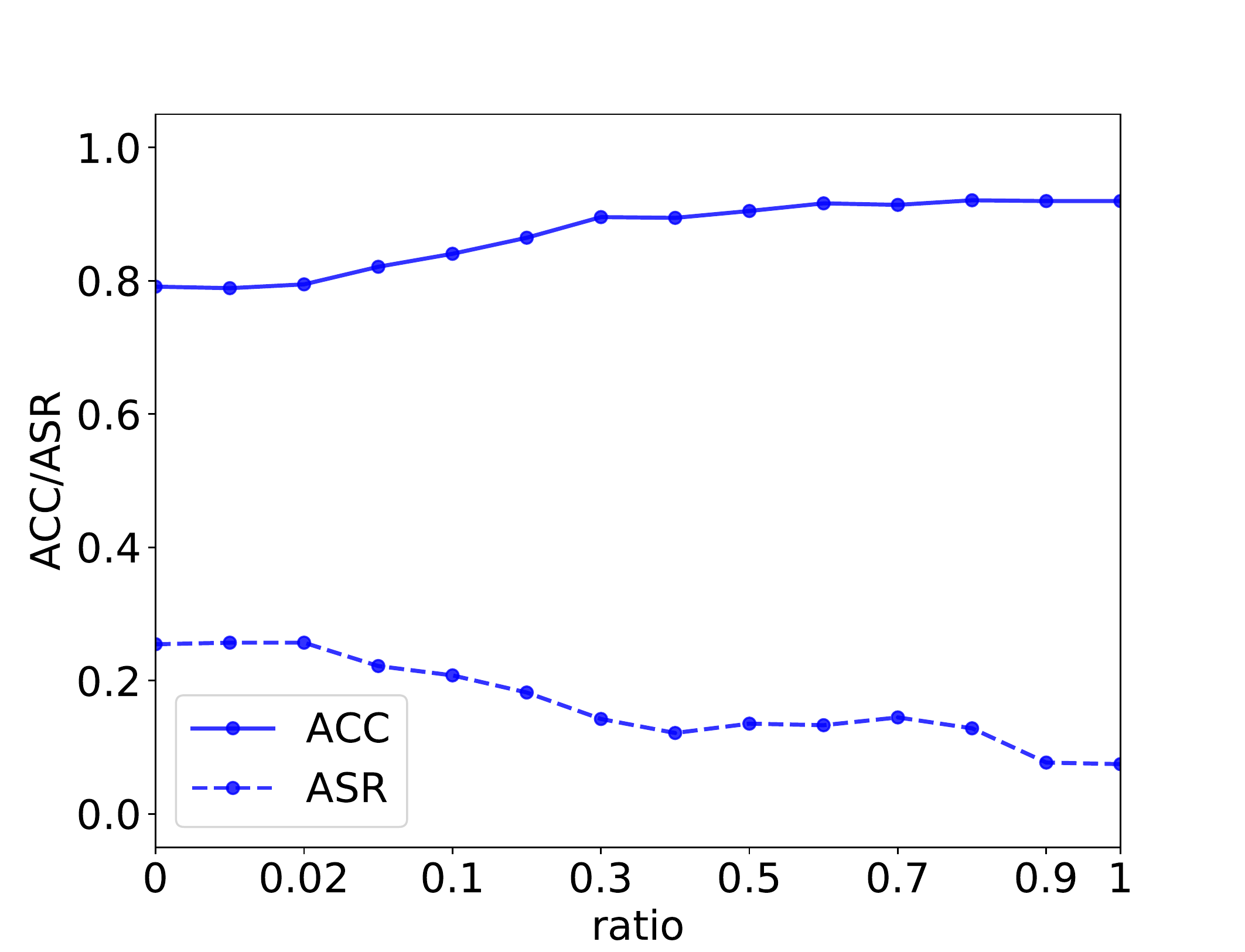}}
\hfil
\subcaptionbox{Loss Visualization, Trigger Word+ES (SST-2).}{\includegraphics[height=1.3 in,width=0.28\linewidth]{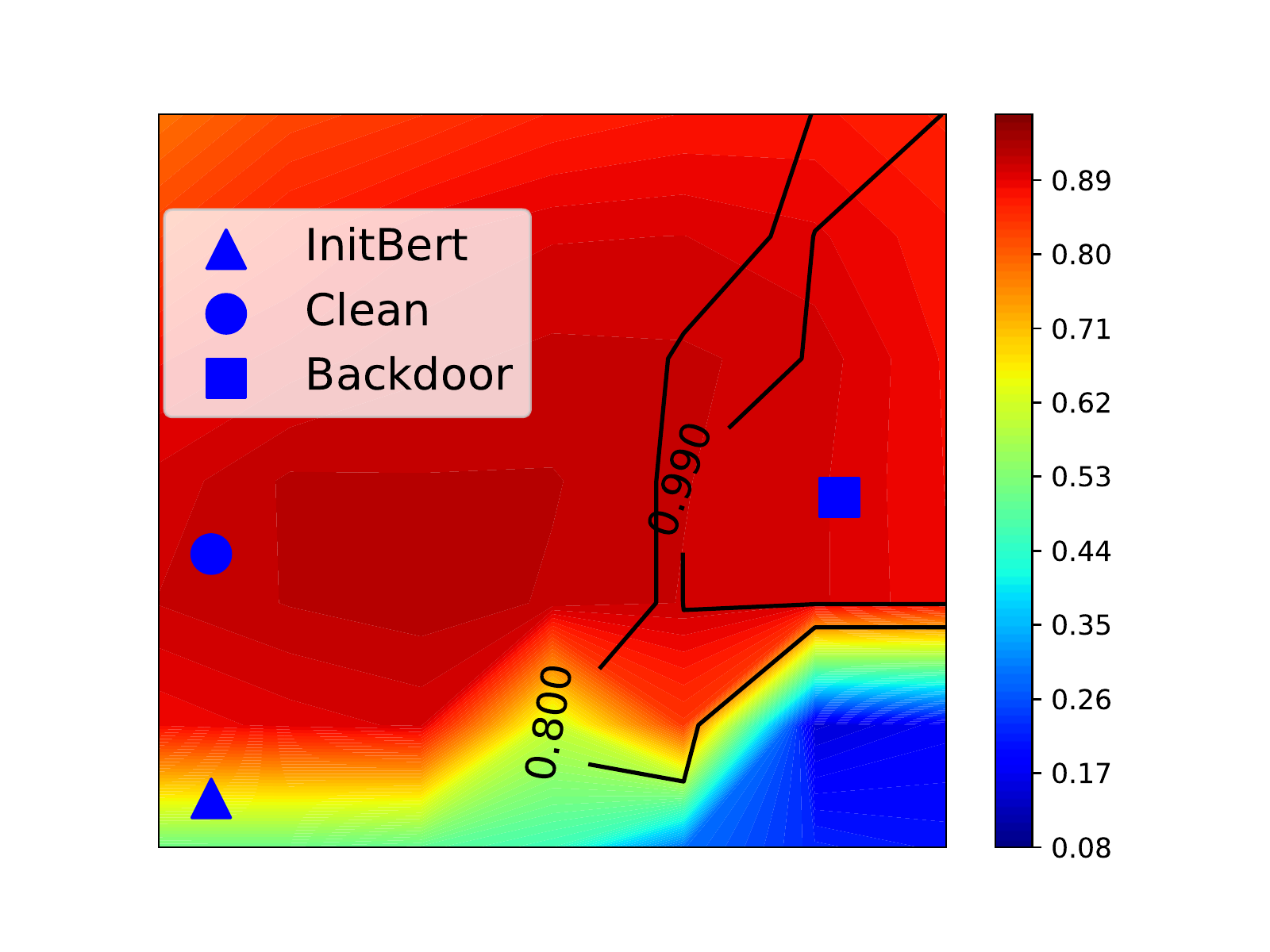}}
\hfil
\subcaptionbox{ACC/ASR (w/o E-PUR), Trigger Word+ES (SST-2).}{\includegraphics[height=1.3 in,width=0.28\linewidth]{fig/appendix/SST-2_NoES-finetune-plot.pdf}}
\hfil
\subcaptionbox{ACC/ASR (w/ E-PUR), Trigger Word (SST-2).}{\includegraphics[height=1.3 in,width=0.28\linewidth]{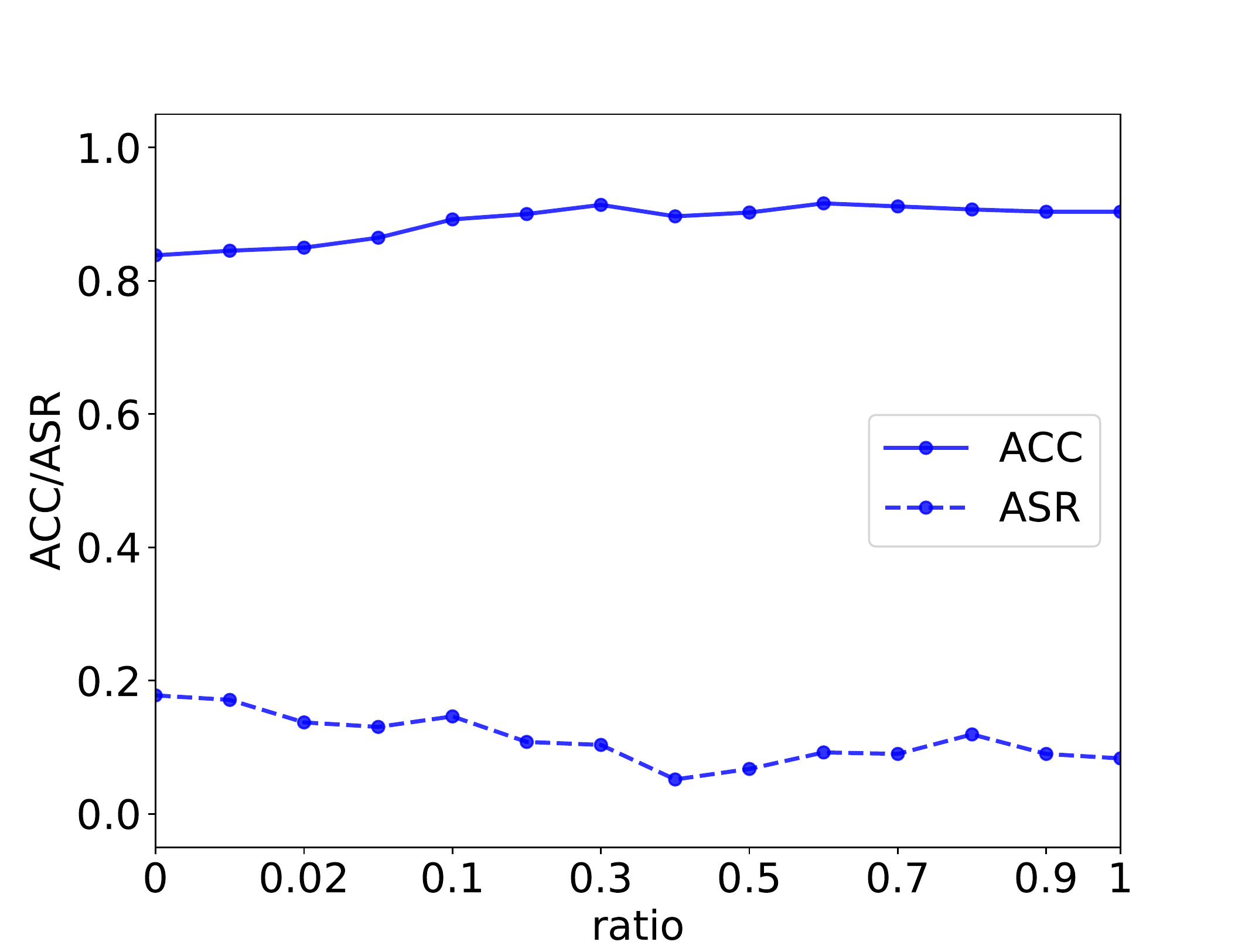}}
\hfil
\subcaptionbox{Loss Visualization, Trigger Word+ES (Scratch) (SST-2).}{\includegraphics[height=1.3 in,width=0.28\linewidth]{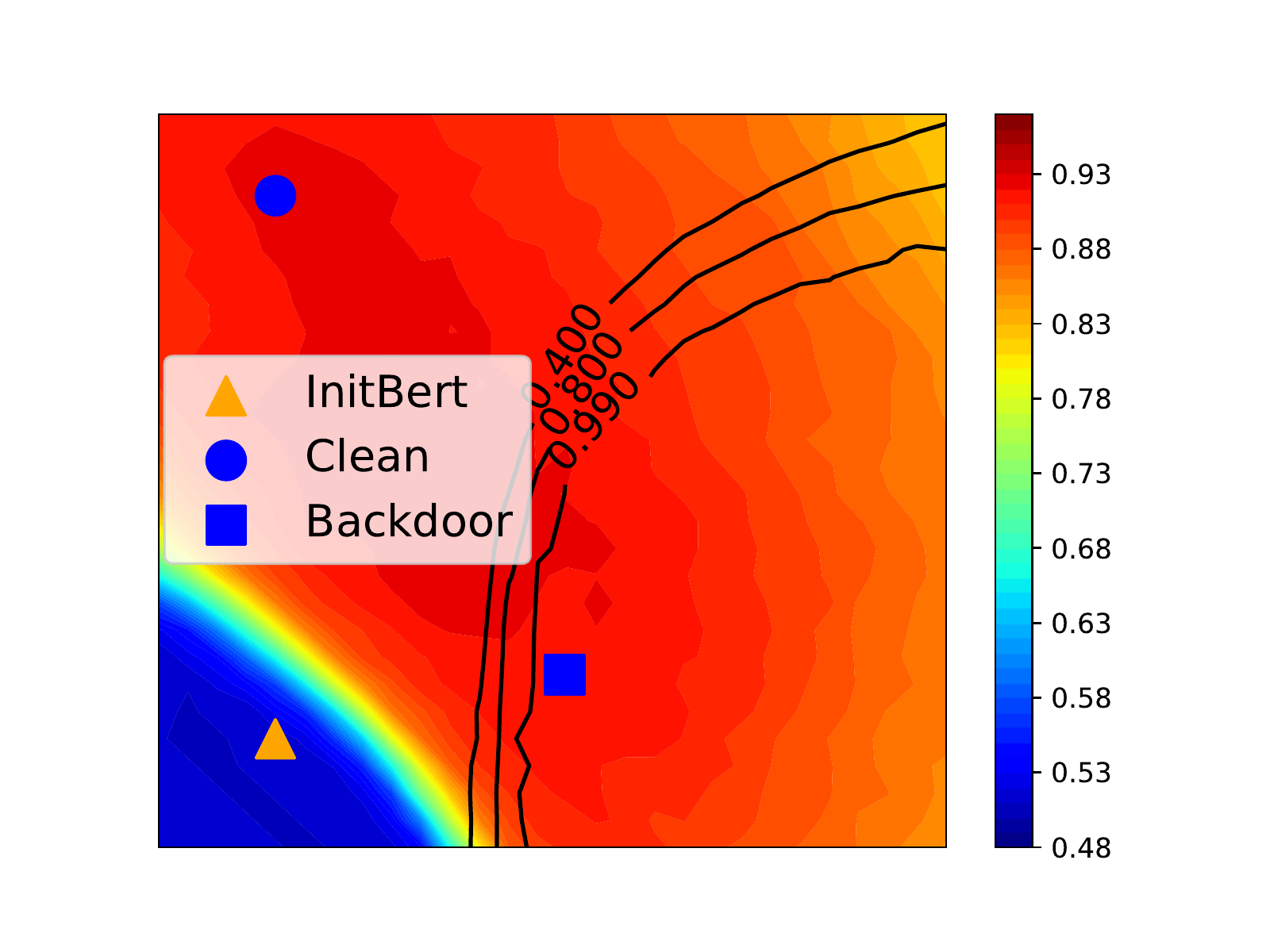}}
\hfil
\subcaptionbox{ACC/ASR (w/o E-PUR), Trigger Word+ES (Scratch) (SST-2).}{\includegraphics[height=1.3 in,width=0.28\linewidth]{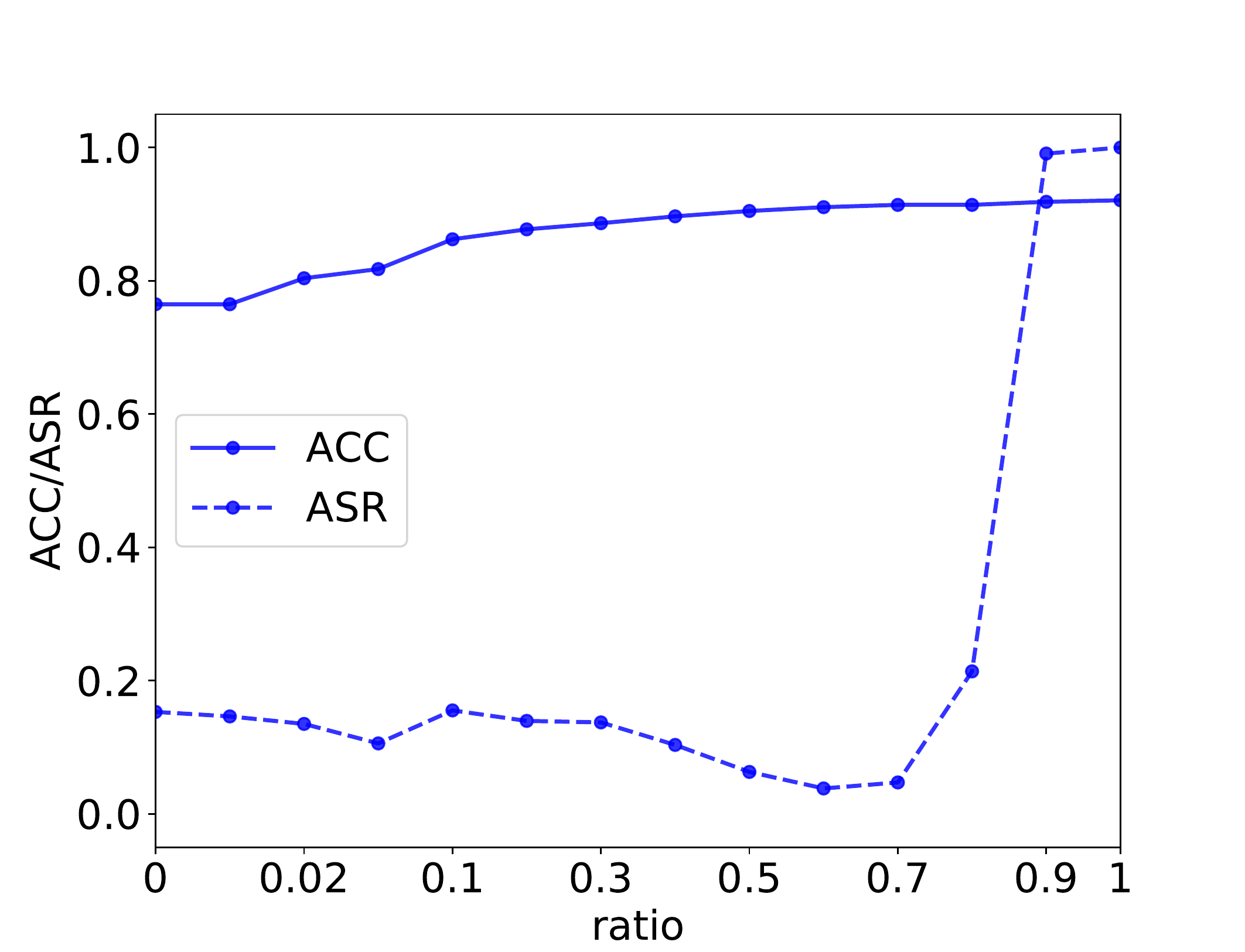}}
\hfil
\subcaptionbox{ACC/ASR (w/ E-PUR), Trigger Word+ES (Scratch) (SST-2).}{\includegraphics[height=1.3 in,width=0.28\linewidth]{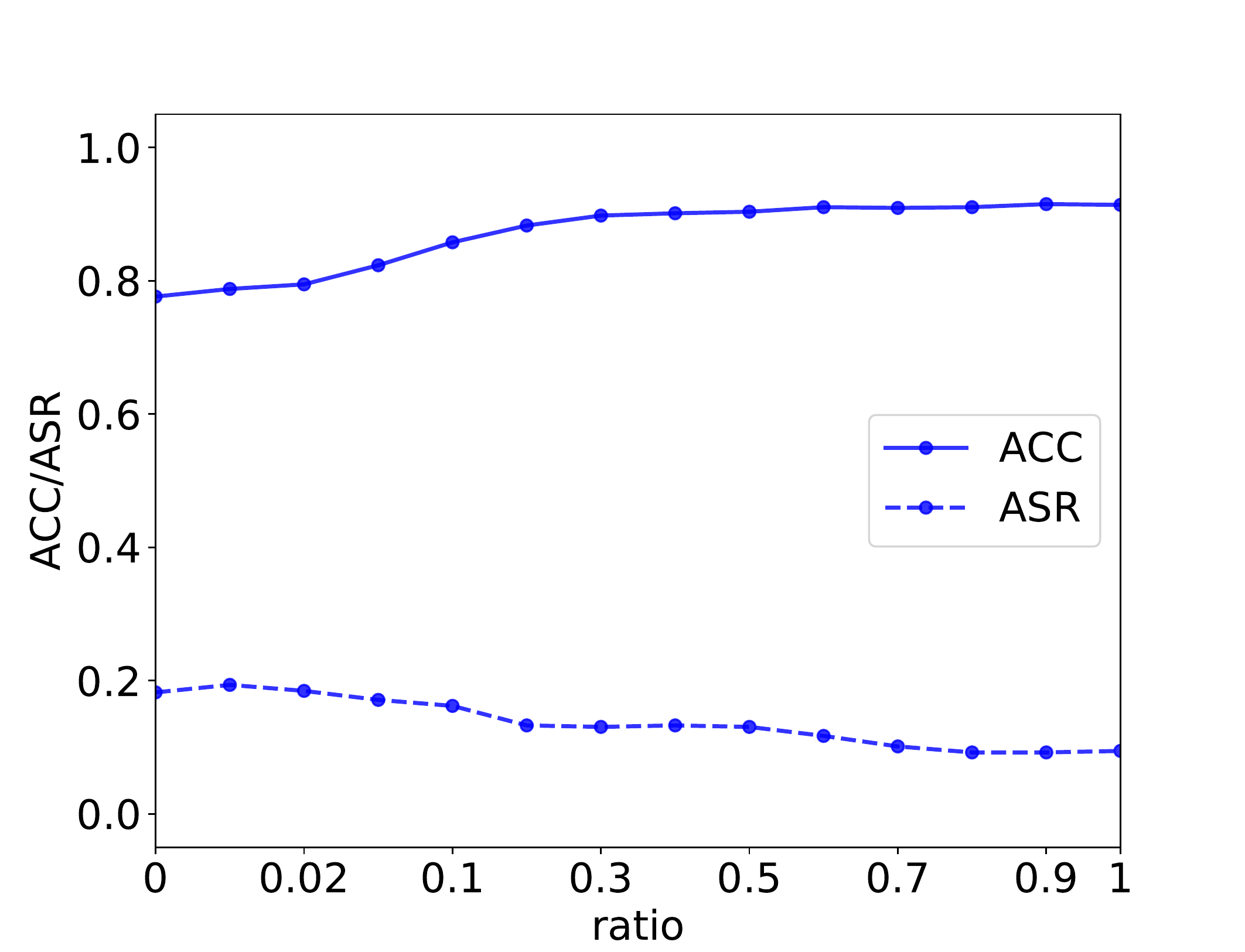}}
\vskip -0.1 in
\caption{Visualization of the clean ACC and the backdoor ASR in the parameter spaces, and ACC/ASR with different reserve ratios under multiple trigger word based backdoor attacks on the SST-2 sentiment classification.}
\vskip -0.15 in
\label{fig:1}
\end{figure*}

\begin{figure*}[!h]
\centering
\vskip 0.3 in
\subcaptionbox{Loss Visualization, Trigger Word (QNLI).}{\includegraphics[height=1.6 in,width=0.32\linewidth]{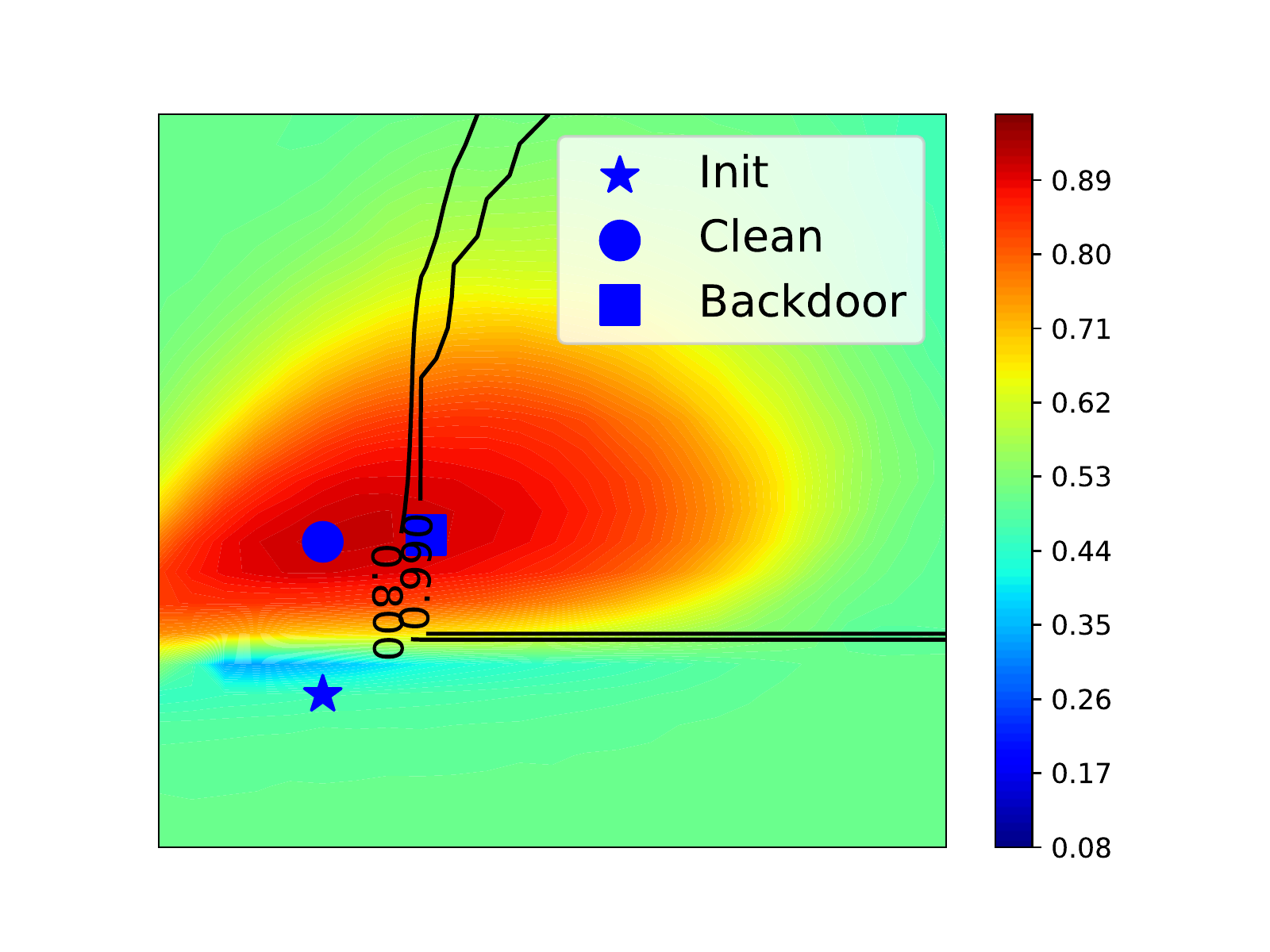}}
\hfil
\subcaptionbox{ACC/ASR (w/o E-PUR), Trigger Word (QNLI).}{\includegraphics[height=1.6 in,width=0.32\linewidth]{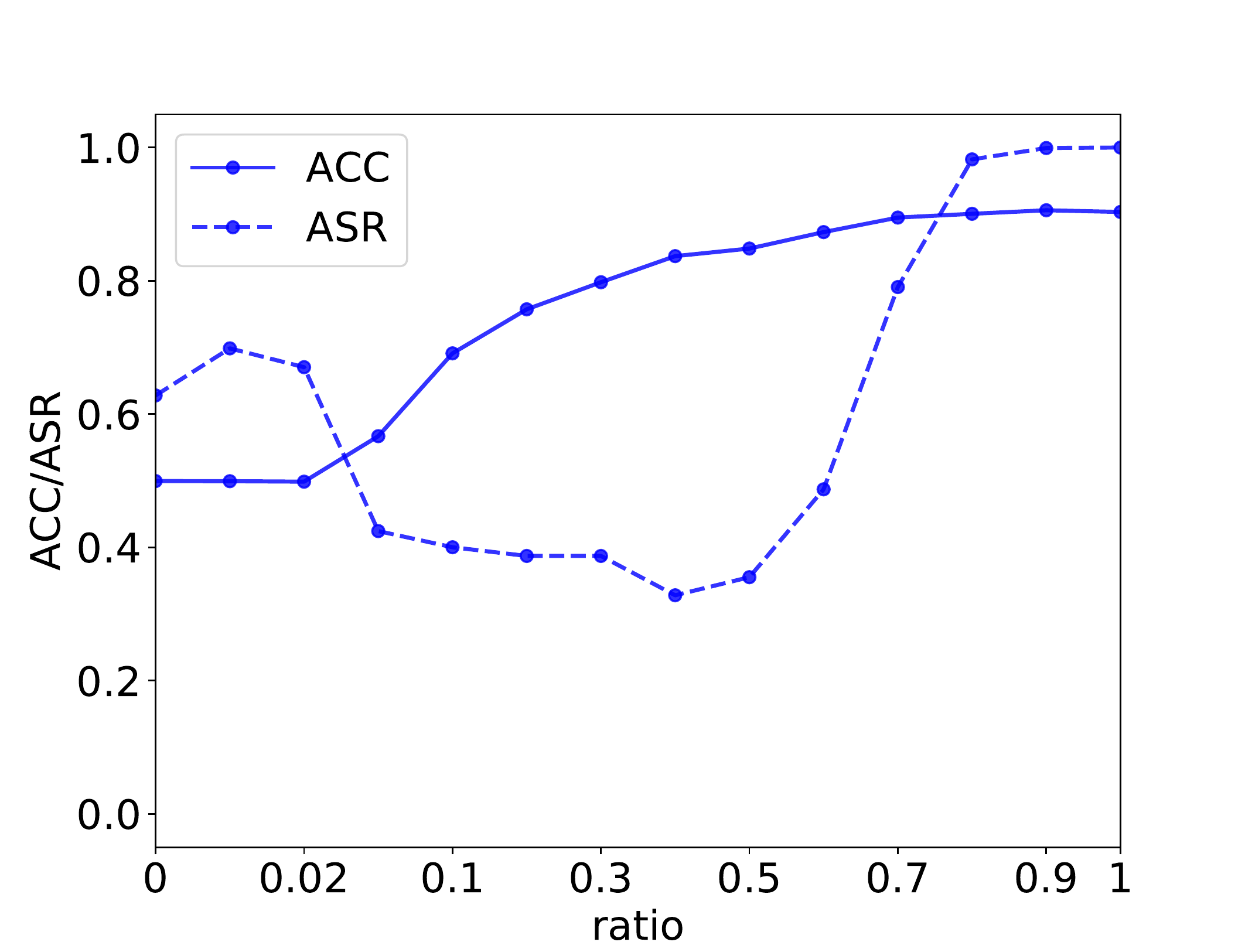}}
\hfil
\subcaptionbox{ACC/ASR (w/ E-PUR), Trigger Word (QNLI).}{\includegraphics[height=1.6 in,width=0.32\linewidth]{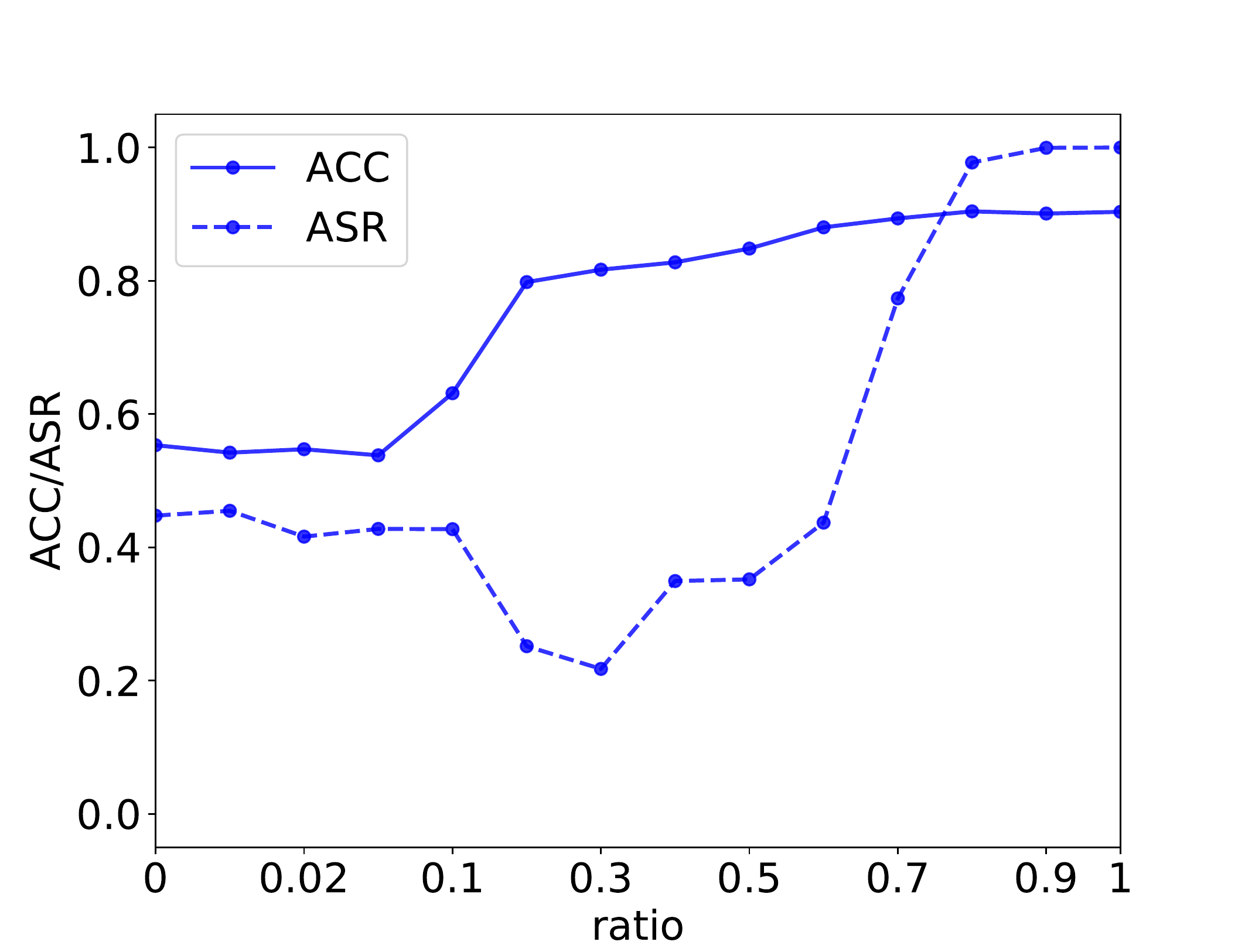}}
\hfil
\vskip 0.3 in
\subcaptionbox{Loss Visualization, Trigger Word (Scratch) (QNLI).}{\includegraphics[height=1.6 in,width=0.32\linewidth]{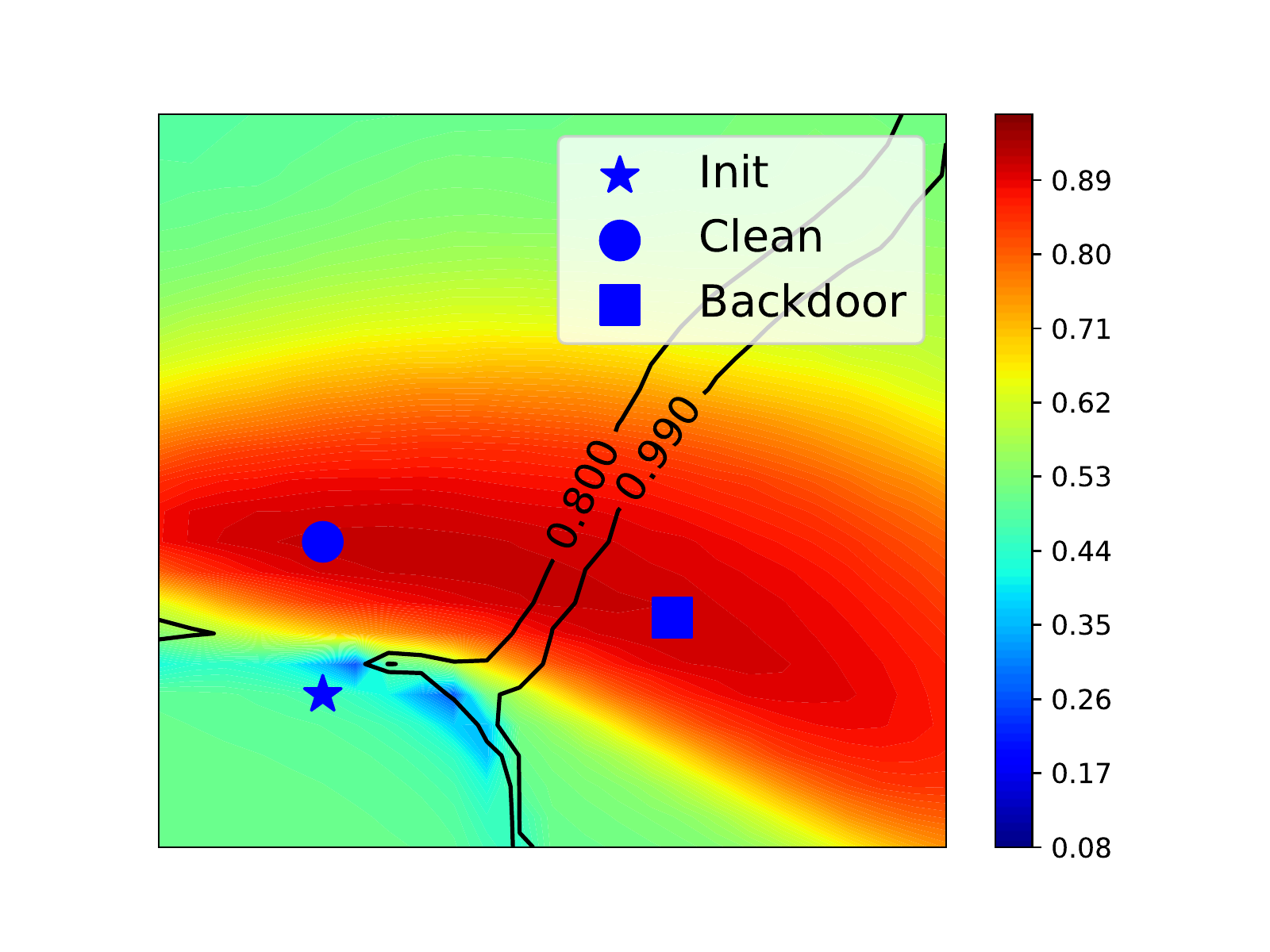}}
\hfil
\subcaptionbox{ACC/ASR (w/o E-PUR), Trigger Word (Scratch) (QNLI).}{\includegraphics[height=1.6 in,width=0.32\linewidth]{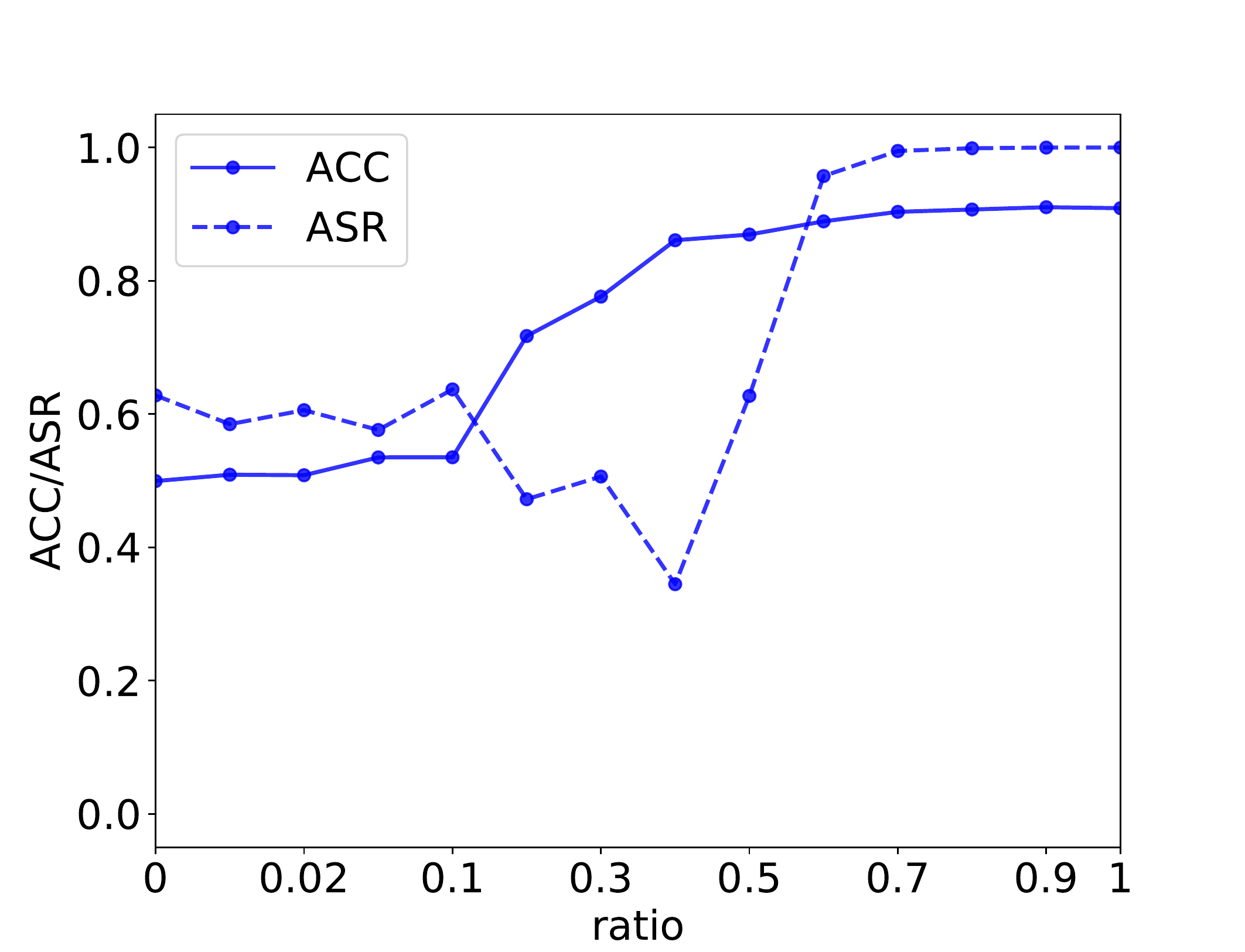}}
\hfil
\subcaptionbox{ACC/ASR (w/ E-PUR), Trigger Word (Scratch) (QNLI).}{\includegraphics[height=1.6 in,width=0.32\linewidth]{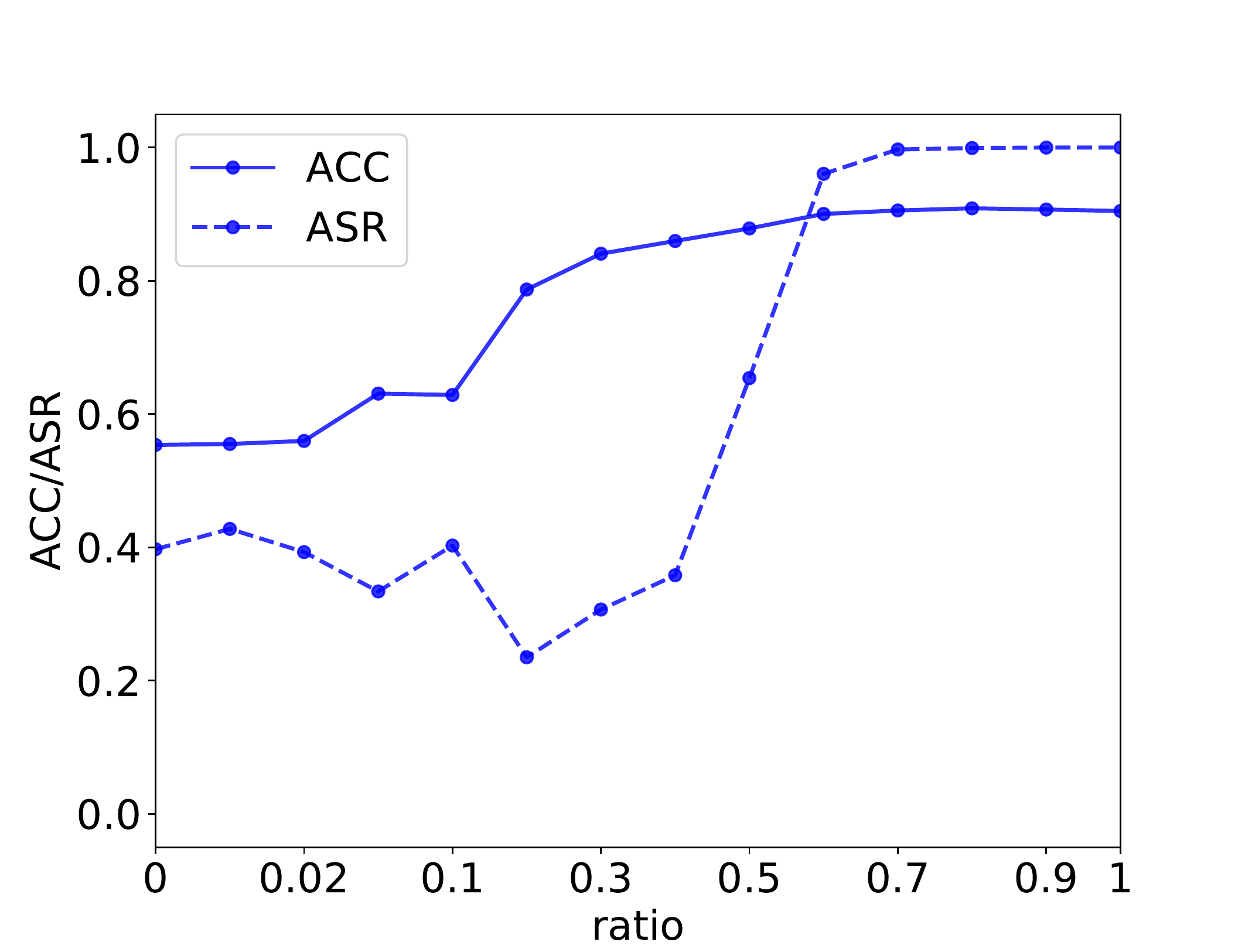}}
\hfil
\vskip 0.3 in
\subcaptionbox{Loss Visualization, Trigger Word+EP (QNLI).}{\includegraphics[height=1.6 in,width=0.32\linewidth]{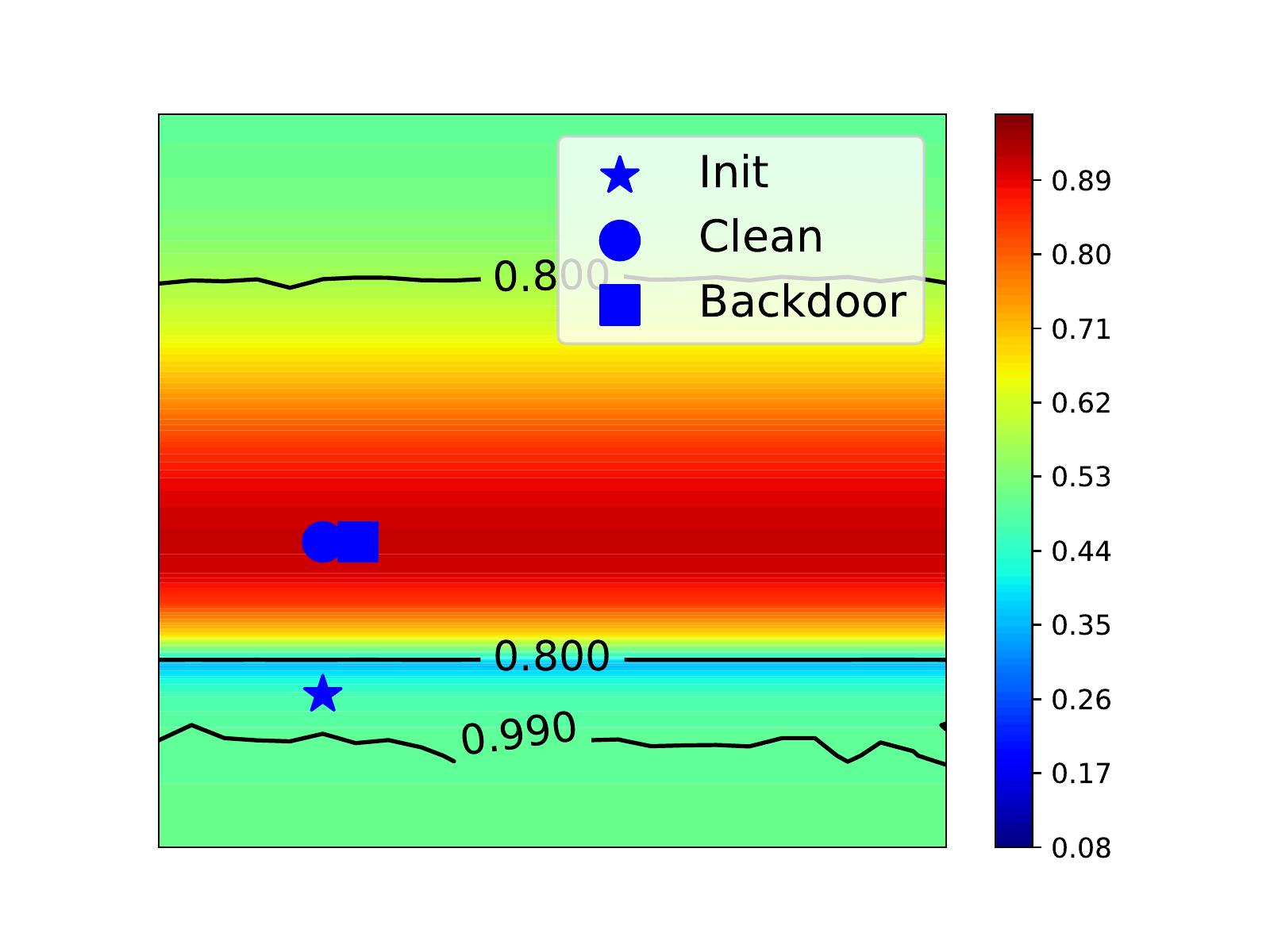}}
\hfil
\subcaptionbox{ACC/ASR (w/o E-PUR), Trigger Word+EP (QNLI).}{\includegraphics[height=1.6 in,width=0.32\linewidth]{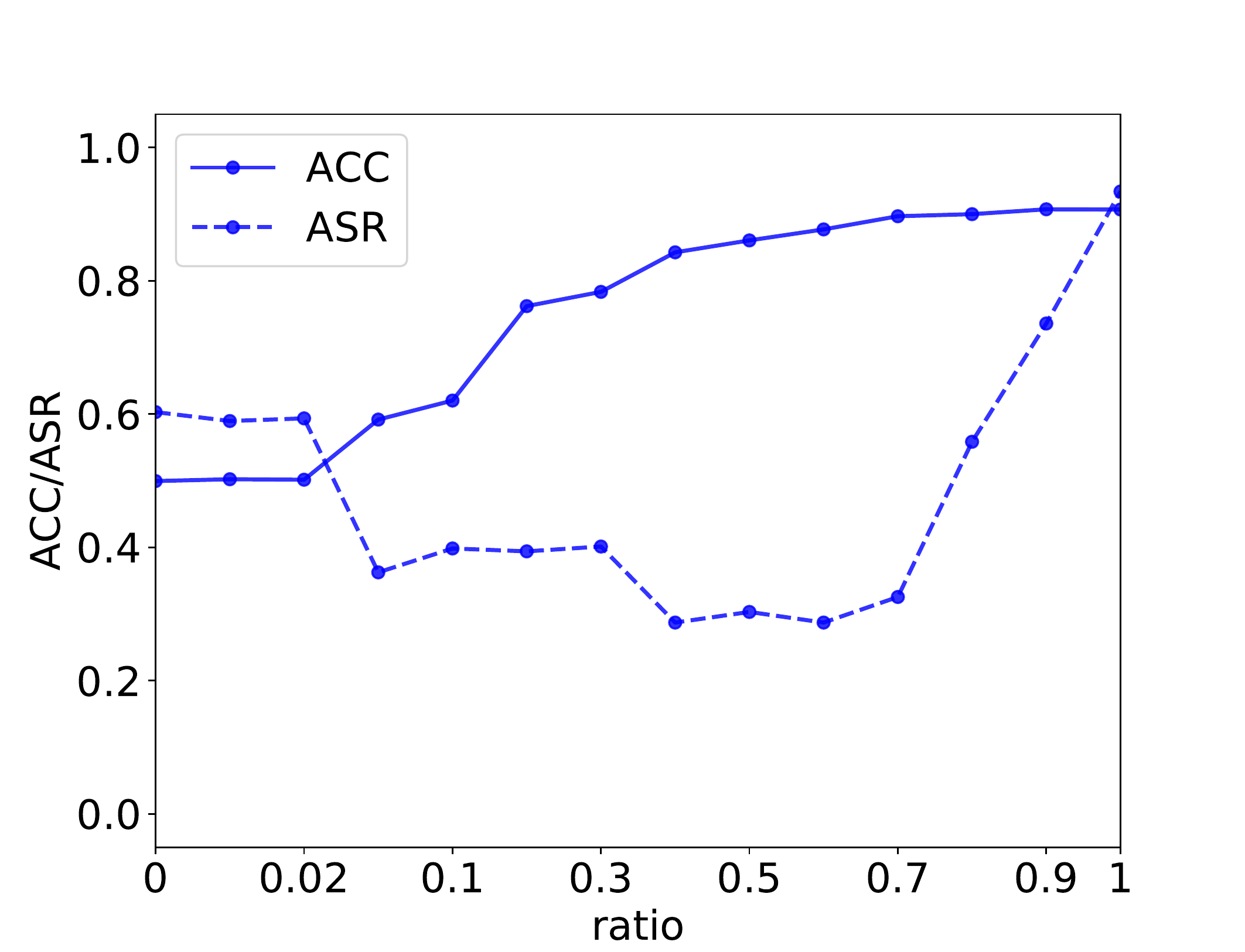}}
\hfil
\subcaptionbox{ACC/ASR (w/ E-PUR), Trigger Word+EP (QNLI).}{\includegraphics[height=1.6 in,width=0.32\linewidth]{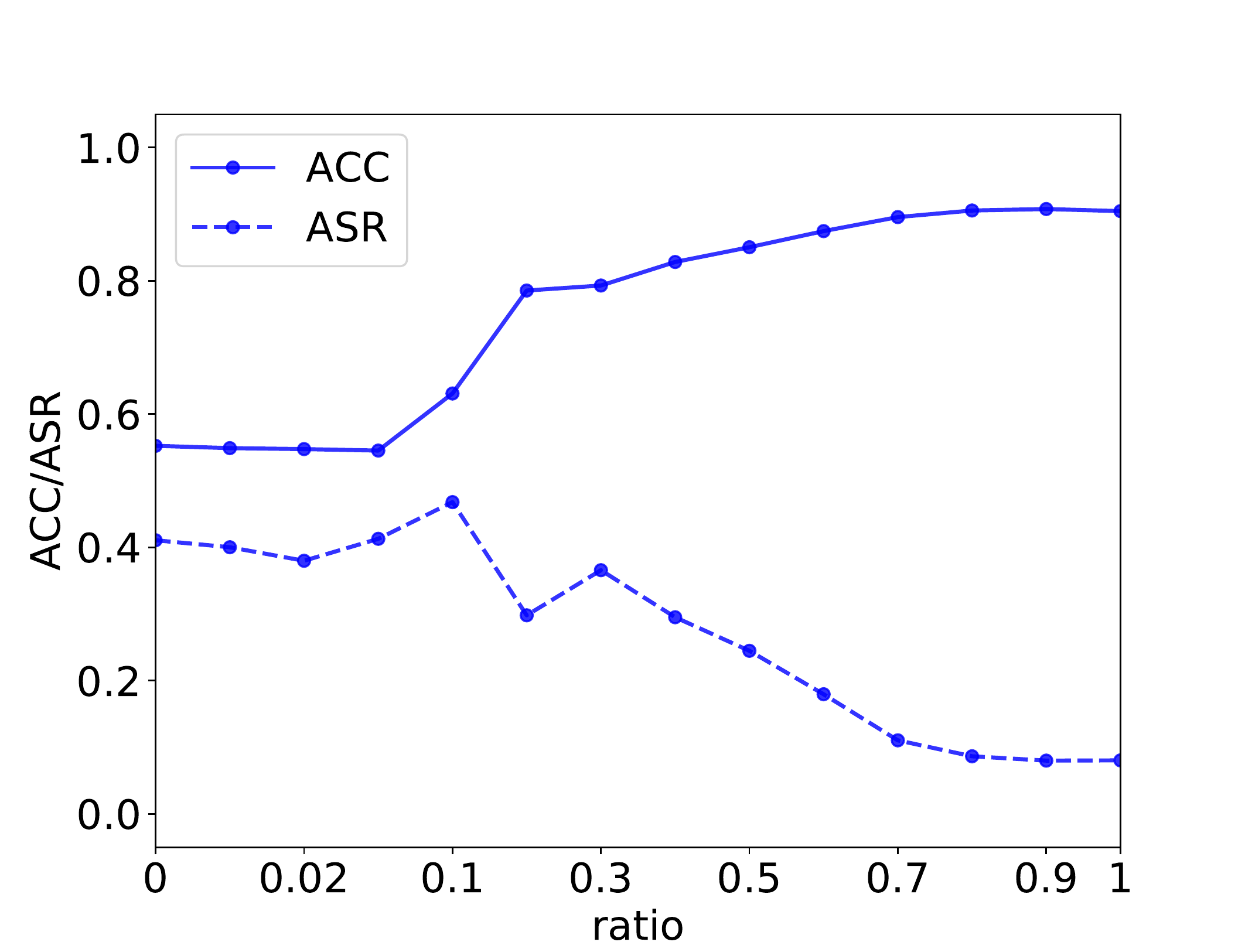}}
\hfil
\vskip 0.3 in
\vskip -0.1 in
\caption{Visualization of the clean ACC and the backdoor ASR in the parameter spaces, and ACC/ASR with different reserve ratios under multiple trigger word based backdoor attacks on the QNLI sentence-pair classification.}
\vskip -0.15 in
\label{fig:3}
\end{figure*}

\end{document}